%% file: paper.tex
\def\year{2022}\relax
\documentclass[letterpaper]{article} % DO NOT CHANGE THIS
\usepackage{aaai22}  % DO NOT CHANGE THIS
\usepackage{times}  % DO NOT CHANGE THIS
\usepackage{helvet}  % DO NOT CHANGE THIS
\usepackage{courier}  % DO NOT CHANGE THIS
\usepackage[hyphens]{url}  % DO NOT CHANGE THIS
\usepackage{graphicx} % DO NOT CHANGE THIS
\usepackage{natbib}  % DO NOT CHANGE THIS AND DO NOT ADD ANY OPTIONS TO IT
\usepackage{caption} % DO NOT CHANGE THIS AND DO NOT ADD ANY OPTIONS TO IT
\DeclareCaptionStyle{ruled}{labelfont=normalfont,labelsep=colon,strut=off} % DO NOT CHANGE THIS
\input{math_commands.tex}

\usepackage{nameref}
\usepackage{url}
\usepackage{ifthen}
\usepackage{braket}
\usepackage{amssymb,amsthm,amsfonts}
\usepackage{mathtools}
\usepackage{graphicx}
\usepackage{scalerel}
\usepackage{stmaryrd}
\usepackage{centernot}
\usepackage{todonotes}
\usepackage{nicefrac}
\usepackage{subcaption}
\usepackage{mathabx}
\usepackage{multibib}
\usepackage{booktabs}
\usepackage{enumerate}

\newcites{AR}{Additional References}

\newtheorem{theorem}{Theorem}[section]
\newtheorem{corollary}{Corollary}[theorem]
\newtheorem{lemma}[theorem]{Lemma}
\newtheorem{assumption}[theorem]{Assumption}
\theoremstyle{remark}
\newtheorem{remark}{Remark}

\newif\ifappendix
\appendixtrue
\ifappendix
\nocopyright
\fi

\newcommand{\smallparagraph}[1]{\smallskip\noindent\textbf{#1}}
\newcommand{\tabitem}{~~\llap{\textbullet}~~}

\newcommand{\tuple}[1]{\ensuremath{\left\langle #1 \right\rangle}}

\newcommand{\N}{\mathbb{N}}
\newcommand{\fun}[1]{\ensuremath{\mathopen{}\mathclose\bgroup\left(#1\aftergroup\egroup\right)}}
\newcommand{\condition}[1]{\ensuremath{\1_{#1}}}
\newcommand{\vect}[1]{\ensuremath{\bm{#1}}}

\newcommand{\images}[1]{\ensuremath{\mathrm{Im}\fun{#1}}}

\newcommand{\mdp}{\ensuremath{\mathcal{M}}}
\newcommand{\states}{\ensuremath{\mathcal{S}}}
\newcommand{\actions}{\ensuremath{\mathcal{A}}}

\newcommand{\probtransitions}{\ensuremath{\mathbf{P}}} %^{\transitions}}
\newcommand{\rewards}{\ensuremath{\mathcal{R}}}
\newcommand{\labels}{\ensuremath{\ell}}
\newcommand{\atomicprops}{\ensuremath{\mathbf{AP}}}
\newcommand{\sinit}{\ensuremath{s_{\mathit{I}}}}
\newcommand{\mdptuple}{\ensuremath{ \tuple{\states, \actions, \probtransitions, \rewards, \labels, \atomicprops, \sinit} }}

\newcommand{\state}{\ensuremath{s}}

\newcommand{\action}{\ensuremath{a}}

\newcommand{\reward}{\ensuremath{r}}
\newcommand{\labeling}{\ensuremath{l}}
\newcommand{\labelset}[1]{\ensuremath{\mathsf{#1}}}

\newcommand{\act}[1]{\ensuremath{\mathit{Act}\ifthenelse{\equal{#1}{}}{}{(#1)}}}

\newcommand{\inftrajectories}[1]{\ensuremath{\mathit{Traj}_{#1}}}
\newcommand{\seq}[2]{\ensuremath{#1_{\scriptscriptstyle 0:#2}}}
\newcommand{\trajectory}{\tau}
\newcommand{\trajectorytuple}[3]{\ensuremath{\tuple{#1_{\scriptscriptstyle 0:#3}, #2_{\scriptscriptstyle 0: #3-1}}}}
\newcommand{\trace}{\ensuremath{\hat{\trajectory}}}
\newcommand{\tracetuple}[5]{\ensuremath{\tuple{#1_{\scriptscriptstyle 0:#5}, #2_{\scriptscriptstyle 0: #5-1}, #3_{\scriptscriptstyle 0:#5 - 1}, #4_{\scriptscriptstyle 0: #5}}}}
\newcommand{\defaulttrace}{\ensuremath{\tracetuple{\state}{\action}{\reward}{\labeling}{T}}}
\newcommand{\traces}[1]{\ensuremath{\mathit{Traces}_{#1}}}
\newcommand{\policy}{\ensuremath{\pi}}

\newcommand{\mpolicies}[1]{\ensuremath{\Pi_{#1}^{\textnormal {ml}}}}

\newcommand{\car}{\ensuremath{\scaleto{\rotatebox[origin=c]{90}{$\circlearrowright$}}{1.4ex}}}
\newcommand{\selfloop}[2]{\ensuremath{#1^{\nobreak\hspace{.065em} \car #2 }}}
\newcommand{\getstates}[1]{\ensuremath{\llbracket #1 \rrbracket}}
\newcommand{\valuessymbol}[2]{\ensuremath{V}_{#1}^{#2}}
\newcommand{\values}[3]{\ensuremath{\valuessymbol{#1}{#2}\fun{#3}}}
\newcommand{\qvaluessymbol}[2]{\ensuremath{Q_{#1}^{#2}}}
\newcommand{\qvalues}[4]{\ensuremath{\qvaluessymbol{#1}{#2}\fun{#3, #4}}}

\newcommand{\stationary}[1]{\ensuremath{\xi_{#1}}}

\newcommand{\encodersymbol}{\ensuremath{Q}}
\newcommand{\encoder}{\ensuremath{\encodersymbol_\encoderparameter}}

\newcommand{\decodersymbol}{\ensuremath{P}}
\newcommand{\decoder}{\ensuremath{\decodersymbol_\decoderparameter}}

\newcommand{\encoderparameter}{\ensuremath{\iota}}
\newcommand{\decoderparameter}{\ensuremath{\theta}}
\newcommand{\generative}{\ensuremath{\mathcal{G}}}

\newcommand{\discount}{\ensuremath{\gamma}}

\newcommand{\always}{\ensuremath{\square}}
\newcommand{\eventually}{\ensuremath{\lozenge}}

\newcommand{\until}[2]{\ensuremath{#1 \, \mathcal{U} \, #2}}

\newcommand{\Prob}{\ensuremath{\displaystyle \mathbb{P}}}
\newcommand{\measurableset}{\ensuremath{\mathcal{X}}}

\newcommand{\borel}[1]{\ensuremath{\Sigma\fun{#1}}}

\newcommand{\sampledot}{\ensuremath{{\cdotp}}}
\newcommand{\expectedsymbol}[1]{\ensuremath{\mathop{\mathbb{E}}\ifthenelse{\equal{#1}{}}{}{_{#1}}}}
\newcommand{\expected}[2]{\ensuremath{\expectedsymbol{#1} \left[ #2 \right]}}

\newcommand{\entropy}[1]{\ensuremath{\displaystyle H(#1)}}
\newcommand{\divergencesymbol}{\ensuremath{D}}
\newcommand{\divergence}[2]{\ensuremath{\divergencesymbol\fun{#1, #2}}}
\newcommand{\dklsymbol}{\ensuremath{\divergencesymbol_{{\mathrm{KL}}}}}
\newcommand{\dkl}[2]{\ensuremath{\dklsymbol\fun{#1 \parallel #2}}}
\newcommand{\dtvsymbol}{\ensuremath{d_{{TV}}}}
\newcommand{\dtv}[2]{\ensuremath{\dtvsymbol\fun{#1, #2}}}

\newcommand{\distributions}[1]{\ensuremath{\mathcal{P}\fun{#1}}}
\newcommand{\support}[1]{\ensuremath{Supp\fun{#1}}}

\newcommand{\distortion}{\ensuremath{\mathbf{D}}}
\newcommand{\rate}{\ensuremath{\mathbf{R}}}

\newcommand{\elbo}{\ensuremath{\mathrm{ELBO}}}
\newcommand{\error}{\ensuremath{\varepsilon}}
\newcommand{\proberror}{\ensuremath{\delta}}

\newcommand{\logistic}[2]{\ensuremath{\mathrm{Logistic}(#1, #2)}}
\newcommand{\bernoulli}[1]{\ensuremath{\mathrm{Bernoulli}(#1)}}
\newcommand{\categorical}[1]{\ensuremath{\mathrm{Categorical}(#1)}}
\newcommand{\relaxedcategorical}[2]{\ensuremath{\mathrm{RelaxedCategorical}(#1, #2)}}
\newcommand{\relaxedbernoulli}[2]{\ensuremath{\mathrm{RelaxedBernoulli}(#1, #2)}}
\newcommand{\logit}{\ensuremath{{\alpha}}}
\newcommand{\logits}{\ensuremath{\vect{\logit}}}
\newcommand{\temperature}{\ensuremath{\lambda}}

\newcommand{\wassersteinsymbol}[1]{\ensuremath{W}_{#1}}
\newcommand{\wassersteindist}[3]{\ensuremath{\wassersteinsymbol{#1}\left( #2, #3 \right)}}
\newcommand{\distance}{\ensuremath{d}}
\newcommand{\logic}{\ensuremath{\mathcal{L}}}
\newcommand{\bidistance}{\ensuremath{\tilde{\distance}}}

\newcommand{\functionalexpr}{\ensuremath{\mathcal{F}_{\discount}^{\logic}}}

\newcommand{\couplings}[2]{\ensuremath{\Lambda(#1, #2)}}
\newcommand{\coupling}{\ensuremath{\lambda}}

\newcommand{\Lipschf}[1]{\ensuremath{\mathcal{F}_{#1}}}

\newcommand{\overbar}[1]{\mkern 1.5mu\overline{\mkern-1.5mu#1\mkern-1.5mu}\mkern 1.5mu}
\newcommand{\overbarit}[1]{\,\overline{\!{#1}}}
\newcommand{\embed}{\ensuremath{\phi}}
\newcommand{\embeda}{\ensuremath{\psi}}
\newcommand{\latentembeda}{\ensuremath{\encodersymbol^{\actions}}}
\newcommand{\latentmdp}{\ensuremath{\overbarit{\mdp}}}
\newcommand{\latentprobtransitions}{\ensuremath{\overbar{\probtransitions}}}
\newcommand{\latentstates}{\ensuremath{\overbarit{\mathcal{\states}}}}
\newcommand{\latentrewards}{\ensuremath{\overbarit{\rewards}}}
\newcommand{\latentlabels}{\ensuremath{\overbarit{\labels}}}
\newcommand{\latentmdptuple}{\ensuremath{\tuple{\latentstates, \latentactions, \latentprobtransitions, \latentrewards, \latentlabels, \atomicprops, \zinit}}}
\newcommand{\latentstate}{\ensuremath{\overbarit{\state}}}
\newcommand{\zinit}{\ensuremath{\latentstate_I}}
\newcommand{\latentactions}{\ensuremath{\overbarit{\actions}}}
\newcommand{\latentaction}{\ensuremath{\overbarit{\action}}}
\newcommand{\latentvaluessymbol}[2]{\overbarit{\ensuremath{V}}_{#1}^{#2}}
\newcommand{\latentvalues}[3]{\ensuremath{\latentvaluessymbol{#1}{#2}\fun{#3}}}
\newcommand{\latentqvaluessymbol}[2]{\ensuremath{\overbarit{Q}_{#1}^{#2}}}
\newcommand{\latentqvalues}[4]{\ensuremath{\latentqvaluessymbol{#1}{#2}\fun{#3, #4}}}
\newcommand{\latentpolicy}{\ensuremath{\overbar{\policy}}}
\newcommand{\latentpolicies}{\ensuremath{\overbar{\Pi}}}
\newcommand{\latentmpolicies}{\ensuremath{\overbar{\Pi}^{\textnormal {ml}}}}
\newcommand{\localtransitionloss}[1]{L_{\probtransitions}^{#1}}
\newcommand{\localrewardloss}[1]{L_{\rewards}^{#1}}
\newcommand{\localtransitionlossupper}[1]{\dot{L}_{\probtransitions}^{#1}}
\newcommand{\localrewardlossupper}[1]{\dot{L}_{\rewards}^{#1}}
\newcommand{\localtransitionlossapprox}[1]{\hat{L}_{\probtransitions}^{#1}}
\newcommand{\localrewardlossapprox}[1]{\hat{L}_{\rewards}^{#1}}
\newcommand{\KV}{\ensuremath{K_{\latentvaluessymbol{}{}}}}
\newcommand{\KR}[1]{\ensuremath{\ifthenelse{\equal{#1}{}}{K_{\latentrewards}}{K_{\latentrewards}^{#1}}}}
\newcommand{\KP}[1]{\ensuremath{\ifthenelse{\equal{#1}{}}{K_{\latentprobtransitions}}{K_{\latentprobtransitions}^{#1}}}}
\newcommand{\Rmax}[1]{\ensuremath{\ifthenelse{\equal{#1}{}}{|\latentrewards^\star|}{|\latentrewards_{#1}^{\star}|}}}

\newcommand{\latentvariables}{\ensuremath{\mathcal{Z}}}
\newcommand{\latentvariable}{\ensuremath{z}}

\newcommand{\replaybuffer}{\ensuremath{\mathcal{D}}}
\newcommand{\experience}{\ensuremath{\eta}}
\newcommand{\loss}{\ensuremath{L}}

\title{Distillation of RL Policies with Formal Guarantees via Variational
Abstraction of Markov Decision Processes
\ifappendix{\\\Large Technical Report}\fi
}

\author {
    Florent Delgrange,\textsuperscript{\rm 1}
    Ann Now\'e,\textsuperscript{\rm 1}
    Guillermo A. P\'erez\textsuperscript{\rm 2}
}
\affiliations {
    \textsuperscript{\rm 1} AI Lab, Vrije Universiteit Brussel\\
    florent.delgrange@ai.vub.ac.be\\
    \textsuperscript{\rm 2} University of Antwerp -- Flanders Make
}

\begin{document}

\maketitle

\begin{abstract}
\input{abstract}
\end{abstract}

\section{Introduction}
While \emph{reinforcement learning} (RL) has been applied to a wide range of challenging domains, from game playing \citep{DBLP:journals/nature/MnihKSRVBGRFOPB15} to real-world applications such as effective canal control \citep{REN2021103049}, more widespread deployment in the real world is hampered by the lack of guarantees provided with the learned policies.
Although there are RL algorithms which have limit-convergence guarantees
in the discrete setting \citep{DBLP:journals/ml/Tsitsiklis94} --- and even in some continuous settings with function approximation, e.g., \citet{phdthesis:Nowe94} --- these are lost when applying more advanced
techniques which make use of general nonlinear function approximators \citep{DBLP:journals/tac/TsitsiklisR97} to deal with
continuous \emph{Markov decision processes} (MDPs) such as \emph{deep}-RL (e.g.,
\citealt{DBLP:journals/nature/MnihKSRVBGRFOPB15}). %,DBLP:conf/icml/HaarnojaZAL18}).
In this paper, we apply such advanced RL algorithms to unknown continuous MDPs with (i)
reachability, (ii) safety-constrained reachability, or (iii) discounted-reward
objectives. To recover the formal guarantees, we use the obtained
policy to train a \emph{variational autoencoder} (VAE) which gives us a
\emph{discrete latent model} that approximates the unknown environment.
We build upon the \emph{DeepMDP} framework
\cite{DBLP:conf/icml/GeladaKBNB19} to provide guarantees on the quality of the abstraction induced by this model.
DeepMDPs are provided with such guarantees when their \emph{loss functions} are minimized. These can be defined on the entire state space (\emph{global}) or on states visited under a given policy (\emph{local}).
The guarantees concern a \emph{state embedding function}, linking the latent and original MDPs and are defined as bounds on the difference of their \emph{value function} and \emph{bisimulation distance}.
The latter was only developed for global losses.
While these are interesting in theory, they are often infeasible to measure in practice.
In contrast, we introduce such bounds in the local setting and further consider an \emph{action embedding function} to handle continuous actions.
Importantly, we focus on general MDPs and do not restrict our attention to deterministic ones as was done by~\citeauthor{DBLP:conf/icml/GeladaKBNB19} to enable the approximation and minimization of their losses via neural networks.
We also give PAC approximation schemes to compute both the losses and said bounds.
% =========================================

Our VAE is trained by maximizing a lower bound on the likelihood of traces generated by executing the RL policy in the environment.
We derive a loss function, incorporating variational versions of the local losses, that enables learning (i) discrete state and action spaces, (ii) an MDP defined over these spaces, (iii) state and action embedding functions, linking the original and discrete MDPs, and (iv) a \emph{distilled} version of the RL policy set over the discrete spaces which can be executed in both models via the embedding functions.
An important challenge for our approach is the \emph{posterior collapse problem} which often occurs when optimizing a variational model (e.g., \citealt{DBLP:conf/icml/AlemiPFDS018}).
We present a novel approach based on \emph{prioritized experience replay} \citep{DBLP:journals/corr/SchaulQAS15} to resolve this when learning a discrete latent model.

All of the above result in an efficient way of training a VAE to obtain a discrete latent model that is provably approximately bisimilar to the unknown MDP, further yielding a distilled version of the RL policy.
These enable the application of formal
methods and tools that have been developed for discrete 
MDPs: for instance, \textsc{Prism} \cite{KNP11}, \textsc{Modest} \cite{DBLP:conf/tacas/HartmannsH14}, and \textsc{Storm} \cite{stt:Hensel2021}.

\input{related_work}

\begin{section}{Background}
\input{background}
\end{section}

\begin{section}{Latent Space Models}\label{section:deepmdp}
\input{deepmdp}
\end{section}

\begin{section}{Variational Markov Decision Processes}\label{section:vae-mdp}
\input{new_vae}
\end{section}
\begin{figure*}
    \begin{subfigure}{0.745\textwidth}
        \centering
        \includegraphics[width=\textwidth]{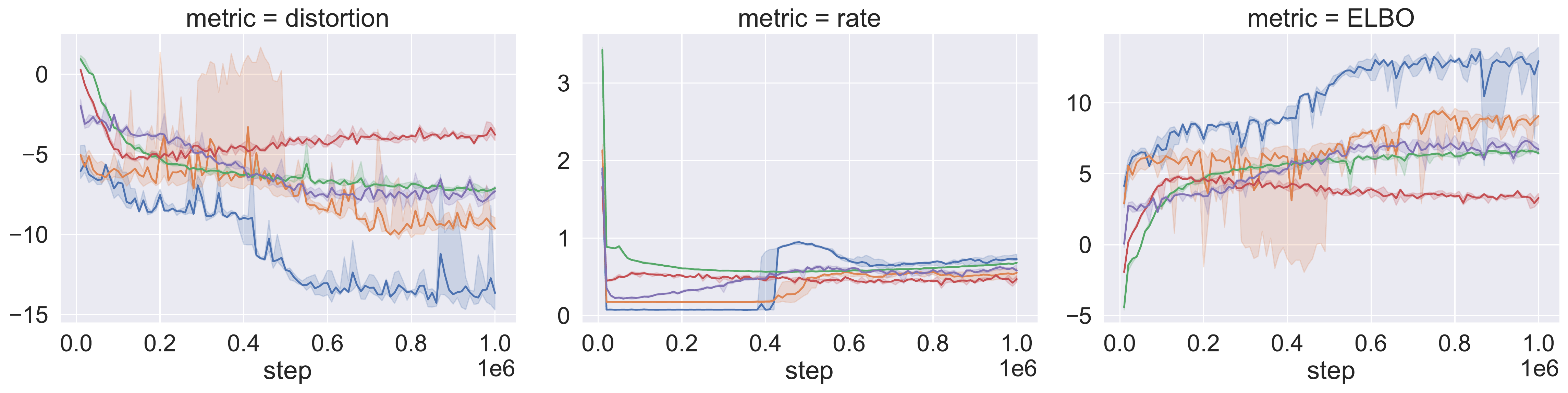}
        \caption{$\distortion$, $\rate$, and $\elbo$}
        \label{subfig:dist-rate-elbo}
    \end{subfigure}~\begin{subfigure}{0.25\textwidth}
        \centering
       \includegraphics[width=\textwidth]{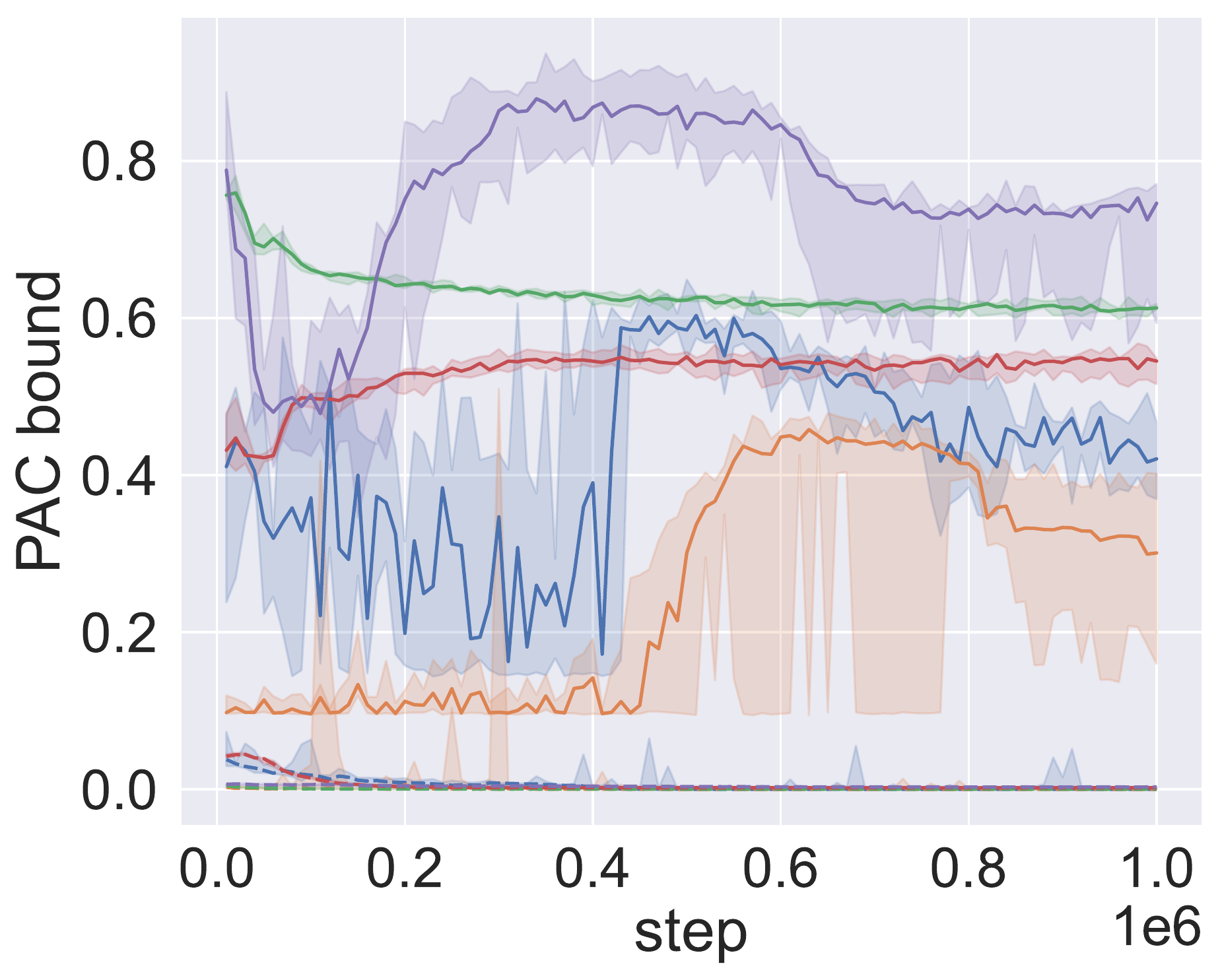}
       \caption{Local losses (PAC bound)}
       \label{subfig:local-losses}
    \end{subfigure}\\
    \begin{subfigure}{\textwidth}
        \centering 
        \includegraphics[width=\textwidth]{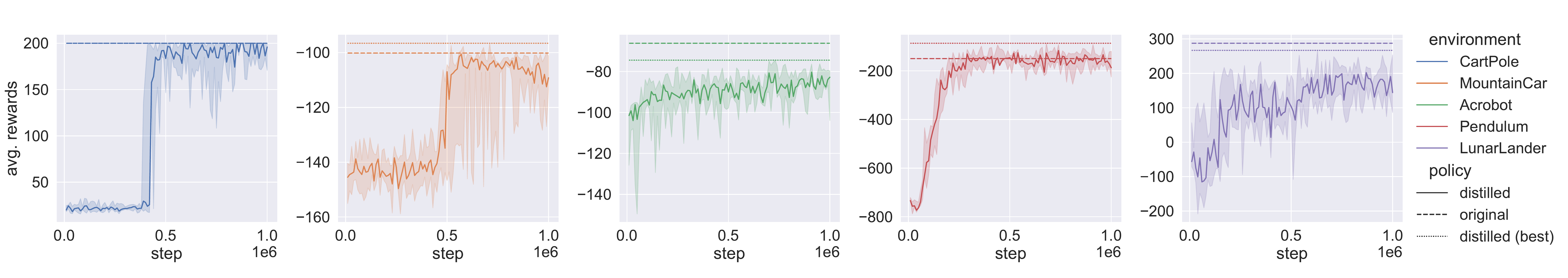}
        \caption{Distilled policy evaluation}
        \label{subfig:policy-evaluation}
    \end{subfigure}
    \caption{Plots reporting~(\ref{subfig:dist-rate-elbo})~$\distortion$, $\rate$, and $\elbo$ (approximated over a large batch sampled from the replay buffer, using discrete latent distributions and averaged using importance sampling weights),
    (\ref{subfig:local-losses})~PAC local losses approximation for $\error=10^{-2}$, $\proberror = 5\cdot 10^{-3}$:
    solid lines stand for the transition loss and dashed lines for the reward loss,
    and (\ref{subfig:policy-evaluation})~the expected episode return
    (approximated by averaging over $30$ episodes).
    For each environment, we train five different instances of our VAE with different random seeds, where the solid line corresponds to the median and the shaded interval to the interquartile range.}
    \label{fig:plots}
\end{figure*}
\begin{section}{Experiments}\label{sec:experiments}
\input{experiments}
\end{section}

\begin{section}{Conclusion}
\input{conclusion}
\end{section}

\section*{Acknowledgments}
This research received funding from the Flemish Government (AI Research Program) and was supported by the DESCARTES iBOF project. G.A. Perez is also supported by the Belgian FWO “SAILor” project (G030020N).

\bibliography{references}

\ifappendix
\appendix
\onecolumn

\input{proofs}
\fi

\end{document}

%% file: math_commands.tex
\usepackage{amsmath,amsfonts,bm}

\def\eqref#1{equation~\ref{#1}}
\def\1{\bm{1}}

\def\vx{{\bm{x}}}

\DeclareMathAlphabet{\mathsfit}{\encodingdefault}{\sfdefault}{m}{sl}
\SetMathAlphabet{\mathsfit}{bold}{\encodingdefault}{\sfdefault}{bx}{n}

\newcommand{\R}{\mathbb{R}}

\newcommand{\sigmoid}{\sigma}

\DeclareMathOperator*{\argmax}{arg\,max}

%% file: abstract.tex
We consider the challenge of policy simplification and verification in the
context of policies learned through reinforcement learning (RL) in continuous environments. In well-behaved settings, RL algorithms have convergence
guarantees in the limit. While these guarantees are valuable, they are insufficient for
safety-critical applications. Furthermore, they are lost when
applying advanced techniques such as deep-RL. To recover guarantees when
applying advanced RL algorithms to  more complex environments with (i)
reachability, (ii) safety-constrained reachability, or (iii) discounted-reward
objectives, we build upon the DeepMDP framework introduced by Gelada et al. to
derive new bisimulation bounds between the unknown environment and a learned discrete
latent model of it. Our bisimulation bounds enable the application of
formal methods for Markov decision processes. Finally, we
show how one can use a policy obtained via state-of-the-art RL to efficiently train a
variational autoencoder that yields a discrete latent model with provably
approximately correct bisimulation guarantees. Additionally, we obtain
a distilled version of the policy for the latent model.

%% file: related_work.tex
\smallparagraph{Other related work.}
%~We already mentioned the {DeepMDP} framework \citep{DBLP:conf/icml/GeladaKBNB19} 
%which introduced abstraction quality guarantees for representation learning purposes.
%
Frameworks providing formal guarantees \emph{during the RL process} include the work of \citet{DBLP:conf/tacas/Junges0DTK16}, \emph{Shielded-RL} \citep{DBLP:conf/aaai/AlshiekhBEKNT18,jansen_et_al:LIPIcs:2020:12815}, and \emph{AlwaysSafe} \citep{DBLP:conf/atal/SimaoJS21}.
These all require an abstract model of the safety aspect of the environment. Our approach is complementary in that we assume no prior knowledge and \emph{learn an abstraction}.
Notably, our goal is not the same: they aim at verifying whether the exploration is safe while our goal is to verify policies learned via \emph{any} RL technique.
%
%The \textsc{Mosaic} approach \citep{DBLP:conf/formats/Bacci020} shares ours
Other approaches %\cite{DBLP:conf/formats/Bacci020,DBLP:conf/fmcad/AlamdariAHL20,DBLP:conf/ijcai/CarrJT20}
share ours 
% \citet{DBLP:conf/formats/Bacci020} share
%this purpose with ours
in the particular case of verifying deep-RL policies, but rely on a known (abstraction of the) environment model.
\citet{DBLP:conf/formats/Bacci020} require the neural network (NN) specifying the policy, the environment to be  deterministic, and a formal description of the probability of action failures. %with which faults occur when attempting to execute particular actions.
%
%\todo{F: RNN paper here}
%Finally,
% \citet{DBLP:conf/ijcai/CarrJT20} focus on policies represented as recurrent neural networks (RNNs).
% %Although they require the environment to be discrete, the authors discretize the RNN hidden states by using \emph{quantized} autoencoders%
% Although they require the environment to be discrete, the authors discretize the RNN hidden states by using \emph{quantized} autoencoders%
\citet{DBLP:conf/ijcai/CarrJT20} require the environment to be discrete and focus on policies represented as recurrent NNs by discretizing their hidden states via \emph{quantized} autoencoders%
%Although they require the environment to be discrete, the authors discretize the RNN hidden states by using \emph{quantized} autoencoders%
%Although they require the environment to be discrete, the authors discretize the RNN hidden states by using \emph{quantized} autoencoders%
% (another approach to learning discrete latent spaces)
, in the same spirit as our policy distillation.
%, and (iii) the hidden states of the RNN to be discretized (here, via a quantized autoencoder, another approach to learning discrete latent spaces).
Finally, \citet{DBLP:conf/fmcad/AlamdariAHL20} focus on %model-checking and synthesizing controllers based on 
tree-based policies distilled from deep-RL, without considering abstraction quality guarantees.

VAEs have been used in the context of (model-based) RL to learn latent representations of the unknown environment and train
simpler policies
from the features extracted
%defined over the induced latent observation space
(e.g., \citealt{DBLP:conf/icml/CorneilGB18, DBLP:conf/nips/FreemanHM19, DBLP:conf/nips/LeeNAL20,ALA2021:Burden-et-al}).
In particular, \citet{DBLP:conf/icml/CorneilGB18} focused on learning discrete latent MDPs from continuous-state environments with discrete actions (without guarantees nor distilled policies) to plan via \emph{prioritized sweeping}. % \citep{DBLP:books/lib/SuttonB98}.

%% file: background.tex
We write
$[T] = \{n \in \N \mid n \leq T\}$.
For  $A \subseteq X$, we denote by $\condition{A} \colon X \to [1]$ the indicator function: $\condition{A}\fun{a} = 1$ iff $a \in A$.
Let $\measurableset$ be a complete and separable space and $\borel{\measurableset}$ denote the set of all Borel subsets of $\measurableset$.
We write $\distributions{\measurableset}$ for the set of measures $P$
defined on $\measurableset$ and $\support{P} = \{x \in \measurableset \mid
P(x) > 0\}$ to denote their support.

\smallparagraph{Discrepancy measures.}~Let $P, Q \in
\distributions{\measurableset}$ with density functions $p$ and $q$.
Their discrepancy can be measured via
\begin{itemize}
    \item \emph{Kullback-Leibler} (KL) divergence:
    $\dklsymbol(P\parallel Q) = \expectedsymbol{x \sim P}[\log\fun{\nicefrac{p\fun{x}}{ q\fun{x}}}].$
    \item \emph{Wasserstein}: $\wassersteinsymbol{\distance}\,(P, Q) = \inf_{\coupling \in \couplings{P}{Q}} \expectedsymbol{x, y \sim \coupling}\distance(x, y),$
    where $\distance \colon \measurableset \times \measurableset \to \mathopen[ 0, \infty \mathclose[$
    is a distance metric over $\measurableset$ and $\couplings{P}{Q}$ is the set of all \emph{couplings} of $P$ and $Q$.
    \item \emph{Total Variation} (TV):
    $\dtv{P}{Q} = \sup_{A \in \borel{\measurableset}} |P\fun{A} - Q\fun{A}|.$
    If $\measurableset$
    is equipped with 
    the discrete metric $\condition{\neq}$, TV coincides with the Wasserstein measure.
\end{itemize}

\smallparagraph{Markov decision processes.}~%
A \emph{Markov decision process} (MDP) is a tuple $\mdp = \mdptuple$ where $\states$ is a set of \emph{states}; $\actions$, a set of \emph{actions}; $\probtransitions \colon \states \times \actions \to \distributions{\states}$, a \emph{probability transition function}; $\rewards \colon \states \times \actions \to \R$, a \emph{reward function}; $\labels \colon \states \to 2^\atomicprops$, a \emph{labeling function} over a set of atomic propositions $\atomicprops$; and $\sinit \in \states$, the \emph{initial state}.
The set of \emph{enabled actions} of $\state \in \states$ is $\act{\state} \subseteq \actions$.
We assume 
$\act{\state} \neq \emptyset$ for all $\state \in \states$. 
If $|\act{\state}| = 1$ for all $\state \in \states$, 
$\mdp$
is a fully stochastic process called a \emph{Markov chain} (MC).

Let $\labelset{T} \subseteq \atomicprops$, we write $\getstates{\labelset{T}} = \set{\state \mid \labels\fun{\state}\cap \labelset{T} \neq \emptyset} \subseteq \states$ and $\getstates{\neg \labelset{T}} = \states \setminus \getstates{\labelset{T}}$.
We assume $\atomicprops$ and labels being respectively \emph{one-hot} and binary encoded.
We write $\mdp_\state$ for the MDP obtained when we replace the initial state of $\mdp$ by $\state \in \states$, $\mdp \oplus \rewards'$ when we replace the reward function by $\rewards'$,
and $\selfloop{\mdp}{B}$ when we make absorbing states from $B \subseteq \states$, i.e., by changing $\probtransitions\fun{\sampledot \mid \state, \action}$ to $\probtransitions'\fun{\sampledot \mid \state, \action}$ such that $\probtransitions'\fun{B \mid \state, \action} = 1$ for all $\state \in B$, $\action \in \act{\state}$.
We refer to MDPs with continuous states or actions spaces as \emph{continuous MDPs}. 
In that case, we assume $\states$ and $\actions$ are complete separable metric spaces equipped with a Borel $\sigma$-algebra
and $\labels^{-1}\fun{\labelset{T}} \in \borel{\states}$ for any $\labelset{T} \subseteq \atomicprops$.

\smallparagraph{Trajectories.}~A \emph{trajectory} $\trajectory{}$ of $\mdp$ is a sequence of states and actions $\trajectory = \trajectorytuple{\state}{\action}{T}$
where $\state_0 = \sinit$, $\state_{t + 1} \sim \probtransitions\fun{\sampledot \mid \state_t, \action_t}$ and $\action_t \in \act{\state_t}$ for $t \in [T - 1]$.
The set of infinite trajectories of $\mdp$ is \inftrajectories{\mdp}.
An \emph{execution trace} $\trace$ of $\mdp$ is a trajectory that additionally records labels and rewards encountered.
The set of execution traces of $\mdp$ is $\traces{\mdp}$.

\smallparagraph{Policies.}~A \emph{(memoryless) policy} $\policy \colon \states \to \distributions{\actions}$ of $\mdp$ is a 
stochastic mapping from states to actions
such that $\support{\policy\fun{\sampledot \mid \state}} \subseteq \act{\state}$.
The set of memoryless policies of $\mdp$ is $\mpolicies{\mdp}$.
An MDP $\mdp$ and $\policy \in \mpolicies{\mdp}$
induce an MC $\mdp_\policy$ along with a unique probability measure $\Prob^\mdp_\policy$ on the Borel $\sigma$-algebra over measurable subsets $E \subseteq \inftrajectories{\mdp}$~\citep{DBLP:books/wi/Puterman94}.
We drop the superscript when the context is clear.
For $\policy \in \mpolicies{\mdp}$, we denote by $\probtransitions_\policy \colon \states \to \distributions{\states}$ the probability transition distribution of $\mdp_\policy$, and by $\rewards_\policy \colon \states \to \R$ its reward function.
We write $\trace = \defaulttrace \sim \mdp_{\policy}$ for $\trajectorytuple{\state}{\action}{T} \sim \Prob^{\mdp}_{\policy}$ with $\trace \in \traces{\mdp_{\policy}}$.

\smallparagraph{Stationary distributions.}~Let $\policy \in \mpolicies{\mdp}$, $\stationary{\policy}^{t}: \states \to \distributions{\states}$ with $ \stationary{\policy}^{t}\fun{\state' \mid \state} =  \Prob^{\mdp_\state}_\policy\fun{\set{{\seq{\state}{\infty}, \seq{\action}{\infty}}
\mid \state_t = \state'}}$ be the distribution giving the probability for the agent of being in each state of $\mdp_\state$ after $t$ steps, and $B \subseteq \states$.
$B$ is a \emph{strongly connected component} (SCC)
of $\mdp_\policy$ if for any pair of states $\state, \state' \in B$, $\stationary{\policy}^t\fun{\state' \mid \state} > 0$ for some $t \in \N$. 
It is a \emph{bottom SCC} (BSCC) if (i) $B$ is a maximal SCC, and (ii) for each $\state \in B$, $\probtransitions_\policy\fun{B \mid \state} = 1$.
The unique stationary distribution of $B$ is
$\stationary{\policy} \in \distributions{B}$.
We write $\state, \action \sim \stationary{\policy}$ as shorthand for first sampling $\state$ from $\stationary{\policy}$ and then $\action$ from $\policy$.
An MDP $\mdp$ is \emph{ergodic} if for all $\policy \in \mpolicies{\mdp}$, the state space of $\mdp_\policy$ consists of a unique aperiodic BSCC %
with $\stationary{\policy} = \lim_{t \to \infty} \stationary{\policy}^t\fun{\sampledot \mid \state}$ for all $\state \in \states$.

\smallparagraph{Events and functions.}~Let $\labelset{C}, \labelset{T} \subseteq \atomicprops$, we define the \emph{constrained reachability} (resp. \emph{reachability}) event as $\until{\labelset{C}}{\labelset{T}} =
\{\, \seq{\state}{\infty}, \seq{\action}{\infty} \mid \exists i \in \N, \forall j < i, \state_j \in \getstates{\labelset{C}} \wedge \state_i \in \getstates{\labelset{T}}\, \} \in \borel{\inftrajectories{\mdp}}$
(resp. $\eventually \labelset{T} = \until{\neg \emptyset}{\labelset{T}}$).
Safety w.r.t. a set of failure states $\labelset{T}$ can be expressed as a safe-constrained reachability event to a safe destination $\labelset{C}$ (resp. safety event) through $\until{\neg \labelset{T}}{\labelset{C}}$
(resp. $\always \neg \labelset{T} = \inftrajectories{\mdp} \setminus \eventually \labelset{T}$).
Let $\discount \in \mathopen[0, 1\mathclose[,$
$\varphi \in \set{\epsilon, \until{\labelset{C}}{\labelset{T}}, \eventually \labelset{T}}$ where $\epsilon$ is the empty symbol, and
$\rewards^{\labelset{T}} = \fun{1 - \discount}\condition{\getstates{\labelset{T}} \times \actions}$,
the \emph{value} obtained by running
$\policy \in \mpolicies{\mdp}$
from state $\state$ in $\mdp$ is
$\values{\policy}{\varphi}{\state} = \expectedsymbol{\policy}^{\mdp[\state, \varphi]}\left[{\sum_{t=0}^{\infty} \discount^t \rewards(\state_t, \action_t)}\right]$.
It corresponds to the \emph{expected discounted} (i) \emph{return} when $\varphi = \epsilon$ with $\mdp[\state] = \mdp_\state$, (ii) \emph{constrained reachability} when $\varphi = \until{\labelset{C}}{\labelset{T}}$ with $\mdp[\state, \until{\labelset{C}}{\labelset{T}}] = \selfloop{\mdp_\state}{\getstates{\neg \labelset{C}} \cup \getstates{\labelset{T}}} \oplus \rewards^\labelset{T}$, (iii) \emph{reachability} when $\varphi = \eventually \labelset{T}$ with $\mdp[\state, \eventually \labelset{T}] = \selfloop{\mdp_\state}{\getstates{\labelset{T}}} \oplus \rewards^\labelset{T}$.
When $\varphi \in \set{\until{\labelset{C}}{\labelset{T}}, \eventually
\labelset{T}}$, observe that $\values{\policy}{\varphi}{t}$ = 1 for $t \in
\getstates{\labelset{T}}$ and $\lim_{\discount \to
1}\values{\policy}{\varphi}{\state} =
\Prob_\policy^{\mdp_\state}\fun{\varphi}$ for $\state \in \states$.
The \emph{action-value function} is $\qvalues{\policy}{\varphi}{\state}{\action} = \rewards'\fun{\state, \action} + \expected{\state' \sim \probtransitions\fun{\sampledot \mid \state, \action}}{\discount \values{\policy}{\varphi}{\state'}}$, with $\rewards' = \rewards$ if $\varphi = \epsilon$ and $\rewards' = \rewards^\labelset{T}$ otherwise.

%% file: deepmdp.tex
Given the original (continuous, possibly unknown) environment modeled as an MDP, a \emph{latent space model} is another (simpler, smaller, and explicit) MDP with state-action space linked to the original one via \emph{embedding functions}.
The latter can be learned to optimize an \emph{equivalence criterion} between the two models.
Formally, fix MDPs $\mdp = \mdptuple$ and $\latentmdp = \latentmdptuple$ such that $\latentstates$ is equipped with metric $\distance_{\latentstates}$.
Let $\embed\colon \states \to \latentstates$ and $\embeda \colon \states \times \latentactions \to \actions$ be respectively state and action {embedding functions}.
We refer to $\tuple{\latentmdp, \embed, \embeda}$ as a {latent space model} of $\mdp$ and $\latentmdp$ as its \emph{latent MDP}.
We write $\latentmpolicies = \mpolicies{\latentmdp}$, and ${\latentqvaluessymbol{\latentpolicy}{}}$ for the action-value function of a policy $\latentpolicy \in \latentmpolicies$ in $\latentmdp$.
We also consider $\latentpolicy$ as a policy in $\mdp$: states passed to $\latentpolicy$ are embedded with $\embed$, then actions executed are embedded with $\embeda$.
Let $\latentpolicy \in \latentmpolicies$ and $\state \in \states$, we write $\latentaction \sim \latentpolicy\fun{\sampledot \mid \state}$ for $ \latentaction \sim \latentpolicy\fun{\sampledot \mid \embed\fun{\state}}$ and $\qvalues{\latentpolicy}{}{\state}{\latentaction}$ as shorthand for $\qvalues{\latentpolicy}{}{\state}{\embeda\fun{\state, \latentaction}}$.

A particular point of interest is to focus on discrete latent models,
where $\distance_{\latentstates} = \condition{\neq}$.
In the following, we adopt the latent space model formalism of
\citet{DBLP:conf/icml/GeladaKBNB19}.

\smallparagraph{Notations.}~Let $\latentpolicy \in \latentmpolicies$, we write $\Rmax{\latentpolicy}$ for $\sup_{\latentstate \in \latentstates} \left| \latentrewards_{\latentpolicy}\fun{\latentstate} \right|$.
We say that $\latentmdp$ is \emph{$\langle \KR{\latentpolicy}, \KP{\latentpolicy}\rangle$-Lipschitz} if for all $\latentstate_1, \latentstate_2 \in \latentstates$,
\begin{align*}
    \left|\latentrewards_{\latentpolicy}\fun{\latentstate_1} - \latentrewards_{\latentpolicy}\fun{\latentstate_2}\right| &\leq \KR{\latentpolicy} \distance_{\latentstates}\fun{\latentstate_1, \latentstate_2},\\
    \wassersteindist{\distance_{\latentstates}}{\latentprobtransitions_{\latentpolicy}\fun{\sampledot \mid \latentstate_1}}{ \latentprobtransitions_{\latentpolicy}\fun{\sampledot \mid \latentstate_2}} &\leq \KP{\latentpolicy} \distance_{\latentstates}\fun{\latentstate_1, \latentstate_2}.
\end{align*}

\smallparagraph{Local losses.}~%
Let $\stationary{} \in \distributions{\states \times \actions}$,
\emph{local losses} are defined as:% follows:
\begin{align*}
    \localrewardloss{\stationary{}} &= \expectedsymbol{\state, \latentaction \sim \stationary{}}\left| \rewards\fun{\state, \latentaction} - \latentrewards\fun{\embed\fun{\state}, \latentaction} \right|, \\
    \localtransitionloss{\stationary{}} &= \expectedsymbol{\state, \latentaction \sim \stationary{}} \wassersteindist{\distance_{\latentstates}}{\embed\probtransitions\fun{\sampledot \mid \state, \latentaction}}{\latentprobtransitions\fun{\sampledot \mid \embed\fun{\state}, \latentaction}}
\end{align*}
where (i) $\rewards\fun{\state, \latentaction}$, (ii)
$\probtransitions\fun{\sampledot \mid \state, \latentaction}$, and (iii)
$\embed\probtransitions\fun{\sampledot \mid \state, \latentaction}$ are
shorthand for (i) $\rewards\fun{\state, \embeda\fun{\state, \latentaction}}$,
(ii) $\probtransitions\fun{\sampledot \mid \state, \embeda\fun{\state,
\latentaction}}$, and (iii) the distribution over $\latentstates$ of sampling
$\state' \sim \probtransitions\fun{\sampledot \mid \state, \latentaction}$ and
then embedding $\latentstate' = \embed\fun{\state}$. %

Assuming $\latentmdp$ is discrete, all $\wassersteinsymbol{\distance_{\latentstates}}$ terms can be replaced by $\dtvsymbol$ 
since Wasserstein coincides with TV when using the discrete metric.
In that case, (optimal) constants $\KR{\latentpolicy}$ and
$\KP{\latentpolicy}$ can be computed in polynomial time in $\latentmdp$ for any $\latentpolicy \in \latentmpolicies$.

Henceforth, we make the following assumptions.
\begin{assumption}\label{assumption:ergodicity}
MDP $\mdp$ is ergodic.
\end{assumption}
\begin{assumption}\label{assumption:rewards}
Rewards of $\mdp$ are scaled in the interval $\mathopen[ -\frac{1}{2}, \frac{1}{2}\mathclose]$, i.e., $\rewards \colon \states \times \actions \to \mathopen[ -\frac{1}{2}, \frac{1}{2}\mathclose]$.
\end{assumption}
\begin{assumption}\label{assumption:labeling}
The embedding function preserves the labels, i.e., $\embed\fun{\state} = \latentstate \implies \labels\fun{\state} = \latentlabels\fun{\latentstate}$ for $\state \in \states$, $\latentstate \in \latentstates$. 
\end{assumption}
Seemingly restrictive at first glance, Assumption~\ref{assumption:ergodicity} is compliant with RL environments and a wide range of continuous learning tasks (every episodic RL process is ergodic, see \citealt{DBLP:conf/nips/Huang20}).
Discarding it restricts the upcoming guarantees to BSCCs. 
Assumption~\ref{assumption:rewards} basically requires rewards to be bounded and re-scalable in this interval.
This
is a reasonable assumption in practice and re-scaling is straightforward if bounds are known (otherwise, dynamical re-scaling is still feasible but complicates the implementation).
In Sect.~\ref{subsec:discrete-latent-distr}, we show that Assumption~\ref{assumption:labeling} can be made trivial.
Note that our approach requires Assumption~\ref{assumption:rewards} and~\ref{assumption:labeling}. 

\ifappendix
Proofs of the claims made in the following two subsections are provided in supplementary material (Appendix~\ref{appendix:section:latent-space-model}).
\fi

\subsection{Bisimulation and Value Difference Bounds}

We aim now at formally checking
whether
the latent space model offers a good abstraction of the original MDP $\mdp$.
To do so, we present bounds that link the two MDPs.
We extend the bounds from \citet{DBLP:conf/icml/GeladaKBNB19} to discrete spaces while additionally taking into account state
labels and discounted reachability events.
Moreover, we present new \emph{bisimulation} bounds in the local setting.

\smallparagraph{Bisimulation.}
A \emph{(probabilistic) bisimulation}
is a behavioral equivalence between states. 
Formally, a bisimulation on $\mdp$ is an equivalence relation $B_{\Phi}$ such
that for all $\state_1, \state_2 \in \states$ and $\Phi \subseteq
\set{\rewards, \labels}$, $\state_1 B_{\Phi} \state_2$ iff
$\probtransitions\fun{T \mid \state_1, \action} = \probtransitions\fun{T \mid
\state_2, \action}$,
$\labels\fun{\state_1} = \labels\fun{\state_2}$ if  $\labels
    \in \Phi$,  and $\rewards\fun{\state_1, \action} = \rewards\fun{\state_2, \action}$ if $\rewards \in \Phi$,
for each action $\action \in \actions$ and (Borel measurable) equivalence class $T \in \states / B_{\Phi}$.
Properties of bisimulation include trace, trajectory, and value equivalence \citep{DBLP:conf/popl/LarsenS89,DBLP:journals/ai/GivanDG03}.
The relation can be extended to compare two MDPs (in our case $\mdp$ and $\latentmdp$) by considering the disjoint union of their state space.
We denote the largest bisimulation relation by $\sim_{\Phi}$.

\smallparagraph{Pseudometrics.}
\citet{DBLP:journals/tcs/DesharnaisGJP04} introduced \emph{bisimulation pseudometrics} for continuous Markov processes, generalizing the notion of bisimilariy by assigning a \emph{bisimilarity distance} between states. 
A \emph{pseudometric} $\bidistance$ satisfies symmetry 
and the triangle inequality. %, but not the identity of indiscernible.

Probabilistic bisimilarity can be characterized by a logical family of functional expressions
%\citep{DBLP:journals/tcs/DesharnaisGJP04}
derived from a logic $\logic$.
More specifically, given a policy $\policy \in \mpolicies{\mdp}$, we consider a family $\functionalexpr\fun{\policy}$ of real-valued functions $f$, parameterized by the discount factor $\discount$ and defining the semantics of $\logic$ in $\mdp_\policy$.
\ifappendix
Such a logic allows formalizing discounted properties, including reachability (e.g., \citealt{DBLP:conf/fsttcs/ChatterjeeAMR08}).
\fi
The related pseudometric $\bidistance_{\policy}$
is defined as: $\bidistance_\policy\fun{\state_1, \state_2} = \sup_{f \in \functionalexpr\fun{\policy}} \left| f\fun{\state_1} - f\fun{\state_2} \right|$, for all $\state_1,\state_2 \in \states$.
\iffalse
\begin{align*}
    && \bidistance_\policy\fun{\state_1, \state_2} =& \sup_{f \in \functionalexpr\fun{\policy}} \left| f\fun{\state_1} - f\fun{\state_2} \right|, & \forall \state_1, \state_2 \in \states.
\end{align*}
\fi
We distinguish between pseudometrics $\bidistance^{\rewards}_\policy$, characterized by functional expressions including rewards%,
\ifappendix
$\,$(e.g., \citealt{DBLP:conf/birthday/FernsPK14}), 
\else
,
\fi
and $\bidistance^{\labels}_\policy$, whose functional expressions are based on state labels%
\ifappendix
$\,$(e.g., \citealt{DBLP:conf/fossacs/ChenBW12}).
\else
.
\fi
Let $\tilde{P}$ be the space of pseudometrics on $\states$ and $\varphi \in \set{\rewards, \labels}$.
Define $\Delta\colon \tilde{P} \to \tilde{P}$ so that $\Delta\fun{\distance^{\varphi}_\policy}\fun{\state_1, \state_2}$ is $\fun{1 - \discount} \left| \rewards_\policy\fun{\state_1} - \rewards_\policy\fun{\state_2} \right| \cdot \condition{\set{\rewards}}\fun{\varphi} + M$ where:
\[
M = \max \begin{cases}
\discount
    \wassersteindist{\distance_{\policy}^{\varphi}}{\probtransitions_\policy\fun{\sampledot
  \mid \state_1}}{\probtransitions_\policy\fun{\sampledot \mid \state_2}},\\
  \condition{\neq}\fun{\labels\fun{\state_1},
  \labels\fun{\state_2}} \cdot \condition{\set{\labels}}\fun{\varphi},
\end{cases}
\]
\iffalse
\begin{align*}
    \Delta\fun{\distance^{\varphi}_\policy}\fun{\state_1, \state_2} =& \fun{1 - \discount} \left| \rewards_\policy\fun{\state_1} - \rewards_\policy\fun{\state_2} \right| \cdot \condition{\set{\rewards}}\fun{\varphi} \\
    &+ \max \{ \discount
    \wassersteindist{\distance_{\policy}^{\varphi}}{\probtransitions_\policy\fun{\sampledot
  \mid \state_1}}{\probtransitions_\policy\fun{\sampledot \mid \state_2}},\\
    &\phantom{+ \max \{\} }\condition{\neq}\fun{\labels\fun{\state_1},
  \labels\fun{\state_2}} \cdot \condition{\set{\labels}}\fun{\varphi}\}
\end{align*}
\fi
then $\bidistance^{\varphi}_\policy$ is its unique fixed point whose kernel is $\sim_{\{\varphi\}}$, i.e., $\bidistance^{\varphi}_{\latentpolicy}(\state_1, \state_2) = 0$ iff
$\state_1 \!\! \sim_{\{\varphi\}} \! \! \state_2$
\citep{DBLP:conf/icalp/BreugelW01,DBLP:conf/birthday/FernsPK14},
and
$ | \values{\policy}{}{\state_1} - \values{\policy}{}{\state_2} | \leq \nicefrac{\bidistance^{\rewards}_{\policy}\fun{\state_1, \state_2}}{1 - \discount}
$ \cite{DBLP:conf/uai/FernsPP05}.

\smallparagraph{Bisimulation bounds.}~%
While computing this distance is
intractable
\ifappendix
in practice, especially 
\fi
when continuous spaces are involved, we emphasize that local losses can be used to evaluate the original and latent model bisimilarity.
Fix $\latentpolicy \in \latentmpolicies$ and assume $\latentmdp$ is discrete.
Given the induced stationary distribution $\stationary{\latentpolicy}$ in $\mdp$,
\begin{align}
    \expectedsymbol{\state \sim \stationary{\latentpolicy}} \bidistance^{\rewards}_{\latentpolicy}\fun{\state, \embed\fun{\state}} & \leq  \localrewardloss{\stationary{\latentpolicy}} + \discount \localtransitionloss{\stationary{\latentpolicy}} \frac{\KR{\latentpolicy}}{1 - \discount \KP{\latentpolicy}}, \notag \\
    \expectedsymbol{\state \sim \stationary{\latentpolicy}} \bidistance^{\labels}_{\latentpolicy}\fun{\state, \embed\fun{\state}} & \leq \frac{\discount \localtransitionloss{\stationary{\latentpolicy}}}{1 - \discount}, \label{eq:bidistance-bound}
\end{align}
where $\latentmdp$ is $\langle\KR{\latentpolicy}, \KP{\latentpolicy}\rangle$-Lipschitz.
The result provides us a general way to assess the quality of the
abstraction: % induced by %the embedding functions:
%$\tuple{\embed, \embeda}$
%a latent space model:
the bisimulation distance between states and their embedding is
guaranteed to be small in average whenever local losses are small.
%Moreover, this allows us to bound the bisimulation distance between states with same representation by local losses:
Analogously, local losses can further be used to check whether two states with the same representation are indeed bisimilarly close: % representation induced by $\embed$:
% In terms of representation, local losses further bound to evaluate whether two states with the same embedding 
%
for $\state_1, \state_2 \in \states$ with $\embed\fun{\state_1} = \embed\fun{\state_2}$,%
\begin{align}
    \bidistance^{\rewards}_{\latentpolicy}\fun{\state_1, \state_2} \leq& \left[\localrewardloss{\stationary{\latentpolicy}} + \frac{
    \discount \localtransitionloss{\stationary{\latentpolicy}} \KR{\latentpolicy}}{1 - \discount \KP{\latentpolicy}}\right] \fun{\stationary{\latentpolicy}^{-1}\fun{\state_1} + \stationary{\latentpolicy}^{-1}\fun{\state_2}}, \notag \\
    \bidistance^{\labels}_{\latentpolicy}\fun{\state_1, \state_2} \leq&  \frac{\discount \localtransitionloss{\stationary{\latentpolicy}}}{1 - \discount} \fun{\stationary{\latentpolicy}^{-1}\fun{\state_1} + \stationary{\latentpolicy}^{-1}\fun{\state_2}}. \label{eq:bidistance-abstraction-quality}
\end{align}
\ifappendix
\begin{remark}
%Our goal is to lift the guarantees obtained by model-checking the latent MDP to the real environment, so that executing $\latentpolicy$ 
Our goal is to quantify the gap between the behaviors induced by the latent policy executed in the original environment compared to those induced when executed in the latent MDP.
This allows for example to lift the guarantees obtained by model-checking $\latentmdp_{\latentpolicy}$ to $\mdp_{\latentpolicy}$.
% For this reason, notice that we consider states produced by $\latentpolicy$ in the original environment in Eq.~\ref{eq:bidistance-abstraction-quality} and to define the expectation of Eq.~\ref{eq:bidistance-bound}.
For this reason, notice that the expectations of Eq.~\ref{eq:bidistance-bound} are set over states produced by $\latentpolicy$ \emph{in the original environment}: 
this way, states that are likely to be seen under $\latentpolicy$ in $\mdp$ are ensured to have this behavioral gap bounded by a factor of $\localrewardloss{\stationary{\latentpolicy}}$ and/or $\localtransitionloss{\stationary{\latentpolicy}}$.
\end{remark}
\fi
\smallparagraph{Value difference bounds.}~%
Considering discounted returns or a specific event, the quality of the latent abstraction %induced by $\embed$ and $\embeda$
can be in particular formalized by means of \emph{value difference bounds}.
These bounds can be intuitively derived by taking the value function as a
real-valued function from $\functionalexpr\fun{\policy}$.
Let $\labelset{C}, \labelset{T} \subseteq \atomicprops$, $\varphi \in \set{\until{\labelset{C}}{\labelset{T}}, \eventually \labelset{T}}$ and $\KV = \min\fun{\nicefrac{\Rmax{\latentpolicy}}{1 - \discount}, \nicefrac{\KR{\latentpolicy}}{1 - \discount \KP{\latentpolicy}}}$,
then, 
\begin{align}
    \expectedsymbol{\state, \latentaction \sim \stationary{\latentpolicy}} \left| \qvalues{\latentpolicy}{}{\state}{\latentaction} - \latentqvalues{\latentpolicy}{}{\embed\fun{\state}}{\latentaction} \right| & \leq \frac{\localrewardloss{\stationary{\latentpolicy}} + \discount \KV \localtransitionloss{\stationary{\latentpolicy}}}{1 - \gamma}, \notag \\ 
    \expectedsymbol{\state, \latentaction \sim \stationary{\latentpolicy}} \left| \qvalues{\latentpolicy}{\varphi}{\state}{\latentaction} - \latentqvalues{\latentpolicy}{\varphi}{\embed\fun{\state}}{\latentaction} \right| & \leq \frac{\discount \localtransitionloss{\stationary{\latentpolicy}}}{1 - \discount}. \label{eq:value-dif-bound}
\end{align}
    Moreover, for any states $\state_1, \state_2 \in \states$ with $\embed\fun{\state_1} = \embed\fun{\state_2}$,
    \begin{align}
      &\left| \values{\latentpolicy}{}{\state_1} {-}
      \values{\latentpolicy}{}{\state_2} \right| \leq
      \frac{\localrewardloss{\stationary{\latentpolicy}} {+} \discount \KV
      \localtransitionloss{\stationary{\stationary{\latentpolicy}}}}{1 {-}
      \discount} \fun{\stationary{\latentpolicy}^{-1}\fun{\state_1} {+}
      \stationary{\latentpolicy}^{-1}\fun{\state_2}}, \notag\\
      &\left| \values{\latentpolicy}{\varphi}{\state_1} {-}
      \values{\latentpolicy}{\varphi}{\state_2} \right| \leq \frac{\discount
      \localtransitionloss{\stationary{\latentpolicy}}}{1 {-} \discount}
      \fun{\stationary{\latentpolicy}^{-1}\fun{\state_1} {+}
      \stationary{\latentpolicy}^{-1}\fun{\state_2}}. \label{eq:value-diff-abstraction-quality}
    \end{align}
Intuitively, when local losses are sufficiently small, then (i) the value difference of states and their embedding that are likely to be seen under a latent policy is also small, and (ii) states with the same embedding have close values. 

\subsection{Checking the Quality of the Abstraction}
While bounding the difference between values offered by $\latentpolicy \in \latentmpolicies$ in $\mdp$ and $\latentmdp$ is theoretically possible using $\localrewardloss{\stationary{\latentpolicy}}$ and $\localtransitionloss{\stationary{\latentpolicy}}$, we need to accurately approximate these losses from samples to further offer practical guarantees (recall that $\mdp$ is unknown).
Although the agent is able to produce execution traces by
interacting with $\mdp$,
%================================================================
estimating the expectation over the Wasserstein is intractable.
Intuitively, even if approximating Wasserstein from samples is possible (e.g., \citealp{DBLP:conf/aistats/GenevayCBCP19}), 
this would require access to a generative model for $\probtransitions\fun{\sampledot \mid \state, \action}$ (e.g., \citealp{DBLP:journals/ml/KearnsMN02}) from which we would have to draw a sufficient number of samples for each $\state, \action$ drawn from $\stationary{\latentpolicy}$ to then be able to estimate the expectation. 
\citet{DBLP:conf/icml/GeladaKBNB19} overcome this issue by assuming a deterministic MDP%
, which allows optimizing an approximation of $\localtransitionloss{\stationary{\latentpolicy}}$ through gradient descent.
To deal with general MDPs, we study
an upper bound on $\localtransitionloss{\stationary{\latentpolicy}}$ that can be efficiently approximated from samples:
\begin{align*}
    &\localtransitionloss{\stationary{\latentpolicy}} \leq \expectedsymbol{\state, \latentaction, \state' \sim \stationary{\latentpolicy}} \wassersteindist{\distance_{\latentstates}}{\embed\fun{\sampledot \mid \state'}}{\latentprobtransitions\fun{\sampledot \mid \embed\fun{\state}, \latentaction}} = \localtransitionlossupper{\stationary{\latentpolicy}}, \label{eq:transitionloss-bounds}
\end{align*}
where $\stationary{\latentpolicy}(\state, \latentaction, \state')=  \stationary{\latentpolicy}(\state, \latentaction) \cdot \probtransitions(\state' \mid \state, \latentaction)$
and $\embed(\latentstate \mid \state) = \condition{=}(\embed\fun{\state}, \latentstate)$. 
We now provide \emph{probably approximately correct} (PAC) guarantees for estimating local losses for discrete latent models, derived from 
the
\ifappendix
\citeauthor{10.2307/2282952:hoeffding}'s inequalities.
\else
Hoeffding's inequalities.
\fi
\begin{lemma}\label{lemma:pac-loss}
Suppose $\latentmdp$ is discrete and the agent interacts with $\mdp$ by executing $\latentpolicy \in \latentpolicies$, thus producing $ \tuple{\seq{\state}{T}, \seq{\latentaction}{T - 1}, \seq{\reward}{T - 1}} \sim \stationary{\latentpolicy}$.
Let $\error, \proberror \in \mathopen]0, 1\mathclose[$ and denote by
\begin{align*}
    \localrewardlossapprox{\stationary{\latentpolicy}} &= \frac{1}{T} \sum_{t = 0}^{T - 1} \left| \reward_t - \latentrewards\fun{\embed\fun{\state_t}, \latentaction_t}  \right|, \\
    \localtransitionlossapprox{\stationary{\latentpolicy}} &= \frac{1}{T} \sum_{t = 0}^{T - 1} \left[ 1 - \latentprobtransitions\fun{\embed\fun{\state_{t + 1}} \mid \embed\fun{\state_t}, \latentaction_t} \right].
\end{align*}
Then after $T \geq \left\lceil \nicefrac{- \log\fun{\frac{\proberror}{4}}}{2
  \error^2} \right\rceil$ steps,
    $\left| \localrewardloss{\stationary{\latentpolicy}} -
    \localrewardlossapprox{\stationary{\latentpolicy}}  \right| \leq \error$ and
    $\left| \localtransitionlossupper{\stationary{\latentpolicy}} -
    \localtransitionlossapprox{\stationary{\latentpolicy}} \right| \leq
    \error$
with probability at least $1 - \proberror$.
\end{lemma}
This yields the following Theorem, finally allowing to check the  abstraction quality as a bounded value difference.
\begin{theorem}\label{thm:pac-qvaluedif}
    Let $\latentpolicy \in \latentmpolicies$ and assume $\latentmdp$ is discrete and $\langle \KR{\latentpolicy}, \KP{\latentpolicy} \rangle$-Lipschitz.
    Let $\labelset{C}, \labelset{T} \subseteq \atomicprops$, $\varphi \in \set{\until{\labelset{C}}{\labelset{T}}, \eventually \labelset{T}}$, and $\KV = \min\fun{\nicefrac{\Rmax{\latentpolicy}}{1 - \discount}, \nicefrac{\KR{\latentpolicy}}{1 - \discount \KP{\latentpolicy}}}$.
    Let $\error, \proberror \in \mathopen]0, 1\mathclose[$ and $\stationary{\latentpolicy}$ be the stationary distribution of $\mdp_{\latentpolicy}$.
    Then, after $T \geq \left\lceil \nicefrac{- \log\fun{\frac{\proberror}{4}} \fun{1 +
    \discount \KV}^2}{2\error^2\fun{1 - \discount}^2}\right\rceil$ interaction steps through $\stationary{\latentpolicy}$,
    \begin{align*}
        \expectedsymbol{\state, \latentaction \sim \stationary{\latentpolicy}}
        \left| \qvalues{\latentpolicy}{}{\state}{\latentaction} -
        \latentqvalues{\latentpolicy}{}{\embed\fun{\state}}{\latentaction}
        \right| & \leq
        \frac{\localrewardlossapprox{\stationary{\latentpolicy}} + \discount
        \KV \localtransitionlossapprox{\stationary{\latentpolicy}}}{1 -
      \gamma} + \error, \\
        \expectedsymbol{\state, \latentaction \sim \stationary{\latentpolicy}} \left| \qvalues{\latentpolicy}{\varphi}{\state}{\latentaction} - \latentqvalues{\latentpolicy}{\varphi}{\embed\fun{\state}}{\latentaction} \right| & \leq \frac{\discount \localtransitionlossapprox{\stationary{\latentpolicy}}}{1 - \discount} + \frac{\discount \error}{1 + \discount \KV}
    \end{align*}
    with probability at least $1 - \proberror$.
\end{theorem}

%% file: new_vae.tex
We now provide a framework based on \emph{variational autoencoders} \citep{DBLP:journals/corr/KingmaW13} that allows us to learn a discrete latent space model of $\mdp$ through the interaction of the agent executing a pre-learned RL policy $\policy \in \mpolicies{\mdp}$ with the environment.
Concretely, we seek a discrete latent space model $\tuple{\latentmdp_{\decoderparameter}, \embed_{\encoderparameter}, \embeda_{\encoderparameter, \decoderparameter}}$ such that $ \latentmdp_{\decoderparameter} = \tuple{\latentstates, \latentactions, \latentprobtransitions_{\decoderparameter}, \latentrewards_{\decoderparameter}, \latentlabels_{\decoderparameter}, \atomicprops, \zinit}$.
We propose to learn the parameters $\tuple{\encoderparameter, \decoderparameter}$ of an \emph{encoder} $\encoder$ and a \emph{behavioral model} $\decoder$ from which we can retrieve
(i) the embedding functions
$\embed_{\encoderparameter}$ and $\embeda_{\encoderparameter, \decoderparameter}$,
(ii) the latent MDP components $\latentprobtransitions_{\decoderparameter}$, $\latentrewards_{\decoderparameter}$, and $\latentlabels_{\decoderparameter}$,
and (iii) a latent policy $\latentpolicy_\decoderparameter \in \latentmpolicies$, via
$
    \min_{\decoderparameter} \divergence{\mdp_\policy}{\decoder},
$
where $\divergencesymbol$ is a discrepancy measure.
Intuitively, the end goal is to learn
(i) a discrete representation of $\states$ and $\actions$
(ii) to mimic the behaviors of the original MDP over the induced latent spaces, thus yielding a latent MDP with a bisimulation distance close to $\mdp$, and
(iii) to \emph{distill} $\policy$ into $\latentpolicy_{\decoderparameter}$. 
In the following, we distinguish the case where we only learn $\latentstates$, with $\latentactions = \actions$ (in that case, $\actions$ is assumed to be discrete), and the one where we additionally need to discretize the set of actions and learn $\latentactions$.

\subsection{Evidence Lower Bound}
In this work, we focus on the case where $\dklsymbol$ is used as discrepancy measure: the goal is optimizing $\min_{\decoderparameter}
\dklsymbol(\mdp_{\policy} \parallel \decoder)$ or equivalently maximizing the marginal
\emph{log-likelihood of traces} of $\mdp$, i.e.,
$\expected{\trace \sim \mdp_\policy}{\log \decoder\fun{\trace}}$, where
\begin{equation}
    \decoder\fun{\trace} =
            \int_{\inftrajectories{\latentmdp_{{\decoderparameter}}}} \decoder\fun{\trace \mid \seq{\latentvariable}{T}} \, d\latentprobtransitions_{\latentpolicy_{ \decoderparameter}}\fun{\seq{\latentvariable}{T}}, \label{eq:maximum-likelihood}
\end{equation}
$\trace = \langle \seq{\state}{T}, \seq{\action}{T - 1}, \seq{\reward}{T - 1}, \seq{\labeling}{T} \rangle$,
$\latentvariable \in \latentvariables$ with $\latentvariables = \latentstates$ if $\latentactions = \actions$ and $\latentvariables = \latentstates \times \latentactions$ otherwise,
$\latentprobtransitions_{\latentpolicy_{\decoderparameter}}(\seq{\latentstate}{T}) = \prod_{t = 0}^{T - 1} \latentprobtransitions_{\latentpolicy_{\decoderparameter}}(\latentstate_{t + 1} \mid \latentstate_t)$, and $\latentprobtransitions_{\latentpolicy_{ \decoderparameter}}(\seq{\latentstate}{T}, \seq{\latentaction}{T-1}) = \prod_{t = 0}^{T - 1} {\latentpolicy_{\decoderparameter}}(\latentaction_t \mid \latentstate_t) \cdot \latentprobtransitions_{\decoderparameter}(\latentstate_{t + 1} \mid \latentstate_t, \latentaction_t)$.
The dependency of $\trace$ on $\latentvariables$ in Eq.~\ref{eq:maximum-likelihood}
is made explicit by the law of total probability.

Optimizing $\expected{\trace \sim \mdp_\policy}{\log \decoder\fun{\trace}}$ through Eq.~\ref{eq:maximum-likelihood} is typically intractable \citep{DBLP:journals/corr/KingmaW13}.
To overcome this, we use an encoder $\encoder\fun{\seq{\latentvariable}{T} \mid \trace}$ to set up
a \emph{lower bound} on the log-likelihood of produced traces,
often referred to as \emph{evidence lower bound} (ELBO, \citealp{DBLP:journals/jmlr/HoffmanBWP13}):
\begin{align*}
    &\log \decoder\fun{\trace} - \dkl{\encoder\fun{\sampledot \mid \trace}}{{\decoder\fun{\sampledot \mid \trace}}} \notag \\
    = & \expected{\seq{\latentvariable}{T} \sim \encoder\fun{\sampledot \mid \trace}}{\log \decoder\fun{\trace \mid \seq{\latentvariable}{T}}}  - \dkl{\encoder\fun{\sampledot \mid \trace}}{\latentprobtransitions_{\latentpolicy_{ \decoderparameter}}}.
\end{align*}
The purpose of optimizing the ELBO is twofold.
First, this allows us learning $\embed_{\encoderparameter}$ via (i) $\encoder(\seq{\latentstate}{T} \mid \trace) = \prod_{t = 0}^{T} \embed_{\encoderparameter}(\latentstate_{t} \mid \state_t)$ or (ii) $\encoder(\seq{\latentstate}{T}, \seq{\latentaction}{T - 1} \mid \trace) = \encoder(\seq{\latentaction}{T-1} \mid \seq{\latentstate}{T}, \trace) \cdot \encoder(\seq{\latentstate}{T} \mid \trace)$, where $\encoder\fun{\seq{\latentaction}{T-1} \mid \seq{\latentstate}{T}, \trace} = \prod_{t = 0}^{T - 1} \latentembeda_{\encoderparameter}(\latentaction_t \mid \latentstate_t, \action_t)$, $\latentembeda_\encoderparameter$
being an
action encoder.
We assume here that
encoding states and actions to latent spaces is independent of rewards.
We additionally make them independent of $\seq{\labeling}{T}$ by assuming that
$\labels$ is known.
This allows $\embed_\encoderparameter$ to encode states and their labels directly into the latent space (cf. Sect.~\ref{subsec:discrete-latent-distr}).

Second, we assume the existence of latent reward and label models, i.e., $\decoder^{\rewards}$ and $\decoder^{\labels}$, allowing to recover respectively $\latentrewards_{\decoderparameter}$ and $\latentlabels_{\decoderparameter}$, as well as a \emph{generative model} $\decoder^{\generative}$, enabling the reconstruction of states and actions.
This allows decomposing the behavioral model $\decoder$ into:
\begin{multline*}
\quad \quad \begin{aligned}
    &\decoder\fun{\seq{\state}{T}, \seq{\action}{T - 1}, \seq{\reward}{T - 1}, \seq{\labeling}{T} \mid \seq{\latentvariable}{T}} \\
    =& \decoder^{\generative}(\seq{\state}{T} \mid \seq{\latentstate}{T}) 
    \cdot \decoder^{\generative}(\seq{\action}{T - 1} \mid \seq{\latentvariable}{T}) \\
    & \cdot \decoder^{\rewards}(\seq{\reward}{T - 1} \mid \seq{\latentstate}{T}, \seq{\latentaction}{T - 1}) \cdot \decoder^{\labels}(\seq{\labeling}{T} \mid \seq{\latentstate}{T}),
\end{aligned}\\
\begin{aligned}
    \text{where }\decoder^{\generative}(\seq{\state}{T} \mid \seq{\latentstate}{T}) &= \textstyle{\prod_{t = 0}^{T}} \, \decoder^{\generative}(\state_t \mid \latentstate_t), \\
    \decoder^{\generative}(\seq{\action}{T - 1} \mid \seq{\latentvariable}{T - 1}) &=
    \begin{cases}
        \textstyle{\prod_{t = 0}^{T - 1}}\, \latentpolicy_{\decoderparameter}(\action_t \mid \latentstate_t) \text{ if } \latentactions = \actions \\
        \prod_{t = 0}^{T - 1}\, \embeda_{\decoderparameter}(\action_t \mid \latentstate_t, \latentaction_t) \text{ else,}
    \end{cases} \\
    \decoder^{\rewards}(\seq{\reward}{T - 1} \mid \seq{\latentstate}{T}, \seq{\latentaction}{T - 1}) &= \textstyle{\prod_{t = 0}^{T - 1}} \, \decoder^{\rewards}(\reward_{t} \mid \latentstate_t, \latentaction_t), \text{and} \\
    \decoder^{\labels}(\seq{\labeling}{T} \mid \seq{\latentstate}{T}) &= \textstyle{\prod_{t = 0}^{T}} \, \decoder^{\labels}(\labeling_t \mid {\latentstate}_{t}).
\end{aligned}
\end{multline*}
Model $\embeda_{\decoderparameter}$
allows learning the action embedding function via $\embeda_{\encoderparameter, \decoderparameter}\fun{\action \mid \state, \latentaction} = \expectedsymbol{\latentstate \sim \embed_{\encoderparameter}\fun{\sampledot \mid \state}}\embeda_{\decoderparameter}\fun{\action \mid \latentstate, \latentaction}$ for all $\state \in \states, \action \in \actions, \latentaction \in \latentactions$.
We also argue that a \emph{perfect reconstruction} of labels is possible, i.e., $\decoder^{\labels}\fun{\seq{\labeling}{T} \mid \seq{\latentstate}{T}} = 1$, due to the labels being encoded into the latent space.
From now on, we thus omit the label term.

Deterministic embedding functions $\embed_{\encoderparameter}$, $\embeda_{\encoderparameter, \decoderparameter}$
and $\latentrewards_{\decoderparameter}$
can finally be obtained by
taking the mode of their distribution.%

\smallparagraph{Back to the local setting.}~Taking Assumption~\ref{assumption:ergodicity} into account, drawing multiple finite traces $\trace \sim \mdp_{\policy}$ can be seen as a continuous interaction with $\mdp$ along an infinite trace \citep{DBLP:conf/nips/Huang20}.
\ifappendix
This observation allows us to formulate the ELBO in the local setting\footnote{The expert reader might note this is not the standard definition of ELBO, we prove in Appendix (Corollary~\ref{cor:elbo-local}) that this reformulation is valid because of developments in Sect.~\ref{section:deepmdp}.}
\else
This observation allows us to formulate the ELBO in the local setting
\fi
and connect to local losses:
\[
%\begin{align*}
    \max_{\encoderparameter, \decoderparameter}\; \elbo\fun{\latentmdp_{\decoderparameter}, \embed_{\encoderparameter}, \embeda_{\encoderparameter, \decoderparameter}} = - \min_{\encoderparameter, \decoderparameter} \left\{\distortion_{\encoderparameter, \decoderparameter} + \rate_{\encoderparameter, \decoderparameter} \right\},
%\end{align*}
\]
where $\distortion$ and $\rate$ denote respectively the \emph{distortion} and \emph{rate} of the variational model  \citep{DBLP:conf/icml/AlemiPFDS018}, given by
\begin{equation*}
    \distortion_{\encoderparameter, \decoderparameter} = -
    \left\lbrace\begin{array}{@{}r@{\quad}l@{}}
        \displaystyle \expectedsymbol{}_{\substack{\state, \action, \reward, \state' \sim \stationary{\policy} \\ \latentstate, \latentstate' \sim \embed_{\encoderparameter}\fun{\sampledot \mid \state, \state'}}} \big[
        \begin{aligned}[t]
        & \log \decoder^{\generative}(\state' \mid \latentstate') + \log \latentpolicy_{\decoderparameter}(\action \mid \latentstate) +
        \\
        & \log \decoder^{\rewards}\fun{\reward \mid \latentstate, \action}\big]
        \text{ if } \latentactions = \actions,
        \end{aligned} \hfill \\
        \displaystyle \expectedsymbol{}_{\substack{\state, \action, \reward, \state' \sim \stationary{\policy} \\ \latentstate, \latentstate' \sim \embed_{\encoderparameter}\fun{\sampledot \mid \state, \state'} \\ \latentaction \sim \latentembeda_{\encoderparameter}\fun{\sampledot \mid \latentstate, \action}}}
        \big[
        \begin{aligned}[t]
        & \log \decoder^{\generative}\fun{\state' \mid \latentstate'} + \\
        & \log \embeda_{\decoderparameter}\fun{\action \mid \latentstate, \latentaction} + \\
        & \log \decoder^{\rewards}\fun{\reward \mid \latentstate, \latentaction}\big] \text{ else, and}
        \end{aligned} \hfill
        \end{array}
    \right.
\end{equation*}
\begin{equation*}
    \rate_{\encoderparameter, \decoderparameter} =
    \left\lbrace\begin{array}{@{}r@{\quad}l@{}}
        \displaystyle \expectedsymbol{}_{\substack{\state, \action, \state' \sim \stationary{\policy} \\ \latentstate \sim \embed_{\encoderparameter}(\sampledot \mid \state)}}  \dkl{\embed_{\encoderparameter}\fun{\sampledot \mid \state'}}{\latentprobtransitions_{\latentpolicy_{\decoderparameter}}(\sampledot \mid \latentstate)} \text{ if } \latentactions = \actions, \hfill \\
        \displaystyle \expectedsymbol{}_{\substack{\state, \action, \state' \sim \stationary{\policy} \\ \latentstate \sim \embed_{\encoderparameter}\fun{\sampledot \mid \state} \\ \latentaction \sim \latentembeda_{\encoderparameter}\fun{\sampledot \mid \latentstate, \action} }}
        \big[
        \begin{aligned}[t]
        & \dkl{\embed_{\encoderparameter}\fun{\sampledot \mid \state'}}{\latentprobtransitions_{\decoderparameter}\fun{\sampledot \mid \latentstate, \latentaction}} + \\
        & \dkl{\latentembeda_{\encoderparameter}\fun{\sampledot\mid\latentstate, \action}}{\latentpolicy_{\decoderparameter}\fun{\sampledot \mid \latentstate}} \big] \text{ else.}
        \end{aligned} \hfill
    \end{array}
    \right.
\end{equation*}
We omit the subscripts when the context is clear.
The optimization of $\elbo\fun{\latentmdp_{\decoderparameter}, \embed_{\encoderparameter}, \embeda_{\encoderparameter, \decoderparameter}}$ allows for an indirect optimization of the local losses through their variational versions:
(i) $\localrewardloss{\stationary{\latentpolicy}}$ via the log-likelihood of rewards produced, and 
(ii) $\localtransitionlossupper{\stationary{\latentpolicy}}$ where we change the Wasserstein term to the KL divergence.
Note that this last change means we do not necessarily obtain the theoretical guarantees on the quality of the abstraction via its optimization.
Nevertheless, our experiments indicate KL divergence is a good proxy of the Wasserstein term in practice.
In particular, in the discrete setting, Wasserstein matches TV and one can relate the proxy with the original metric using the Pinkster’s inequality.

\subsection{VAE Distributions}\label{subsec:discrete-latent-distr}
\smallparagraph{Discrete distributions.}~We aim at learning \emph{discrete} latent spaces $\latentstates$ and $\latentactions$, the distributions $\embed_{\encoderparameter}$, $\encoder^{\actions}$, $\latentprobtransitions_{\decoderparameter}$, and $\latentpolicy_{\decoderparameter}$ are thus supposed to be discrete.
Two main challenges arise:
(i) gradient descent is not applicable to learn $\encoderparameter$ and $\decoderparameter$ due to the discontinuity of $\latentstates$ and $\latentactions$, and (ii) sampling from these distributions must be a \emph{derivable operation}.
We overcome these by using \emph{continuous relaxation of Bernoulli distributions} to learn a binary representation of the latent states, and the \emph{Gumbel softmax trick} for the latent action space \citep{DBLP:conf/iclr/JangGP17,DBLP:conf/iclr/MaddisonMT17}.

\smallparagraph{Labels.}~To enable $\log \decoder^\labels\fun{\seq{\labeling}{T} \mid \seq{\latentstate}{T}} = 0$, we linearly encode $\labels\fun{\state_t} = \labeling_t$ into each $\latentstate_t$ via $\embed_{\encoderparameter}$.
Recall that labels are binary encoded, so we allocate them $|\atomicprops|$ bits in $\latentstates$.
Then, $\embed_{\encoderparameter}\fun{\latentstate \mid \state} > 0$ implies $\labels\fun{\state} = \latentlabels_\decoderparameter\fun{\latentstate}$, for all $\state \in \states, \latentstate \in \latentstates$,
\iffalse
This enforces the following property of $\embed_{\encoderparameter}$:
\[
\forall \state \in \states, \, \forall \latentstate \in \latentstates,\, \embed_{\encoderparameter}\fun{\latentstate \mid \state} > 0 \implies \labels\fun{\state} = \latentlabels_\decoderparameter\fun{\latentstate},
\]
\fi
satisfying Assumption~\ref{assumption:labeling} if $\embed_{\encoderparameter}$ is deterministic.

\smallparagraph{Decoders.}~For $\decoder^{\generative}$, $\embeda_{\decoderparameter}$, and $\decoder^{\rewards}$, we learn the parameters of multivariate normal distributions.
This further allows linking all $\latentstate \in \latentstates$ to the parameters of $\decoder^{\generative}(\sampledot \mid \latentstate)$ for explainability.

\subsection{Posterior Collapse}\label{sec:posterior-collapse}
A common issue encountered while optimizing variational models via the ELBO is 
\emph{posterior collapse}.
Intuitively, this results in a degenerate local optimum where the model learns to ignore the latent space.
With a discrete encoder, this translates into a deterministic mapping to a single latent state, regardless of the input.
From an information-theoretic point of view, optimizing the ELBO gives way to a trade-off between the minimization of
$\rate$ and $\distortion$, where the feasible region is a convex set \citep{DBLP:conf/icml/AlemiPFDS018}.
Posterior collapse occurs when $\rate \approx 0$
(\emph{auto-decoding limit}).
On the other hand, one can achieve $\distortion \approx 0$
(\emph{auto-encoding limit})
at the price of a higher rate.

\smallparagraph{Regularization terms.}~%
Various solutions have been proposed in the literature to prevent posterior collapse.
They include \emph{entropy regularization} (via $\alpha \in \mathopen[ 0, \infty \mathclose[$, e.g.,
\ifappendix
\citealt{DBLP:conf/corl/BurkeHR19,DBLP:conf/icml/DongS0B20}) and
\else
\citealt{DBLP:conf/icml/DongS0B20}) and
\fi
\emph{KL-scaling} (via $\beta \in \mathopen[0, 1\mathclose]$, e.g., \citealt{DBLP:conf/icml/AlemiPFDS018}),
consisting in changing $\elbo$ to
$\tuple{\alpha, \beta}\text{-}\elbo = - (\distortion + \beta \cdot \rate) + \alpha \cdot \entropy{\encoder}$, where 
$\entropy{\encoder}$ denotes the entropy of an encoding distribution.
We choose to measure the entropy of the \emph{marginal} encoder, given by $\encoder\fun{\latentstate} = \expectedsymbol{\state \sim \stationary{\policy}} \embed_{\encoderparameter}\fun{\latentstate \mid \state}$. 
Intuitively, this encourages the encoder to learn to make plenty use of the latent space. 
The parameter $\beta$ allows to interpolate between auto-encoding and auto-decoding behavior, which is not possible with the standard ELBO objective.

A drawback of these methods is that we no longer optimize a lower bound on the log-likelihood of the input while optimizing
$\tuple{\alpha, \beta}\text{-}\elbo$.
In practice, setting up annealing schemes for $\alpha$ and $\beta$ allows to eventually recover $\elbo$ and avoid posterior collapse ($\alpha=0$ and $\beta=1$ matches $\elbo$).

\begin{figure}[t]
    \centering
    \includegraphics[width=\linewidth]{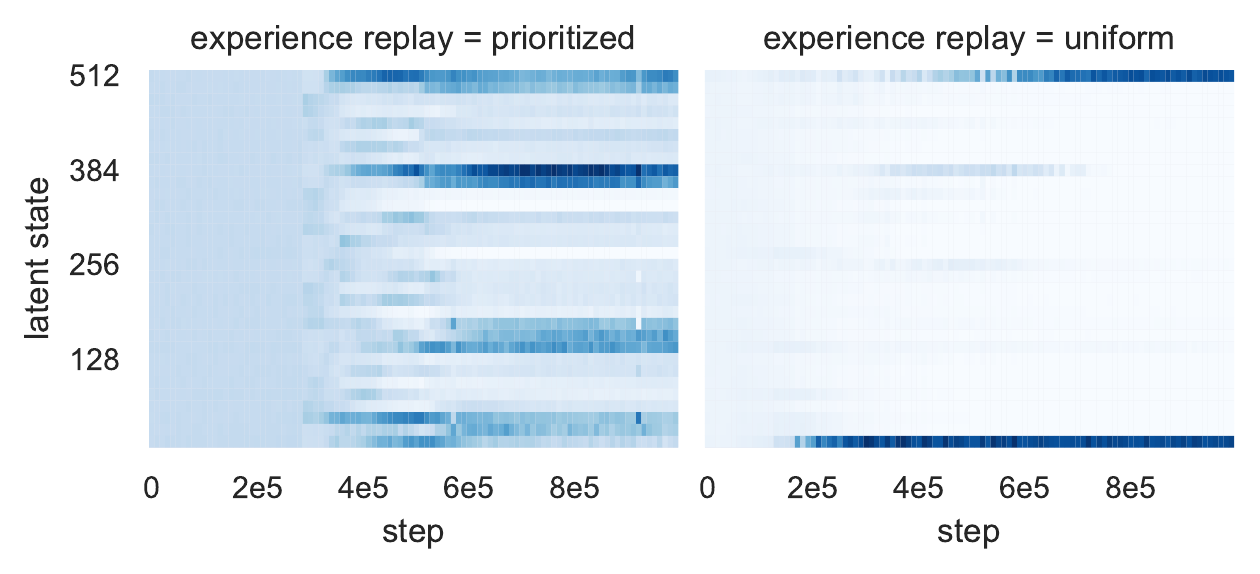}
    \caption{Latent space distribution along training steps for the CartPole environment. The intensity of the blue hue corresponds to the frequency of latent states produced by $\embed_{\encoderparameter}$ during training. We compare a bucket-based prioritized against a simple uniform experience replay.
    The latent space learned via the uniform buffer collapses to two latent states.}
    \label{fig:prioritized-replay}
\end{figure}
\smallparagraph{Prioritized replay buffers.}~%
To enable meaningful use of latent space to represent input state-actions pairs,
$\encoder$ should learn to (i) exploit the entire latent space, and (ii) encode wisely states and actions of transitions yielding poor ELBO, being generally sensitive to bad representation embedding.
This motivates us to use a \emph{prioritized replay buffer} \citep{DBLP:journals/corr/SchaulQAS15} to store transitions and sample them when optimizing $\elbo$.
Draw $\tuple{\state, \action, \reward, \state'} \sim \stationary{\policy}$ and let $p_{\tuple{\state, \action, \reward, \state'}}$ be its priority, we introduce the following priority functions.

\begin{itemize}
    \item \emph{Bucket-based priority:}~we partition the buffer in $|\latentstates|$ \emph{buckets}.
    Let $N \in \N$ and $b \colon \latentstates \to \N$ be respectively step and latent state counters.
    At each step, let $\latentstate \sim \embed_{\encoderparameter}\fun{\sampledot \mid \state}$, we assign $p_{\tuple{\state, \action, \reward, \state'}}$ to $\nicefrac{N}{b\fun{\latentstate}}$, then increment $N$ and $b\fun{\latentstate}$ of one.
    This allows $\embed_\encoderparameter$ to process states being infrequently visited under $\policy$ and learn to fairly distribute $\latentstates$.
    \item \emph{Loss-based priority:}~we set $p_{\tuple{\state, \action, \reward, \state'}}$ to its individual transition loss,
    which enables to learn improving the representation of states and actions that yield poor ELBO.
\end{itemize}

%% file: experiments.tex
The goal of our experiments{\ifappendix\footnote{The code for conducting the experiments is available at \url{https://github.com/florentdelgrange/vae_mdp}}\fi} is to evaluate the quality of the latent space model learned and the policy distilled via our VAE-MDP framework. This evaluation consists of:
    an analysis of the training of the latent space model 
    and the benefits of our method to avoid posterior collapse,
    assessing the quality of the abstraction learned via PAC local losses bounds, and
    testing the performance of the distilled policy. 
    This allows to assess if the latent model learned yields a sound compression of the state-action space that retains the necessary information to optimize the return.
We {evaluate} our method on classic OpenAI environments \citep{DBLP:journals/corr/BrockmanCPSSTZ16} with (i) continuous states and discrete actions (CartPole, MountainCar, and Acrobot), and (ii) where both, states and actions, are continuous (Pendulum and LunarLander).
We distill RL policies $\policy$ learned via DQN \citep{DBLP:journals/nature/MnihKSRVBGRFOPB15} for case (i), and SAC \citep{DBLP:conf/icml/HaarnojaZAL18} for case (ii).
\ifappendix
The exact setting and labeling functions used to reproduce our results 
are provided in supplementary material (Appendix~\ref{appendix:experiments}).
\fi

\smallparagraph{Latent spaces.}~Since latent spaces are trained to enable formal verification, we choose $\log_2|\latentstates|$ and $|\latentactions|$ ranging from $9$ bits and $2$ actions (coarser) to $16$ bits and $5$ actions (finer) to make them tractable for model checkers --- see \citet{DBLP:conf/isola/BuddeHKKPQTZ20} for a performance comparison of modern tools for a range of instances with large model size. %
Depending on the property to verify, the latent space may have to be finer since
(i) we need to reserve $|\atomicprops|$ bits in the representation of $\latentstates$ for labels, and
(ii) this allows the agent to take more precise decisions over a finer partition of the latent space. % 
We then select the model with the best trade-off between abstraction quality (measured via $\localrewardlossapprox{\stationary{\latentpolicy_{\decoderparameter}}}$ and $\localtransitionlossapprox{\stationary{\latentpolicy_{\decoderparameter}}}$) and performance (i.e., the return approximated by running $\latentpolicy$ in $\mdp$).

\smallparagraph{ELBO optimization.}~In order to allow $\elbo$ to be trained efficiently, posterior collapse has to be tackled from the very first stages of training. In fact, we found that the KL-scaling and entropy regularization annealing schemes (cf. Sect.~\ref{sec:posterior-collapse}) were necessary to avoid the latent space to collapse into a single state-action pair after only a few training steps.
We found the most efficient to start with an autoencoder behavior ($\beta_0=0$) and a large entropy regularizer (e.g., $\alpha_0=10$) during the $10^4$ first steps, and then anneal them via $\alpha_t = \alpha_0 \cdot (1 - \tau)^t$, $\beta_t = 1 - (1 - \tau)^t$ (e.g., $\tau = 10^{-5}$) to fully recover the original $\elbo$ in a second training phase.
We also found that prioritized experience replays further prevent posterior collapse during training (cf. Fig.~\ref{fig:prioritized-replay}).
Fig.~\ref{subfig:dist-rate-elbo} shows that training $\elbo$ this way results in a stable learning procedure that successfully minimizes $\distortion$ while preventing an auto-decoding behavior by keeping $\rate$ away from $0$.%

\smallparagraph{Local losses.}~We compute the PAC local losses bounds presented in Sect.~\ref{section:deepmdp} along training steps (Fig.~\ref{subfig:local-losses}).
When the policy is eventually distilled (Fig.~\ref{subfig:policy-evaluation}) and $\latentpolicy_{\decoderparameter}$ achieves performance similar to those of $\policy$, the learning curves stabilize, while maximizing $\elbo$ successfully allows minimizing the local losses.
In every environment where we tested our method, a low reward loss is maintained while the transition loss reaches values in the interval $\mathopen[\nicefrac{1}{5}, \nicefrac{3}{5}\mathclose]$ for most of the instances.
These values are thus guaranteed to upper bound the expected bisimulation distance between $\mdp$ and $\latentmdp$ as well as their value difference. %(cf. Sect.~\ref{section:deepmdp}).
We do not expect our approach to reach zero reward and transition loss though, since we pass from continuous to discrete spaces: the abstraction induced by our approach is always coarser than the original spaces, which translates in general to precision loss.
This precision loss is often encoded through the discrete probability transitions. % $\latentprobtransitions_{\decoderparameter}$:
For instance, a state $\state$ which deterministically transitions to a close state $\state'$ in $\mdp$ such that $\embed_{\encoderparameter}\fun{\state} = \latentstate = \embed_{\encoderparameter}\fun{\state'}$ induces $\latentprobtransitions_{\latentpolicy_{ \decoderparameter}}\fun{\latentstate \mid \latentstate} > 0$.
Observe that the PAC computation of $\localtransitionlossapprox{\stationary{\latentpolicy_{ \decoderparameter}}}$ is sensitive  to the entropy of $\latentprobtransitions_{\decoderparameter}$, which thus may induce a residual transition loss.
$\localtransitionloss{\stationary{\latentpolicy_{ \decoderparameter}}}$ is on the contrary not sensitive to this.
In practice, its value will thus often be lower than its approximated upper bound $\localtransitionlossapprox{\stationary{\latentpolicy_{ \decoderparameter}}}$.

\smallparagraph{Policy distillation.}~The guarantees derived in Sect.~\ref{section:deepmdp} are only valid for $\latentpolicy_{\decoderparameter}$, the policy under which formal properties can be verified.
We further need evaluate if it achieves sound performance in the RL environment: 
checking the expected value difference via Thm.~\ref{thm:pac-qvaluedif} is most significant when $\values{\latentpolicy_{\decoderparameter}}{}{\state}$ is worth in any state $\state$ likely to be reached in $\mdp_{\latentpolicy_{\decoderparameter}}$.
We compare the
episode return achieved by $\latentpolicy_{\decoderparameter}$ against $\policy$ in the environment $\mdp$ 
(Fig.~\ref{subfig:policy-evaluation}).
Our framework allows learning to improve  $\latentpolicy_{ \decoderparameter}$ along training steps and eventually achieving the return of the original policy $\policy$.
This illustrates that training our latent model enables distilling $\policy$ into a latent policy $\latentpolicy_{\decoderparameter}$ that achieves similar performance in $\mdp$. 

%% file: conclusion.tex
In this work, we presented VAE-MDPs, a framework for learning discrete latent models of unknown, continuous-spaces environments with bisimulation guarantees.
We detailed how such latent models can be learned by executing an RL policy in the environment and showed that the procedure yields a distilled version of the latter as a side effect.
To provide the guarantees, we introduced new local losses bounds aimed at discrete latent spaces with their PAC-efficient approximation algorithm derived from the execution of the distilled policy.
All this enables the verification of RL policies for unknown continuous MDPs.

Experimental results demonstrate the feasibility of our approach through the PAC bounds and the performance of the distilled policy achieved for various environments.
Our tool can also be used to highlight the lack of robustness of input policies when the distillation fails.

Complementary to safe-RL approaches addressed via formal methods, we emphasize the ability of our tool to be coupled with such algorithms when no model of the environment is known a priori.
The applicability of our method enables its use in future work for real-world case studies and
more complex settings, such as multi-agent systems.

%% file: proofs.tex
% \todo[inline]{Make sure to reshuffle sections to have the right ordering}
\section{Details on Discrepancy Measures}
Let $P, Q \in
\distributions{\measurableset}$ be probability distributions over $\measurableset$ with density functions $p$ and $q$.

\smallparagraph{Wasserstein.}~The Kantorovich-Rubinstein duality allows formulating the Wasserstein distance between $P$ and $Q$ as
    \[
        \wassersteindist{\distance}{P}{Q} = \sup_{f \in \Lipschf{\distance}} \left|{\expectedsymbol{x \sim P} f\fun{x} - \expectedsymbol{y \sim Q} f\fun{y}}\right|,
    \]
    where $\Lipschf{\distance}$ is the set of 1-Lipschitz functions, i.e.,
    $\Lipschf{\distance} = \set{f : \left|f\fun{x} - f\fun{y}\right| \leq
    \distance\fun{x, y}}$.

\smallparagraph{Total Variation.}~One can reformulate the total variation distance between $P$ and $Q$ as
\begin{align*}
\dtv{P}{Q} ={} & \sup_{A \in \borel{\measurableset}} |P\fun{A} - Q\fun{A}|\\
{}={} & \frac{1}{2} \int_{\measurableset} |p\fun{x} - q\fun{x}|\, dx.
\end{align*}
When $\measurableset$ is continuous, TV might thus give an overly strong measure of the numerical differences across probability distributions to all measurable sets since the distance between two point masses $dx, dy$ is $1$ unless $x = y$ \citep{DBLP:conf/uai/FernsPP05}.

\section{Latent Space Models}\label{appendix:section:latent-space-model}
From now on, fix MDP $\mdp = \mdptuple$ with latent space model $\tuple{\latentmdp, \embed, \embeda}, \, \latentmdp = \latentmdptuple$. We write ${\latentvaluessymbol{\latentpolicy}{}}$ for the value function of a policy $\latentpolicy \in \latentmpolicies$ in $\latentmdp$.

\subsection{DeepMDPs}
In the following, we recall DeepMDPs notions (from \citealt{DBLP:conf/icml/GeladaKBNB19}) that are useful to prove bisimumation distance (Lem.~\ref{lemma:bidistance-embedding}, Lem.~\ref{lemma:bidistance-abstraction-quality}) and value difference bounds (Lem.~\ref{lemma:qvalues-diff-bound}, Lem.~\ref{lem:quality-latent-space}). 

\smallparagraph{Smooth-valuations.}
A policy $\latentpolicy \in \latentmpolicies$ is said to be \emph{$\KV$-($\dtvsymbol$-)smooth-valued} if
$\sup_{\latentstate \in \latentstates} \left| \latentvalues{\latentpolicy}{}{\latentstate}\right| \leq \KV$ and if for all $\latentaction \in \latentactions$, $\sup_{\latentstate \in \latentstates} \left| \latentqvalues{\latentpolicy}{}{\latentstate}{\latentaction} \right| \leq \KV.$
Observe that all policies $\latentpolicy \in \latentmpolicies$ are
$(\nicefrac{\Rmax{\latentpolicy}}{1 - \discount})$-smooth-valued, where $\Rmax{\latentpolicy} = \sup_{\latentstate \in \latentstates} \left| \latentrewards_{\latentpolicy}\fun{\latentstate} \right|$.

\smallparagraph{Lipschitzness.}
A policy $\latentpolicy \in \latentmpolicies$ is $\KV$-Lipschitz-valued if for
all $\latentstate_1, \latentstate_2 \in \states$,
\(
    \left|\latentvalues{\latentpolicy}{}{\latentstate_1} - \latentvalues{\latentpolicy}{}{\latentstate_2} \right| \leq \KV\distance_{\latentstates}\fun{\latentstate_1, \latentstate_2}
\)
and for all $\latentaction \in \latentactions$,
\(
    \left|\latentqvalues{\latentpolicy}{}{\latentstate_1}{\latentaction} - \latentqvalues{\latentpolicy}{}{\latentstate_2}{\latentaction} \right| \leq \KV\distance_{\latentstates}\fun{\latentstate_1, \latentstate_2}.
\)
The MDP $\latentmdp$ is \emph{$\langle \KR{\latentpolicy}, \KP{\latentpolicy}\rangle$-Lipschitz} if for all $\latentstate_1, \latentstate_2 \in \latentstates$,
\begin{align*}
    \left|\latentrewards_{\latentpolicy}\fun{\latentstate_1} - \latentrewards_{\latentpolicy}\fun{\latentstate_2}\right| &\leq \KR{\latentpolicy} \distance_{\latentstates}\fun{\latentstate_1, \latentstate_2},\\
    \wassersteindist{\distance_{\latentstates}}{\latentprobtransitions_{\latentpolicy}\fun{\sampledot \mid \latentstate_1}}{ \latentprobtransitions_{\latentpolicy}\fun{\sampledot \mid \latentstate_2}} &\leq \KP{\latentpolicy} \distance_{\latentstates}\fun{\latentstate_1, \latentstate_2}.
\end{align*}
Moreover, all latent policies $\latentpolicy \in \latentmpolicies$ with $\KP{\latentpolicy} \leq \frac{1}{\discount}$ are $(\nicefrac{\KR{\latentpolicy}}{1 - \discount
\KP{\latentpolicy}})$-Lipschitz-valued \citep[Lem.~1]{DBLP:conf/icml/GeladaKBNB19}.
In the discrete setting, we replace the Wasserstein term by TV and we omit the $\distance_{\latentstates}$ terms.

\subsection{Bounded Bisimulation Distance}
\begin{lemma}\label{lemma:bidistance-embedding}
Let $\latentpolicy \in \latentmpolicies$, assume $\latentmdp$ is discrete and $\langle\KR{\latentpolicy}, \KP{\latentpolicy}\rangle$-Lipschitz.
Then, given the induced stationary distribution $\stationary{\latentpolicy}$ in $\mdp$,
\begin{align*}
    \expectedsymbol{\state \sim \stationary{\latentpolicy}} \bidistance^{\rewards}_{\latentpolicy}\fun{\state, \embed\fun{\state}} & \leq  \localrewardloss{\stationary{\latentpolicy}} + \discount \localtransitionloss{\stationary{\latentpolicy}} \frac{\KR{\latentpolicy}}{1 - \discount \KP{\latentpolicy}}, \\
    \expectedsymbol{\state \sim \stationary{\latentpolicy}} \bidistance^{\labels}_{\latentpolicy}\fun{\state, \embed\fun{\state}} & \leq \frac{\discount \localtransitionloss{\stationary{\latentpolicy}}}{1 - \discount}.
\end{align*}
\end{lemma}
\begin{proof}
\citet[Lemma~A.4]{DBLP:conf/icml/GeladaKBNB19} showed that $\bidistance_{\latentpolicy}^{\rewards}$ is $\frac{(1 - \discount) \KR{\latentpolicy}}{1 - \discount \KP{\latentpolicy}}$-Lipschitz in $\latentstates$, whereas $\bidistance_{\latentpolicy}^{\labels}$ is $1$-Lipschitz since it assigns a maximal distance of one whenever states have different labels.
Let $K_{\rewards} = \frac{(1 - \discount) \KR{\latentpolicy}}{1 - \discount \KP{\latentpolicy}}$, $K_{\labels} = 1$ and $\varphi \in \set{\rewards, \labels}$, this allows us deriving the following inequality:
\begin{align*}
    &\expectedsymbol{\state \sim \stationary{\latentpolicy}}{\wassersteindist{\bidistance^{\varphi}_{\latentpolicy}}{\embed\probtransitions_{\latentpolicy}\fun{\sampledot \mid \state}}{\latentprobtransitions_{\latentpolicy}\fun{\sampledot \mid \embed\fun{\state}}}} \\
    =& \expected{\state \sim \stationary{\latentpolicy}}{\inf_{\coupling \in \couplings{\embed\probtransitions_{\latentpolicy}\fun{\sampledot \mid \state}}{\latentprobtransitions_{\latentpolicy}\fun{\sampledot \mid \embed\fun{\state}}}} \sum_{\latentstate_1, \latentstate_2 \in \states} \bidistance^{\varphi}_{\latentpolicy}\fun{\latentstate_1, \latentstate_2} \coupling\fun{\latentstate_1, \latentstate_2}} \\
    \leq& \expected{\state \sim \stationary{\latentpolicy}}{ \inf_{\coupling \in \couplings{\embed\probtransitions_{\latentpolicy}\fun{\sampledot \mid \state}}{\latentprobtransitions_{\latentpolicy}\fun{\sampledot \mid \embed\fun{\state}}}} \sum_{\latentstate_1, \latentstate_2 \in \states}  K_{\varphi} \condition{\neq}\fun{\latentstate_1, \latentstate_2} \coupling\fun{\latentstate_1, \latentstate_2}} \tag{$\bidistance^{\varphi}_{\latentpolicy}$ is $K_{\varphi}$-Lipschitz in $\latentstates$} \\
    =& K_{\varphi} \expected{\state \sim \stationary{\latentpolicy}}{ \inf_{\coupling \in \couplings{\embed\probtransitions_{\latentpolicy}\fun{\sampledot \mid \state}}{\latentprobtransitions_{\latentpolicy}\fun{\sampledot \mid \embed\fun{\state}}}} \sum_{\latentstate_1 \neq \latentstate_2} \coupling\fun{\latentstate_1, \latentstate_2}} \\
    =&  K_{\varphi} \expectedsymbol{\state \sim \stationary{\latentpolicy}}{\dtv{\embed\probtransitions_{\latentpolicy}\fun{\sampledot \mid \state}}{\latentprobtransitions_{\latentpolicy}\fun{\sampledot \mid \embed\fun{\state}}}} \tag{by definition of TV; TV coincides with $\wassersteinsymbol{\condition{\neq}}$} \\
    =&  K_{\varphi} \expected{\state \sim \stationary{\latentpolicy}}{\frac{1}{2} \sum_{\latentstate' \in \latentstates} \left|\embed\probtransitions_{\latentpolicy}\fun{\latentstate' \mid \state} - \latentprobtransitions_{\latentpolicy}\fun{\latentstate' \mid \embed\fun{\state}} \right|} \\
     =&  K_{\varphi} \expected{\state \sim \stationary{\latentpolicy}}{\frac{1}{2} \sum_{\latentstate' \in \latentstates} \left|\expectedsymbol{\latentaction \sim \latentpolicy\fun{\sampledot \mid \state}}\embed\probtransitions\fun{\latentstate' \mid \state, \latentaction} - \expectedsymbol{\latentaction \sim \latentpolicy\fun{\sampledot \mid \state}}\latentprobtransitions\fun{\latentstate' \mid \embed\fun{\state}, \latentaction} \right|} \\
     =&  K_{\varphi} \expected{\state \sim \stationary{\latentpolicy}}{\frac{1}{2} \sum_{\latentstate' \in \latentstates} \left|\expectedsymbol{\latentaction \sim \latentpolicy\fun{\sampledot \mid \state}} \fun{\embed\probtransitions\fun{\latentstate' \mid \state, \latentaction} - \latentprobtransitions\fun{\latentstate' \mid \embed\fun{\state}, \latentaction}} \right|} \\
     \leq& K_{\varphi} \expectedsymbol{\state \sim \stationary{\latentpolicy}}\expected{\latentaction \sim \latentpolicy\fun{\sampledot \mid \state}}{\frac{1}{2} \sum_{\latentstate' \in \latentstates}  \left| \embed\probtransitions\fun{\latentstate' \mid \state, \latentaction} - \latentprobtransitions\fun{\latentstate' \mid \embed\fun{\state}, \latentaction} \right|} \tag{Jensen's  inequality} \\
     \leq& K_\varphi \localtransitionloss{\stationary{\latentpolicy}}.
\end{align*}
Then,
\begin{align*}
    &\expectedsymbol{\state \sim \stationary{\latentpolicy}} \bidistance^{\varphi}_{\latentpolicy}\fun{\state, \embed\fun{\state}} \\
    =& \expected{\state \sim \stationary{\latentpolicy}}{\fun{1 - \discount} \left| \rewards_{\latentpolicy}\fun{\state} - \latentrewards_{\latentpolicy}\fun{\embed\fun{\state}} \right| \cdot \condition{\set{\rewards}}\fun{\varphi} + \discount \wassersteindist{\bidistance_{\latentpolicy}^{\varphi}}{\probtransitions_{\latentpolicy}\fun{\sampledot \mid \state}}{\latentprobtransitions_{\latentpolicy}\fun{\sampledot \mid \embed\fun{\state}}} } \tag{$\labels\fun{\state} = \labels\fun{\embed\fun{\state}}$ by Assumption~\ref{assumption:labeling}} \\
    \leq& \fun{1 - \discount} \localrewardloss{\stationary{\latentpolicy}} \cdot \condition{\set{\rewards}}\fun{\varphi} + \discount \expected{\state \sim \stationary{\latentpolicy}}{ \wassersteindist{\bidistance_{\latentpolicy}^{\varphi}}{\probtransitions_{\latentpolicy}\fun{\sampledot \mid \state}}{\latentprobtransitions_{\latentpolicy}\fun{\sampledot \mid \embed\fun{\state}}} } \\
    \leq& \fun{1 - \discount} \localrewardloss{\stationary{\latentpolicy}} \cdot \condition{\set{\rewards}}\fun{\varphi} \\
    &+ \discount \expected{\state \sim \stationary{\latentpolicy}}{ \wassersteindist{\bidistance_{\latentpolicy}^{\varphi}}{\probtransitions_{\latentpolicy}\fun{\sampledot \mid \state}}{\embed\probtransitions_{\latentpolicy}\fun{\sampledot \mid {\state}}} + \wassersteindist{\bidistance_{\latentpolicy}^{\varphi}}{\embed\probtransitions_{\latentpolicy}\fun{\sampledot \mid \state}}{\latentprobtransitions_{\latentpolicy}\fun{\sampledot \mid \embed\fun{\state}}} } \tag{triangular inequality} \\
    \leq& \fun{1 - \discount} \localrewardloss{\stationary{\latentpolicy}} \cdot \condition{\set{\rewards}}\fun{\varphi} + \discount K_\varphi \localtransitionloss{\stationary{\latentpolicy}} + \discount \expectedsymbol{\state \sim \stationary{\latentpolicy}}{ \wassersteindist{\bidistance_{\latentpolicy}^{\varphi}}{\probtransitions_{\latentpolicy}\fun{\sampledot \mid \state}}{\embed\probtransitions_{\latentpolicy}\fun{\sampledot \mid {\state}}} } \tag{by the inequality derived above} \\
    \leq& \fun{1 - \discount} \localrewardloss{\stationary{\latentpolicy}} \cdot \condition{\set{\rewards}}\fun{\varphi} + \discount K_\varphi \localtransitionloss{\stationary{\latentpolicy}} + \discount \expectedsymbol{\state \sim \stationary{\latentpolicy}}{ \bidistance_{\latentpolicy}^{\varphi}\fun{\state, \embed\fun{\state}} }.
\end{align*}
To see how we pass from the penultimate to the last line, recall that 
\[\expectedsymbol{\state \sim \stationary{\latentpolicy}}\wassersteindist{\bidistance_{\latentpolicy}^{\varphi}}{\probtransitions_{\latentpolicy}\fun{\sampledot \mid \state}}{\embed\probtransitions_{\latentpolicy}\fun{\sampledot \mid {\state}}} = \expectedsymbol{\state \sim \stationary{\latentpolicy}} \sup_{f \in \Lipschf{\bidistance^{\varphi}_{\latentpolicy}}} \left| \expectedsymbol{\state' \sim \probtransitions_{\latentpolicy}\fun{\sampledot \mid \state}}{f\fun{\state'}} - \expectedsymbol{\latentstate' \sim \embed\probtransitions_{\latentpolicy}\fun{\sampledot \mid \state}}{f\fun{\latentstate'}} \right|,\]
where \[\Lipschf{\bidistance^{\varphi}_{\latentpolicy}} = \set{f \colon \states \uplus \latentstates \to \R : \left| f\fun{\state_1^\star} - f\fun{\state_2^{\star}} \right| \leq  \bidistance^{\varphi}_{\latentpolicy}\fun{\state_1^{\star}, \state_2^{\star}}}.\]
Then, we have
\begin{align*}
    &\expectedsymbol{\state \sim \stationary{\latentpolicy}} \sup_{f \in \Lipschf{\bidistance^{\varphi}_{\latentpolicy}}} \left| \expectedsymbol{\state' \sim \probtransitions_{\latentpolicy}\fun{\sampledot \mid \state}}{f\fun{\state'}} - \expectedsymbol{\latentstate' \sim \embed\probtransitions_{\latentpolicy}\fun{\sampledot \mid \state}}{f\fun{\latentstate'}} \right| \\
    =& \expectedsymbol{\state \sim \stationary{\latentpolicy}} \sup_{f \in \Lipschf{\bidistance^{\varphi}_{\latentpolicy}}} \left| \expectedsymbol{\state' \sim \probtransitions_{\latentpolicy}\fun{\sampledot \mid \state}}{f\fun{\state'}} - \expectedsymbol{\state' \sim \probtransitions_{\latentpolicy}\fun{\sampledot \mid \state}}{f\fun{\embed\fun{\state'}}} \right| \\
    =& \expectedsymbol{\state \sim \stationary{\latentpolicy}} \sup_{f \in \Lipschf{\bidistance^{\varphi}_{\latentpolicy}}} \left| \expectedsymbol{\state' \sim \probtransitions_{\latentpolicy}\fun{\sampledot \mid \state}}\fun{f\fun{\state'} - f\fun{\embed\fun{\state'}}} \right| \\
    \leq& \expectedsymbol{\state \sim \stationary{\latentpolicy}} \expectedsymbol{\state' \sim \probtransitions_{\latentpolicy}\fun{\sampledot \mid \state}} \sup_{f \in \Lipschf{\bidistance^{\varphi}_{\latentpolicy}}} \left| {f\fun{\state'} - f\fun{\embed\fun{\state'}}} \right| \\
    =& \expectedsymbol{\state \sim \stationary{\latentpolicy}} \sup_{f \in \Lipschf{\bidistance^{\varphi}_{\latentpolicy}}} \left| {f\fun{\state} - f\fun{\embed\fun{\state}}} \right| \tag{by the stationary property} \\
    \leq& \expectedsymbol{\state \sim \stationary{\latentpolicy}} \bidistance^{\varphi}_{\latentpolicy}\fun{\state, \embed\fun{\state}} \tag{since $f \in \Lipschf{\bidistance^{\varphi}_{\latentpolicy}}$}.
\end{align*}
This finally yields 
\begin{align*}
     (1 - \discount) {\expectedsymbol{\state \sim \stationary{\latentpolicy}} \bidistance^{\varphi}_{\latentpolicy}\fun{\state, \embed\fun{\state}}} &\leq \fun{1 - \discount} \localrewardloss{\stationary{\latentpolicy}} \cdot \condition{\set{\rewards}}\fun{\varphi} + \discount \localtransitionloss{\stationary{\latentpolicy}} K_\varphi \\
    \equiv  \expectedsymbol{\state \sim \stationary{\latentpolicy}} \bidistance^{\varphi}_{\latentpolicy}\fun{\state, \embed\fun{\state}} &\leq \localrewardloss{\stationary{\latentpolicy}} \cdot \condition{\set{\rewards}}\fun{\varphi} + \discount \localtransitionloss{\stationary{\latentpolicy}} \frac{ K_\varphi }{1 - \discount}.
\end{align*}
\end{proof}

\subsection{Abstraction Quality as Bounded Bisimulation Distance}
\begin{lemma}\label{lemma:bidistance-abstraction-quality}
Let $\latentpolicy \in \latentmpolicies$, assume $\latentmdp$ is discrete and $\langle\KR{\latentpolicy}, \KP{\latentpolicy}\rangle$-Lipschitz.
Let $\stationary{\latentpolicy}$ be the stationary distribution of $\mdp$ induced by $\latentpolicy$.
Then, for any pair of states $\state_1, \state_2 \in \states$ such that $\embed\fun{\state_1} = \embed\fun{\state_2}$,
\begin{align*}
    \bidistance^{\rewards}_{\latentpolicy}\fun{\state_1, \state_2} \leq& \left[\localrewardloss{\stationary{\latentpolicy}} + \frac{
    \discount \localtransitionloss{\stationary{\latentpolicy}} \KR{\latentpolicy}}{1 - \discount \KP{\latentpolicy}}\right] \fun{\stationary{\latentpolicy}^{-1}\fun{\state_1} + \stationary{\latentpolicy}^{-1}\fun{\state_2}}, \\
    \bidistance^{\labels}_{\latentpolicy}\fun{\state_1, \state_2} \leq&  \frac{\discount \localtransitionloss{\stationary{\latentpolicy}}}{1 - \discount} \fun{\stationary{\latentpolicy}^{-1}\fun{\state_1} + \stationary{\latentpolicy}^{-1}\fun{\state_2}}.
\end{align*}
\end{lemma}
\begin{proof}
Let $K_{\rewards} = \frac{(1 - \discount) \KR{\latentpolicy}}{1 - \discount \KP{\latentpolicy}}$, $K_{\labels} = 1$, and $\varphi \in \set{\rewards, \labels}$.
We use the fact that $\bidistance^{\varphi}_{\latentpolicy}\fun{\state, \embed\fun{\state}} \leq \stationary{\latentpolicy}^{-1}\fun{\state} \expectedsymbol{\state \sim \stationary{\latentpolicy}}{\bidistance^{\varphi}_{\latentpolicy}\fun{\state, \embed\fun{\state}}}$ for all $\state \in \states$.
Then, we have
\begin{align*}
    &\bidistance^{\varphi}_{\latentpolicy}\fun{\state_1, \state_2}\\
    \leq& \bidistance^{\varphi}_{\latentpolicy}\fun{\state_1, \embed\fun{\state_1}} + \bidistance^{\varphi}_{\latentpolicy}\fun{\embed\fun{\state_1}, \embed\fun{\state_2}} + \bidistance^{\varphi}_{\latentpolicy}\fun{\embed\fun{\state_2}, \state_2} \tag{triangular inequality} \\
    \leq& \stationary{\latentpolicy}^{-1}\fun{\state_1} \expectedsymbol{\state \sim \stationary{\latentpolicy}}{\bidistance^{\varphi}_{\latentpolicy}\fun{\state_1, \embed\fun{\state_1}}} + \stationary{\latentpolicy}^{-1}\fun{\state_2} \expectedsymbol{\state \sim \stationary{\latentpolicy}}{\bidistance^{\varphi}_{\latentpolicy}\fun{\state_2, \embed\fun{\state_2}}} + \bidistance^{\varphi}_{\latentpolicy}\fun{\embed\fun{\state_1}, \embed\fun{\state_2}} \\
    \leq& \fun{\stationary{\latentpolicy}^{-1}\fun{\state_1} + \stationary{\latentpolicy}^{-1}\fun{\state_2}} \fun{\localrewardloss{\stationary{\latentpolicy}} \cdot \condition{\set{\rewards}}\fun{\varphi} + \discount \localtransitionloss{\stationary{\latentpolicy}} \frac{ K_\varphi }{1 - \discount}} + \bidistance^{\varphi}_{\latentpolicy}\fun{\embed\fun{\state_1}, \embed\fun{\state_2}} \tag{Lemma~\ref{lemma:bidistance-embedding}} \\
    =& \fun{\stationary{\latentpolicy}^{-1}\fun{\state_1} + \stationary{\latentpolicy}^{-1}\fun{\state_2}} \fun{\localrewardloss{\stationary{\latentpolicy}} \cdot \condition{\set{\rewards}}\fun{\varphi} + \discount \localtransitionloss{\stationary{\latentpolicy}} \frac{ K_\varphi }{1 - \discount}}. \tag{$\embed\fun{\state_1} = \embed\fun{\state_2}$ by assumption}
\end{align*}
\end{proof}

\subsection{Value Difference Bound}
\begin{lemma}\label{lemma:qvalues-diff-bound}
    Let $\latentpolicy \in \latentmpolicies$ and assume $\latentmdp$ is discrete and $\langle \KR{\latentpolicy}, \KP{\latentpolicy} \rangle$-Lipschitz.
    Let $\labelset{C}, \labelset{T} \subseteq \atomicprops$, $\varphi \in \set{\until{\labelset{C}}{\labelset{T}}, \eventually \labelset{T}}$ and $\KV = \min\fun{\nicefrac{\Rmax{\latentpolicy}}{1 - \discount}, \nicefrac{\KR{\latentpolicy}}{1 - \discount \KP{\latentpolicy}}}$.
    Then, given the induced stationary distribution $\stationary{\latentpolicy}$ in $\mdp$,
    \begin{align*}
        \expectedsymbol{\state, \latentaction \sim \stationary{\latentpolicy}} \left| \qvalues{\latentpolicy}{}{\state}{\latentaction} - \latentqvalues{\latentpolicy}{}{\embed\fun{\state}}{\latentaction} \right| & \leq \frac{\localrewardloss{\stationary{\latentpolicy}} + \discount \KV \localtransitionloss{\stationary{\latentpolicy}}}{1 - \gamma}, \\ 
        \expectedsymbol{\state, \latentaction \sim \stationary{\latentpolicy}} \left| \qvalues{\latentpolicy}{\varphi}{\state}{\latentaction} - \latentqvalues{\latentpolicy}{\varphi}{\embed\fun{\state}}{\latentaction} \right| & \leq \frac{\discount \localtransitionloss{\stationary{\latentpolicy}}}{1 - \discount}.
    \end{align*}
\end{lemma}
\begin{proof}
The first result for discounted return follows directly from \citep{DBLP:conf/icml/GeladaKBNB19}.
The choice of $\KV$ is derived from the fact that $\KP{\latentpolicy} \leq \frac{1}{\discount}$ since $\latentmdp$ is discrete:
$\dtvsymbol$ is bounded by one as well as the diameter of $\latentstates$, equipped with the discrete metric.
For the reachability cases, first observe that $\qvalues{\latentpolicy}{\varphi}{\state}{\sampledot} = 1$ when $\state \in \getstates{\labelset{T}} \subseteq \states$ and $\latentqvalues{\latentpolicy}{\varphi}{\latentstate}{\sampledot} = 1$ when $\latentstate \in \getstates{\labelset{T}} \subseteq \latentstates$.
Moreover, if $\varphi = \until{\labelset{C}}{\labelset{T}}$, $\qvalues{\latentpolicy}{\varphi}{\state}{\sampledot} = 0$ when $\state \in \getstates{\labelset{\neg C}} \subseteq \states$ and $\latentqvalues{\latentpolicy}{\varphi}{\latentstate}{\sampledot} = 0$ when $\latentstate \in \getstates{\labelset{\neg C}} \subseteq \latentstates$.
By Assumption~\ref{assumption:labeling}, $\state \in \getstates{\neg \labelset{C}} \cup \getstates{\labelset{T}}$ implies that the value difference is zero since the labels of each state and those of its latent abstraction are the same. 
Therefore, let $\mathbb{I}$ be $\condition{\getstates{\labelset{C}} \cap \getstates{\neg \labelset{T}}}$ if $\varphi = \until{\labelset{C}}{\labelset{T}}$ and $\condition{\getstates{\neg \labelset{T}}}$ if $\varphi = \eventually \labelset{T}$,
\begin{align*}
    & \expectedsymbol{\state, \latentaction \sim \stationary{\latentpolicy}} \left| \qvalues{\latentpolicy}{\varphi}{\state}{\latentaction} - \latentqvalues{\latentpolicy}{\varphi}{\embed\fun{\state}}{\latentaction} \right| \\
    =& \expected{\state, \latentaction \sim \stationary{\latentpolicy}}{ \mathbb{I}\fun{\state} \left| \expected{\state' \sim \probtransitions\fun{\sampledot \mid \state, \latentaction}}{\discount \values{\latentpolicy}{\varphi}{\state'}} - \expected{\latentstate' \sim \latentprobtransitions\fun{\sampledot \mid \embed\fun{\state}, \latentaction}}{\discount \latentvalues{\latentpolicy}{\varphi}{\latentstate'}} \right|} \\
    \leq& \discount \expectedsymbol{\state, \latentaction \sim \stationary{\latentpolicy}} \left| \expectedsymbol{\state' \sim \probtransitions\fun{\sampledot \mid \state, \latentaction}}{\values{\latentpolicy}{\varphi}{\state'}} - \expectedsymbol{\latentstate' \sim \latentprobtransitions\fun{\sampledot \mid \embed\fun{\state}, \latentaction}}{\latentvalues{\latentpolicy}{\varphi}{\latentstate'}}  \right| \\
    \leq& \discount \expectedsymbol{\state, \latentaction \sim \stationary{\latentpolicy}} \left| \expected{\state' \sim \probtransitions\fun{\sampledot \mid \state, \latentaction}}{\values{\latentpolicy}{\varphi}{\state'} - \latentvalues{\latentpolicy}{\varphi}{\embed\fun{\state'}}} + \expectedsymbol{}_{\substack{\latentstate' \sim \latentprobtransitions\fun{\sampledot \mid \embed\fun{\state}, \latentaction} \\ \state' \sim \probtransitions\fun{\sampledot \mid \state, \action}}}\left[{\latentvalues{\latentpolicy}{\varphi}{\embed\fun{\state'}} - \latentvalues{\latentpolicy}{\varphi}{\latentstate'}}\right] \right| \\
    \leq& \discount \expectedsymbol{\state, \latentaction \sim \stationary{\latentpolicy}} \left| \expected{\state' \sim \probtransitions\fun{\sampledot \mid \state, \latentaction}}{\values{\latentpolicy}{\varphi}{\state'} - \latentvalues{\latentpolicy}{\varphi}{\embed\fun{\state'}}} \right| + \discount  \expectedsymbol{\state, \latentaction \sim \stationary{\latentpolicy}} \left|\expectedsymbol{}_{\substack{\latentstate' \sim \latentprobtransitions\fun{\sampledot \mid \embed\fun{\state}, \latentaction} \\ \state' \sim \probtransitions\fun{\sampledot \mid \state, \action}}}\left[{\latentvalues{\latentpolicy}{\varphi}{\embed\fun{\state'}} - \latentvalues{\latentpolicy}{\varphi}{\latentstate'}}\right] \right| \tag{triangular inequality} \\
    \leq& \discount \expectedsymbol{\state, \latentaction \sim \stationary{\latentpolicy}} \left| \expected{\state' \sim \probtransitions\fun{\sampledot \mid \state, \latentaction}}{\values{\latentpolicy}{\varphi}{\state'} - \latentvalues{\latentpolicy}{\varphi}{\embed\fun{\state'}}} \right| \\
    & + \discount  \expectedsymbol{\state, \latentaction \sim \stationary{\latentpolicy}} \left|\expectedsymbol{\state' \sim \probtransitions\fun{\sampledot \mid \state, \latentaction}}{\latentvalues{\latentpolicy}{\varphi}{\embed\fun{\state'}} - \expectedsymbol{\latentstate' \sim \latentprobtransitions\fun{\sampledot \mid \embed\fun{\state}, \latentaction}}\latentvalues{\latentpolicy}{\varphi}{\latentstate'}} \right| \\
    \leq& \discount \expectedsymbol{\state, \latentaction \sim \stationary{\latentpolicy}} \left| \expected{\state' \sim \probtransitions\fun{\sampledot \mid \state, \latentaction}}{\values{\latentpolicy}{\varphi}{\state'} - \latentvalues{\latentpolicy}{\varphi}{\embed\fun{\state'}}} \right| + \discount  \expectedsymbol{\state, \latentaction \sim \stationary{\latentpolicy}} \wassersteindist{\condition{\neq}}{\embed\probtransitions\fun{\sampledot\mid \state, \latentaction}}{\latentprobtransitions\fun{\sampledot \mid \embed\fun{\state}, \latentaction}} \tag{by definition of the Wasserstein dual} \\
    \leq& \discount \expectedsymbol{\state, \latentaction \sim \stationary{\latentpolicy}} \left| \expected{\state' \sim \probtransitions\fun{\sampledot \mid \state, \latentaction}}{\values{\latentpolicy}{\varphi}{\state'} - \latentvalues{\latentpolicy}{\varphi}{\embed\fun{\state'}}} \right| + \discount \localtransitionloss{\stationary{\latentpolicy}} \\
    \leq& \discount \expectedsymbol{\state, \latentaction \sim \stationary{\latentpolicy}} \expectedsymbol{\state' \sim \probtransitions\fun{\sampledot \mid \state, \latentaction}} \left|\values{\latentpolicy}{\varphi}{\state'} - \latentvalues{\latentpolicy}{\varphi}{\embed\fun{\state'}} \right| + \discount \localtransitionloss{\stationary{\latentpolicy}} \tag{Jensen's inequality} \\
    \leq& \discount \expectedsymbol{\state \sim \stationary{\latentpolicy}} \left|\values{\latentpolicy}{\varphi}{\state} - \latentvalues{\latentpolicy}{\varphi}{\embed\fun{\state}} \right| + \discount \localtransitionloss{\stationary{\latentpolicy}} \tag{by the stationary property} \\
    =& \discount \expectedsymbol{\state \sim \stationary{\latentpolicy}} \left|\expected{\latentaction \sim \latentpolicy\fun{\sampledot \mid \embed\fun{\state}}}{\qvalues{\latentpolicy}{\varphi}{\state}{ \latentaction} - \latentqvalues{\latentpolicy}{\varphi}{\embed\fun{\state}}{\latentaction}} \right| + \discount \localtransitionloss{\stationary{\latentpolicy}}\\
    \leq& \discount \expectedsymbol{\state, \latentaction \sim \stationary{\latentpolicy}} \left|\qvalues{\latentpolicy}{\varphi}{\state}{\latentaction} - \latentqvalues{\latentpolicy}{\varphi}{\embed\fun{\state}}{\latentaction} \right| + \discount \localtransitionloss{\stationary{\latentpolicy}} \tag{Jensen's inequality} 
\end{align*}
Therefore,
\begin{align*}
    (1 - \discount) \expectedsymbol{\state, \latentaction \sim \stationary{\latentpolicy}} \left|\qvalues{\latentpolicy}{\varphi}{\state}{\latentaction} - \latentqvalues{\latentpolicy}{\varphi}{\embed\fun{\state}}{\latentaction} \right| &\leq \discount \localtransitionloss{\stationary{\latentpolicy}}, \text{ i.e.},\\
    \expectedsymbol{\state, \latentaction \sim \stationary{\latentpolicy}} \left|\qvalues{\latentpolicy}{\varphi}{\state}{\latentaction} - \latentqvalues{\latentpolicy}{\varphi}{\embed\fun{\state}}{\latentaction} \right| &\leq \frac{\discount \localtransitionloss{\stationary{\latentpolicy}}}{1 - \discount}
\end{align*}
\end{proof}

\subsection{Abstraction Quality as Value Difference Bound}
\begin{lemma}\label{lem:quality-latent-space}
    Let $\latentpolicy \in \latentmpolicies$ and assume $\latentmdp$ is discrete and $\langle \KR{\latentpolicy}, \KP{\latentpolicy} \rangle$-Lipschitz.
    Let $\labelset{C}, \labelset{T} \subseteq \atomicprops$, $\varphi \in \set{\until{\labelset{C}}{\labelset{T}}, \eventually \labelset{T}}$, $\KV = \min\fun{\nicefrac{\Rmax{\latentpolicy}}{1 - \discount}, \nicefrac{\KR{\latentpolicy}}{1 - \discount \KP{\latentpolicy}}}$, and $\stationary{\latentpolicy}$ be the induced stationary distribution in $\mdp$.
    Then, for any pair of states $\state_1, \state_2 \in \states$ such that $\embed\fun{\state_1} = \embed\fun{\state_2}$,
    \begin{align*}
      \left| \values{\latentpolicy}{}{\state_1} {-}
      \values{\latentpolicy}{}{\state_2} \right| &\leq
      \frac{\localrewardloss{\stationary{\latentpolicy}} {+} \discount \KV
      \localtransitionloss{\stationary{\stationary{\latentpolicy}}}}{1 {-}
      \discount} \fun{\stationary{\latentpolicy}^{-1}\fun{\state_1} {+}
      \stationary{\latentpolicy}^{-1}\fun{\state_2}},\\
      \left| \values{\latentpolicy}{\varphi}{\state_1} {-}
      \values{\latentpolicy}{\varphi}{\state_2} \right| &\leq \frac{\discount
      \localtransitionloss{\stationary{\latentpolicy}}}{1 {-} \discount}
      \fun{\stationary{\latentpolicy}^{-1}\fun{\state_1} {+}
      \stationary{\latentpolicy}^{-1}\fun{\state_2}}.
    \end{align*}
\end{lemma}
\begin{proof}
    The discounted return case follows directly from \citet{DBLP:conf/icml/GeladaKBNB19} and Lemma~\ref{lemma:qvalues-diff-bound}.
    For the reachability cases, we similarly use the fact that $\left| \values{\latentpolicy}{\varphi}{\state} -  \latentvalues{\latentpolicy}{\varphi}{\embed\fun{\state}} \right| \leq \stationary{\latentpolicy}^{-1}\fun{\state} \expectedsymbol{\state \sim \stationary{\latentpolicy}}\left| \values{\latentpolicy}{\varphi}{\state} -  \latentvalues{\latentpolicy}{\varphi}{\embed\fun{\state}} \right|$, which gives us
    \begin{align*}
        & \left| \values{\latentpolicy}{\varphi}{\state_1} -  \values{\latentpolicy}{\varphi}{\state_2} \right| \\
        =& \left| \values{\latentpolicy}{\varphi}{\state_1} -  \latentvalues{\latentpolicy}{\varphi}{\embed\fun{\state_1}}  +  \latentvalues{\latentpolicy}{\varphi}{\embed\fun{\state_1}} -  \latentvalues{\latentpolicy}{\varphi}{\embed\fun{\state_2}}  +  \latentvalues{\latentpolicy}{\varphi}{\embed\fun{\state_2}} - \values{\latentpolicy}{\varphi}{\state_2} \right|\\
        \leq& \left| \values{\latentpolicy}{\varphi}{\state_1} -  \latentvalues{\latentpolicy}{\varphi}{\embed\fun{\state_1}} \right| + \left| \latentvalues{\latentpolicy}{\varphi}{\embed\fun{\state_1}} -  \latentvalues{\latentpolicy}{\varphi}{\embed\fun{\state_2}} \right| + \left| \latentvalues{\latentpolicy}{\varphi}{\embed\fun{\state_2}} - \values{\latentpolicy}{\varphi}{\state_2} \right| \tag{triangular inequality} \\
        \leq& \stationary{\latentpolicy}^{-1}\fun{\state_1} \expectedsymbol{\state \sim \stationary{\latentpolicy}} \left| \values{\latentpolicy}{\varphi}{\state_1} -  \latentvalues{\latentpolicy}{\varphi}{\embed\fun{\state_1}} \right| + \left| \latentvalues{\latentpolicy}{\varphi}{\embed\fun{\state_1}} -  \latentvalues{\latentpolicy}{\varphi}{\embed\fun{\state_2}} \right| \\
        & + \stationary{\latentpolicy}^{-1}\fun{\state_2} \expectedsymbol{\state \sim \stationary{\latentpolicy}} \left| \latentvalues{\latentpolicy}{\varphi}{\embed\fun{\state_2}} - \values{\latentpolicy}{\varphi}{\state_2} \right| \\
        =& \stationary{\latentpolicy}^{-1}\fun{\state_1} \expectedsymbol{\state \sim \stationary{\latentpolicy}} \left| \values{\latentpolicy}{\varphi}{\state_1} -  \latentvalues{\latentpolicy}{\varphi}{\embed\fun{\state_1}} \right| + \stationary{\latentpolicy}^{-1}\fun{\state_2} \expectedsymbol{\state \sim \stationary{\latentpolicy}} \left| \values{\latentpolicy}{\varphi}{\state_2} -  \latentvalues{\latentpolicy}{\varphi}{\embed\fun{\state_2}} \right| \tag{$\embed\fun{\state_1} = \embed\fun{\state_2}$ by assumption} \\
        \leq& \frac{\discount \localtransitionloss{\stationary{\latentpolicy}}}{1 - \discount} \fun{\stationary{\latentpolicy}^{-1}\fun{\state_1} + \stationary{\latentpolicy}^{-1}\fun{\state_2}} \tag{by Lemma~\ref{lemma:qvalues-diff-bound}}
    \end{align*}
\end{proof}
\subsection{Transition-reward function}
In general RL environments, it is quite common for rewards to be generated from a reward function defined over transitions rather than over state-action pairs., i.e., $\rewards \colon \states \times \actions \times \states$.
In that case, we have $\rewards\fun{\state, \action} = \expectedsymbol{\state' \sim \probtransitions\fun{\sampledot \mid \state, \action}}{\rewards\fun{\state, \action, \state'}}$.
Taking transition-rewards into account, we set up an upper bound on $\localrewardloss{\stationary{\latentpolicy}}$ that can be efficiently estimated from samples (using Lemma~\ref{lemma:pac-loss}).
\begin{lemma}
Let $\stationary{\latentpolicy}$ be a stationary distribution of $\mdp$ under $\latentpolicy \in \latentmpolicies$, then
\begin{align*}
    &\localrewardloss{\stationary{\latentpolicy}} \leq \expectedsymbol{\state, \latentaction, \state' \sim \stationary{\latentpolicy}}\left| \rewards\fun{\state, \latentaction, \state'} - \latentrewards\fun{\embed\fun{\state}, \latentaction} \right| = \localrewardlossupper{\stationary{\latentpolicy}}.
\end{align*}
\end{lemma}
\begin{proof}
The upper bound $\localrewardlossupper{\stationary{\latentpolicy}}$ on the local reward loss is given as follows:
\begin{align*}
    & \expectedsymbol{\state, \latentaction \sim \stationary{\latentpolicy}}\left| \rewards\fun{\state, \latentaction} - \latentrewards\fun{\embed\fun{\state}, \latentaction} \right| \\
    =& \expectedsymbol{\state, \latentaction \sim \stationary{\latentpolicy}}\left| \expectedsymbol{\state' \sim \probtransitions\fun{\sampledot \mid \state, \latentaction}}\rewards\fun{\state, \latentaction, \state'} - \latentrewards\fun{\embed\fun{\state}, \latentaction} \right| \\
    \leq& \expectedsymbol{\state, \latentaction \sim \stationary{\latentpolicy}} \expectedsymbol{\state' \sim \probtransitions\fun{\sampledot \mid \state, \latentaction}} \left|\rewards\fun{\state, \latentaction, \state'} - \latentrewards\fun{\embed\fun{\state}, \latentaction} \right| \tag{Jensen's inequality}\\
    =& \expectedsymbol{\state, \latentaction, \state' \sim \stationary{\latentpolicy}} \left|\rewards\fun{\state, \latentaction, \state'} - \latentrewards\fun{\embed\fun{\state}, \latentaction} \right|
\end{align*}
\end{proof}
This ensures Lem.~\ref{lemma:pac-loss} and Thm.~\ref{thm:pac-qvaluedif} to remain valid by approximating $\localrewardlossupper{\stationary{\latentpolicy}}$ instead of $\localrewardloss{\stationary{\latentpolicy}}$.

\subsection{Local Transition Loss Upper Bound}
Recall that to deal with general MDPs, we study
an upper bound on $\localtransitionloss{\stationary{\latentpolicy}}$ (namely, $\localtransitionlossupper{\stationary{\latentpolicy}}$) that can be efficiently approximated from samples.
We derive it in the following Lemma.
\begin{lemma}\label{lemma:local-loss-upper-bounds}
Let $\stationary{\latentpolicy}$ be a stationary distribution of $\mdp$ under $\latentpolicy \in \latentmpolicies$, then
\begin{align*}
    &\localtransitionloss{\stationary{\latentpolicy}} \leq \expectedsymbol{\state, \latentaction, \state' \sim \stationary{\latentpolicy}} \wassersteindist{\distance_{\latentstates}}{\embed\fun{\sampledot \mid \state'}}{\latentprobtransitions\fun{\sampledot \mid \embed\fun{\state}, \latentaction}} = \localtransitionlossupper{\stationary{\latentpolicy}} 
\end{align*}
where we write $\stationary{\latentpolicy}\fun{\state, \latentaction, \state'}  =  \stationary{\latentpolicy}\fun{\state, \latentaction} \cdot \probtransitions\fun{\state' \mid \state, \latentaction}$
and $\embed\fun{\latentstate \mid \state} = \condition{=}\fun{\embed\fun{\state}, \latentstate}$. 
\end{lemma}
\begin{proof}
The upper bound $\localtransitionlossupper{\stationary{\latentpolicy}}$ on the local transition loss is obtained as follows.
\begin{align*}
    & \expectedsymbol{\state, \latentaction \sim \stationary{\latentpolicy}}{\sup_{f \in \Lipschf{\distance_{\latentstates}}} \left| \expectedsymbol{\latentstate_1 \sim \embed\probtransitions\fun{\sampledot \mid \state, \latentaction}} f\fun{\latentstate_1} - \expectedsymbol{\latentstate_2 \sim \latentprobtransitions\fun{\sampledot \mid \embed\fun{\state}, \latentaction}} f\fun{\latentstate_2} \right|} \\
    =& \expectedsymbol{\state, \latentaction \sim \stationary{\latentpolicy}}{\sup_{f \in \Lipschf{\distance_{\latentstates}}} \left| \expectedsymbol{\state' \sim \probtransitions\fun{\sampledot \mid \state, \latentaction}} f\fun{\embed\fun{\state'}} - \expectedsymbol{\latentstate_2 \sim \latentprobtransitions\fun{\sampledot \mid \embed\fun{\state}, \latentaction}} f\fun{\latentstate_2} \right|} \\
    =& \expectedsymbol{\state, \latentaction \sim \stationary{\latentpolicy}}{\sup_{f \in \Lipschf{\distance_{\latentstates}}} \left| \expected{\state' \sim \probtransitions\fun{\sampledot \mid \state, \latentaction}}{f\fun{\embed\fun{\state'}} - \expectedsymbol{\latentstate_2 \sim \latentprobtransitions\fun{\sampledot \mid \embed\fun{\state}, \latentaction}} f\fun{\latentstate_2}} \right|} \\
    \leq& \expectedsymbol{\state, \latentaction \sim \stationary{\latentpolicy}}\sup_{f \in \Lipschf{\distance_{\latentstates}}} \expectedsymbol{\state' \sim \probtransitions\fun{\sampledot \mid \state, \latentaction}} \left|f\fun{\embed\fun{\state'}} - \expectedsymbol{\latentstate_2 \sim \latentprobtransitions\fun{\sampledot \mid \embed\fun{\state}, \latentaction}} f\fun{\latentstate_2} \right| \tag{Jensen's inequality}\\
    \leq& \expectedsymbol{\state, \latentaction \sim \stationary{\latentpolicy}} \expectedsymbol{\state' \sim \probtransitions\fun{\sampledot \mid \state, \latentaction}} \sup_{f \in \Lipschf{\distance_{\latentstates}}} \left|f\fun{\embed\fun{\state'}} - \expectedsymbol{\latentstate_2 \sim \latentprobtransitions\fun{\sampledot \mid \embed\fun{\state}, \latentaction}} f\fun{\latentstate_2} \right| \\
    =& \expectedsymbol{\state, \latentaction, \state' \sim \stationary{\latentpolicy}} \sup_{f \in \Lipschf{\distance_{\latentstates}}} \left|\expectedsymbol{\latentstate_1 \sim {\embed\fun{\sampledot \mid \state'}}} f\fun{{\latentstate_1}} - \expectedsymbol{\latentstate_2 \sim \latentprobtransitions\fun{\sampledot \mid \embed\fun{\state}, \latentaction}} f\fun{\latentstate_2} \right| \\
    =& \expectedsymbol{\state, \latentaction, \state' \sim \stationary{\latentpolicy}} \wassersteindist{\distance_{\latentstates}}{\embed\fun{\sampledot \mid \state'} }{\latentprobtransitions\fun{\sampledot \mid \embed\fun{\state}, \latentaction}}
\end{align*}
\end{proof}

\subsection{Proof of Lemma~\ref{lemma:pac-loss}}
\begin{proof}
Since $\mdp$ is ergodic, the agent acts in the BSCC induced by $\latentpolicy$ and each tuple $\tuple{\state_t, \action_t, \reward_t, \state_{t + 1}}$ can be considered to be drawn independently from the stationary distribution $\stationary{\latentpolicy}$ of the BSCC in the limit.
The result follows then from the Hoeffding's bound \citep{10.2307/2282952:hoeffding}:
\begin{itemize}
    \item consider the sequence of independent random variables $\seq{X^\rewards}{T - 1}$ such that \[X^{\rewards}_t = |\reward_t - \latentrewards\fun{\embed\fun{\state_t}, \latentaction_t}| \in \mathopen[0, 1\mathclose],\] then
    $\localrewardlossapprox{\stationary{\latentpolicy}} = \frac{1}{T} \sum_{t=0}^{T-1} X^\rewards_t$ and
    $\Pr\fun{ \left|\localrewardloss{\star} - \localrewardlossapprox{\stationary{\latentpolicy}} \right| \geq \error} \leq \proberror_\rewards$, with $\proberror_\rewards = 2 e^{-2T\error^2}$ and $\localrewardloss{\star} = \localrewardloss{\stationary{\latentpolicy}}$ if $\rewards$ is defined over state-action pairs and $\localrewardloss{\star} = \localrewardlossupper{\stationary{\latentpolicy}}$ if it is defined over transitions.
    \item consider the sequence of random variables $\seq{X^{\probtransitions}}{T - 1}$ such that
    \[
        X^{\probtransitions}_t =  1 - \latentprobtransitions\fun{\embed\fun{\state_{t + 1}} \mid \embed\fun{\state_t}, \latentaction_t} \in \mathopen[0, 1 \mathclose],
    \]
    The latent MDP $\latentmdp$ being discrete, we have
    \begin{align*}
        \localtransitionlossupper{\stationary{\latentpolicy}} & = \expectedsymbol{\state, \latentaction, \state' \sim \stationary{\latentpolicy}} \dtv{\embed\fun{\sampledot \mid \state'}}{\latentprobtransitions\fun{\sampledot \mid \embed\fun{\state}, \action}} \\
        & = \expected{\state, \latentaction, \state' \sim \stationary{\latentpolicy}}{ \frac{1}{2} \sum_{\latentstate' \in \states} \left| \embed\fun{\latentstate' \mid \state'} - \latentprobtransitions\fun{\latentstate' \mid \embed\fun{\state}, \latentaction} \right|} \\
        & = \expected{\state, \latentaction, \state' \sim \stationary{\latentpolicy}}{ \frac{1}{2} \cdot \fun{\fun{1 - \latentprobtransitions\fun{\embed\fun{\state'} \mid \embed\fun{\state}, \latentaction}} + \sum_{\latentstate' \in \states \setminus\set{\embed\fun{\state'}}} \left| 0 - \latentprobtransitions\fun{\latentstate' \mid \embed\fun{\state}, \latentaction} \right|}}
        \tag{because $\embed\fun{\latentstate' \mid \state'} = 1$ if $ \embed\fun{\state'} = \latentstate'$ and $0$ otherwise.} \\
        & = \expected{\state, \latentaction, \state' \sim \stationary{\latentpolicy}}{ \frac{1}{2} \cdot \fun{\fun{1 - \latentprobtransitions\fun{\embed\fun{\state'} \mid \embed\fun{\state}, \latentaction}} + \sum_{\latentstate' \in \states \setminus\set{\embed\fun{\state'}}} \latentprobtransitions\fun{\latentstate' \mid \embed\fun{\state}, \latentaction}}} \\
        & = \expected{\state, \latentaction, \state' \sim \stationary{\latentpolicy}}{\frac{1}{2} \cdot 2 \cdot \fun{1 - \latentprobtransitions\fun{\embed\fun{\state'} \mid \embed\fun{\state}, \latentaction}}} \\
        & = \expected{\state, \latentaction, \state' \sim \stationary{\latentpolicy}}{1 - \latentprobtransitions\fun{\embed\fun{\state'} \mid \embed\fun{\state}, \latentaction}}
    \end{align*}
    then $\localtransitionlossapprox{\stationary{\latentpolicy}} = \frac{1}{T}\sum_{t = 0}^{T - 1} X^{\probtransitions}_{t}$ and $\Pr\fun{ \left| \localtransitionlossupper{\stationary{\latentpolicy}} - \localtransitionlossapprox{\stationary{\latentpolicy}} \right| \geq \error} \leq \proberror_\probtransitions$, with $\proberror_\probtransitions = 2e^{-2T\error^2}$.
\end{itemize}
Therefore,
\begin{align*}
    & \Pr\fun{\left|\localrewardloss{\star} - \localrewardlossapprox{\stationary{\latentpolicy}} \right| \geq \error \vee \left| \localtransitionlossupper{\stationary{\latentpolicy}} - \localtransitionlossapprox{\stationary{\latentpolicy}} \right| \geq \error} \\
    & \leq \proberror_{\rewards} + \proberror_{\probtransitions} \\
    & = 2 e^{-2 T \error^2} + 2e^{-2 T \error^2} \\
    & = 4 e^{-2T\error^2}.
\end{align*}
We thus consider $\proberror \geq 4e^{-2T\error^2}$, i.e., $\log\fun{\frac{\proberror}{4}} \geq -2T\error^2$ and $T \geq \frac{- \log\fun{\frac{\proberror}{4}}}{2 \error^2}$, which yields
\[ \Pr\fun{\left|\localrewardloss{\star} - \localrewardlossapprox{\stationary{\latentpolicy}} \right| \leq \error \wedge \left| \localtransitionlossupper{\stationary{\latentpolicy}} - \localtransitionlossapprox{\stationary{\latentpolicy}} \right| \leq \error} \geq 1 - \proberror.\]
\end{proof}

\subsection{Proof of Theorem~\ref{thm:pac-qvaluedif}}

\begin{proof}
Let $\error' \in \mathopen]0, 1\mathclose[$ and $T \geq \Big\lceil \frac{- \log\fun{\frac{\proberror}{4}}}{2 \error'^2} \Big\rceil$, then by applying Lemma~\ref{lemma:qvalues-diff-bound} and Lemma~\ref{lemma:pac-loss}, we have
\begin{align*}
    \expectedsymbol{\state, \latentaction \sim \stationary{\latentpolicy}} \left| \qvalues{\latentpolicy}{}{\state}{\latentaction} - \latentqvalues{\latentpolicy}{}{\embed\fun{\state}}{\latentaction} \right| & \leq && \frac{\localrewardloss{\stationary{\latentpolicy}} + \discount \KV \localtransitionloss{\stationary{\latentpolicy}}}{1 - \gamma} &\leq& \frac{\fun{\localrewardlossapprox{\stationary{\latentpolicy}} + \error'} + \discount \KV \fun{\localtransitionlossapprox{\stationary{\latentpolicy}} + \error'}}{1 - \gamma}, \text{ and} \\ 
    \expectedsymbol{\state, \latentaction \sim \stationary{\latentpolicy}} \left| \qvalues{\latentpolicy}{\varphi}{\state}{\latentaction} - \latentqvalues{\latentpolicy}{\varphi}{\embed\fun{\state}}{\latentaction} \right| & \leq && \frac{\discount \localtransitionloss{\stationary{\latentpolicy}}}{1 - \discount} \quad \; \quad &\leq& \frac{\discount \fun{\localtransitionlossapprox{\stationary{\latentpolicy}} + \error'}}{1 - \discount}
\end{align*}
with probability at least $1 - \proberror$.
To offer an error of at most $\error$, we need
\begin{align*}
    && \frac{\error' + \discount \KV \error'}{1 - \discount} &\leq \error& \\
    &\equiv& \error' + \discount \KV \error' &\leq \error (1 - \discount)& \\
    &\equiv& \error'(1 + \discount \KV) &\leq \error (1 - \discount)& \\
    &\equiv& \error' &\leq \frac{\error (1 - \discount)}{1 + \discount \KV}&
\end{align*}
which yields the result.
\end{proof}

\section{Details on Variational MDPs}
\smallparagraph{ELBO.} In this work, we focus on the case where
$\dklsymbol$ is used as divergence measure.
The goal is thus to optimize $\min_{\decoderparameter}
\dkl{\mdp_{\policy}}{\decoder}$, i.e., maximizing the marginal
log-likelihood of traces of the original MDP $\expected{\trace \sim \mdp_\policy}{\log \decoder\fun{\trace}}$.
Recall that
\begin{equation}
    \decoder\fun{\trace} =
            \int_{\inftrajectories{\latentmdp_{{\decoderparameter}}}} \decoder\fun{\trace \mid \seq{\latentvariable}{T}} \, d\latentprobtransitions_{\latentpolicy_{ \decoderparameter}}\fun{\seq{\latentvariable}{T}} \label{eq:maximum-likelihood2}
\end{equation}
where $\trace = \defaulttrace$, $\latentvariable \in \latentvariables$ such that $\latentvariables = \latentstates$ if $\latentactions = \actions$ and $\latentvariables = \latentstates \times \latentactions$ otherwise,
$\latentprobtransitions_{\latentpolicy_{\decoderparameter}}\fun{\seq{\latentstate}{T}} = \prod_{t = 0}^{T - 1} \latentprobtransitions_{\latentpolicy_{\decoderparameter}}\fun{\latentstate_{t + 1} \mid \latentstate_t}$, and $\latentprobtransitions_{\latentpolicy_{ \decoderparameter}}\fun{\seq{\latentstate}{T}, \seq{\latentaction}{T-1}} = \prod_{t = 0}^{T - 1} {\latentpolicy_{\decoderparameter}}\fun{\latentaction_t \mid \latentstate_t} \cdot \latentprobtransitions_{\decoderparameter}\fun{\latentstate_{t + 1} \mid \latentstate_t, \latentaction_t}$.
The dependency of $\trace$ on $\latentvariables$ in Eq.~\ref{eq:maximum-likelihood2}
is made explicit by the law of total probability.

Optimizing $\expected{\trace \sim \mdp_\policy}{\log \decoder\fun{\trace}}$ through Eq.~\ref{eq:maximum-likelihood2} is typically intractable \citep{DBLP:journals/corr/KingmaW13}.
To overcome this, we use an encoder $\encoder\fun{\seq{\latentvariable}{T} \mid \trace}$ to approximate the intractable true posterior
    $\decoder\fun{\seq{\latentvariable}{T} \mid \trace} = \decoder\fun{\trace \mid \seq{\latentvariable}{T}} \cdot \nicefrac{{}\latentprobtransitions_{\latentpolicy_{ \decoderparameter}}\fun{\seq{\latentvariable}{T}}}{\decoder\fun{\trace}}$. 
Such an encoder can be learned through the following optimization:
\begin{equation}
    \min_{\encoderparameter, \decoderparameter} \; \expectedsymbol{\trace \sim \mdp_\policy}{\,\divergence{\encoder\fun{\sampledot \mid \trace}}{{\decoder\fun{\sampledot \mid \trace}}},}
\end{equation}
where $\divergencesymbol$ denotes a discrepancy measure.
If we let once again $\divergencesymbol$ be $\dklsymbol$, one can set a \emph{lower bound} on the log-likelihood of produced traces,
often referred to as \emph{evidence lower bound} (ELBO, \citealp{DBLP:journals/jmlr/HoffmanBWP13}) as follows:
\begin{align*}
    &&& {\dkl{\encoder\fun{\sampledot \mid \trace}}{{\decoder\fun{\sampledot \mid \trace}}}} \notag \\
    &&= \,& {\expected{\seq{\latentvariable}{T} \sim \encoder\fun{\sampledot \mid \trace}}{\log \encoder\fun{\seq{\latentvariable}{T} \mid \trace} - {\decoder\fun{\seq{\latentvariable}{T} \mid \trace}}}} \notag \\
    && \begin{aligned}
       = \\ \, \\
    \end{aligned} \;& \begin{aligned}
        \expectedsymbol{\seq{\latentvariable}{T} \sim \encoder\fun{\sampledot \mid \trace}}[&\log \encoder\fun{\seq{\latentvariable}{T} \mid \trace} - \log \decoder\fun{\trace \mid \seq{\latentvariable}{T}} - \log \latentprobtransitions_{\latentpolicy_{\decoderparameter}}\fun{\seq{\latentvariable}{T}}] + \log \decoder\fun{\trace}
    \end{aligned} %
    \\
    &\iff&& {\log \decoder\fun{\trace} - \dkl{\encoder\fun{\sampledot \mid \trace}}{{\decoder\fun{\sampledot \mid \trace}}}} \notag \\
    && \begin{aligned}
        = 
    \end{aligned} \:& \begin{aligned}
        \expectedsymbol{\seq{\latentvariable}{T} \sim \encoder\fun{\sampledot \mid \trace}}[&\log \decoder\fun{\trace \mid \seq{\latentvariable}{T}}  + \log \latentprobtransitions_{\latentpolicy_{ \decoderparameter}}\fun{\seq{\latentvariable}{T}}
        - \log \encoder\fun{\seq{\latentvariable}{T} \mid \trace}]
    \end{aligned} \label{eq:ground-action-elbo}
\end{align*}

\smallparagraph{Behavioral model.} We assume that the behavioral model $\decoder$ allows recovering the latent MDP parameters (i.e., $\latentprobtransitions_\decoderparameter$, $\latentrewards_{\decoderparameter}$, and $\latentlabels_{\decoderparameter}$), by decomposing it into reward and label models, the distilled policy, as well as a generative model $\decoder^{\generative}$ allowing the reconstruction of states and actions.
Therefore, states, actions, rewards, and labels are independent given the latent sequence:
\begin{align*}
    &\decoder\fun{\seq{\state}{T}, \seq{\action}{T - 1}, \seq{\reward}{T - 1}, \seq{\labeling}{T} \mid \seq{\latentstate}{T}}\\
    = & \decoder^{\generative}\fun{\seq{\state}{T} \mid \seq{\latentstate}{T}} \cdot \decoder\fun{\seq{\action}{T - 1}, \seq{\reward}{T - 1} \mid \seq{\latentstate}{T}} \cdot \decoder^{\labels}\fun{\seq{\labeling}{T} \mid \seq{\latentstate}{T}} \\
    = & \decoder^{\generative}\fun{\seq{\state}{T} \mid \seq{\latentstate}{T}} \cdot 
    \latentpolicy_{\decoderparameter}\fun{\seq{\action}{T - 1} \mid \seq{\latentstate}{T}} \cdot
    \decoder^{\rewards}\fun{\seq{\reward}{T - 1} \mid \seq{\latentstate}{T}, \seq{\action}{T - 1}} \cdot  \decoder^{\labels}\fun{\seq{\labeling}{T} \mid \seq{\latentstate}{T}}, \tag{by the chain rule, $\decoder\fun{\seq{\action}{T - 1}, \seq{\reward}{T - 1} \mid \seq{\latentstate}{T}} = \latentpolicy_{\decoderparameter}\fun{\seq{\action}{T - 1} \mid \seq{\latentstate}{T}} \cdot
    \decoder^{\rewards}\fun{\seq{\reward}{T - 1} \mid \seq{\latentstate}{T}, \seq{\action}{T - 1}}$}  \\
    & \text{and} \\
    &\decoder\fun{\seq{\state}{T}, \seq{\action}{T - 1}, \seq{\reward}{T - 1}, \seq{\labeling}{T} \mid \seq{\latentstate}{T}, \seq{\latentaction}{T - 1}}\\
    = & \decoder^{\generative}\fun{\seq{\state}{T} \mid \seq{\latentstate}{T}} \cdot \decoder^{\generative}\fun{\seq{\action}{T - 1} \mid \seq{\latentstate}{T},  \seq{\latentaction}{T - 1}} \cdot \decoder^{\rewards}\fun{\seq{\reward}{T - 1} \mid \seq{\latentstate}{T}, \seq{\latentaction}{T - 1}} \cdot \decoder^{\labels}\fun{\seq{\labeling}{T} \mid \seq{\latentstate}{T}}
\end{align*}

\smallparagraph{Local setting.} We formulate the ELBO in the local setting thanks to the following Lemma.
\begin{lemma}\label{lemma:trace-to-stationary}
Let  $\stationary{\policy}$ be the stationary distribution of $\mdp_\policy$ and $f \colon \states \times \actions \times \images{\rewards} \times \states \to \R$ be a continuous function.
Assume $\states, \actions$ are compact, then % induced by $\policy$, then 
\begin{align*}
    \lim_{T \to \infty} \frac{1}{T} \expected{\trace \sim \mdp_{\policy}[T]}{\sum_{t = 0}^{T - 1} f\fun{\state_t, \action_t, \reward_t, \state_{t + 1}}} = \expectedsymbol{\state, \action, \reward, \state' \sim \stationary{\policy}} f\fun{\state, \action, \reward, \state'}
\end{align*}
where $\traces{\mdp_{\policy}}\fun{T}$ denotes the set of traces in $\mdp_{\policy}$ of size $T$, $\mdp_{\policy}[T]$ denotes the distribution over $\traces{\mdp_\policy}\fun{T}$, and $ \stationary{\policy}(\state, \action, \reward, \state') = \stationary{\policy}(\state, \action, \state') \cdot \condition{=}(\reward,  \rewards\fun{\state, \action, \state'})$.
\end{lemma}
\begin{proof}
First, observe that
\begin{align*}
    & \expected{\trace \sim \mdp_{\policy}[T]}{\sum_{t = 0}^{T - 1} f\fun{\state_t, \action_t, \reward_t, \state_{t + 1}}} \\
    =&  \expectedsymbol{\action_0 \sim \policy\fun{\sampledot \mid \sinit}}\expectedsymbol{\state_1 \sim \probtransitions\fun{\sampledot \mid \sinit, \action_0}} \expectedsymbol{\action_1 \sim \policy\fun{\sampledot \mid \state_1}} \cdots \expectedsymbol{\state_T \sim \probtransitions\fun{\sampledot \mid \state_{T - 1}, \action_{T - 1}}} \left[ \sum_{t = 0}^{T - 1} f\fun{\state_t, \action_t, \rewards\fun{\state_t, \action_t, \state_{t + 1}}, \state_{t + 1}}\right]\\
    =& \sum_{t = 0}^{T - 1} \expectedsymbol{\action_0 \sim \policy\fun{\sampledot \mid \sinit}}\expectedsymbol{\state_1 \sim \probtransitions\fun{\sampledot \mid \sinit, \action_0}} \expectedsymbol{\action_1 \sim \policy\fun{\sampledot \mid \state_1}} \cdots \expectedsymbol{\state_T \sim \probtransitions\fun{\sampledot \mid \state_{T - 1}, \action_{T - 1}}} f\fun{\state_t, \action_t, \rewards\fun{\state_t, \action_t, \state_{t + 1}}, \state_{t + 1}}\\
    =& \sum_{t = 0}^{T - 1} \expectedsymbol{\trace \sim \mdp_{\policy}[t + 1]} f\fun{\state_t, \action_t, \reward_t, \state_{t + 1}}.
\end{align*}
Second, let $t \in \N, \state, \state' \in \states, \action \in \act{\state},$ and $\reward = \rewards\fun{\state, \action, \state'}$, also observe that
\begin{align*}
    \stationary{\policy}^t\fun{\state \mid \sinit} &= \int_{\traces{\mdp_{\policy}}\fun{t}} \condition{=}\fun{\state_t, \state} \, d\Prob_{\policy}\fun{Cyl\fun{\seq{\state}{t},\seq{\action}{t - 1}}}
    \tag{where $Cyl\fun{\seq{\state}{t}, \seq{\action}{t}} = \set{\tuple{\seq{\state^{\star}}{\infty}, \seq{\action^{\star}}{\infty}} \in \inftrajectories{\mdp_{\policy}} \mid \state^{\star}_{i + 1} = \state_{i + 1}, \action^{\star}_{i} = \action_i \; \forall i \leq t - 1}$} \\
    f\fun{\state, \action, \reward, \state'} \cdot \stationary{\policy}^t\fun{\state \mid \sinit} &=  f\fun{\state, \action, \reward, \state'} \cdot \int_{\traces{\mdp_{\policy}}\fun{t}} \condition{=}\fun{\state_t, \state} \, d\Prob_{\policy}\fun{Cyl\fun{\seq{\state}{t}, \seq{\action}{t - 1}}} \\
    \int_{\states} f\fun{\state, \action, \reward, \state'} \,d\stationary{\policy}^t\fun{\state \mid \sinit} & =  \int_{\states} \int_{\traces{\mdp_{\policy}}\fun{t}} f\fun{\state, \action, \reward, \state'} \cdot \condition{=}\fun{\state_t, \state} \, d\Prob_{\policy}\fun{Cyl\fun{\seq{\state}{t}, \seq{\action}{t - 1}}} \, d\state \\
    & = \int_{\traces{\mdp_{\policy}}\fun{t}} f\fun{\state_t, \action, \reward, \state'} \, d\Prob_{\policy}\fun{Cyl\fun{\seq{\state}{t}, \seq{\action}{t - 1}}}.
\end{align*}
Therefore, we have
\begin{align*}
    & \lim_{T \to \infty} \frac{1}{T} \expected{\trace \sim \mdp_{\policy}[T]}{\sum_{t = 0}^{T - 1} f\fun{\state_t, \action_t, \reward_t, \state_{t + 1}}} \\
    =& \lim_{T \to \infty} \frac{1}{T} \sum_{t = 0}^{T - 1} \expectedsymbol{\trace \sim \mdp_{\policy}[t + 1]} f\fun{\state_t, \action_t, \reward_t, \state_{t + 1}} \\
    =& \lim_{T \to \infty} \frac{1}{T} \sum_{t = 0}^{T - 1} \int_{\traces{\mdp_{\policy}}\fun{t + 1}} f\fun{\state_t, \action_t, \reward_t, \state_{t + 1}} \, d\Prob_\policy\fun{Cyl\fun{\seq{\state}{t + 1}, \seq{\action}{t}}} \\
    =& \lim_{T \to \infty} \frac{1}{T} \sum_{t = 0}^{T - 1} \int_{\states} \int_{\actions} \int_{\states} f\fun{\state, \action, \rewards\fun{\state, \action, \state'}, \state'} \, d\stationary{\policy}^t\fun{\state\mid\sinit} \, d\policy\fun{\action \mid \state} \, d\probtransitions\fun{\state' \mid \state, \action} \\
    =& \lim_{T \to \infty} \int_{\states}  \frac{1}{T} \sum_{t = 0}^{T - 1} \stationary{\policy}^t\fun{d\state \mid \sinit}  \int_{\actions} \int_{\states} f\fun{\state, \action, \rewards\fun{\state, \action, \state'}, \state'} \, d\policy\fun{\action \mid \state} \, d\probtransitions\fun{\state' \mid \state, \action}.
\end{align*}
By definition of $\stationary{\policy}$ and Assumption~\ref{assumption:ergodicity}, $\stationary{\policy}\fun{A} = \lim_{T \to \infty} \frac{1}{T} \sum_{t = 0}^{T - 1} \stationary{\policy}^{t}\fun{A \mid \state}$ for all $\state \in \states$, $A \in \borel{\states}$, i.e., $\frac{1}{T} \sum_{t = 0}^{T - 1} \stationary{\policy}^{t}\fun{A \mid \state}$ weakly converges to $\stationary{\policy}$ (e.g., \citealtAR{DBLP:books/daglib/0020348}).
By Assumption~\ref{assumption:rewards}, the set of images of $\rewards$ is a compact space.
Therefore, $f$ has compact support and by the Portmanteau's Theorem, 
\begin{align*}
    & \lim_{T \to \infty} \int_{\states}  \frac{1}{T} \sum_{t = 0}^{T - 1} \stationary{\policy}^t\fun{d\state \mid \sinit} \int_{\actions} \int_{\states} f\fun{\state, \action, \rewards\fun{\state, \action, \state'}, \state'} \, d\policy\fun{\action \mid \state} \, d\probtransitions\fun{\state' \mid \state, \action} \\
    =& \int_{\states} \stationary{\policy}\fun{d\state} \int_{\actions} \int_{\states} f\fun{\state, \action, \rewards\fun{\state, \action, \state'}, \state'} \, d\policy\fun{\action \mid \state} \, d\probtransitions\fun{\state' \mid \state, \action} \\
    =& \expectedsymbol{\state, \action, \reward, \state' \sim \stationary{\policy}} f\fun{\state, \action, \reward, \state'}.
\end{align*}
\end{proof}

\begin{corollary}\label{cor:elbo-local}
Let $ \embed\fun{\latentstate, \latentstate' \mid \state, \state'} = \embed\fun{\latentstate \mid \state} \embed\fun{\latentstate' \mid \state'}$, then
\begin{align*}
    & \lim_{T \to \infty} \frac{1}{T} \expectedsymbol{}_{\substack{\trace \sim \mdp_{\policy}[T] \\ \seq{\latentvariable}{T} \sim \encoder\fun{\sampledot \mid \trace}}} \left[\log \decoder\fun{\trace \mid \seq{\latentvariable}{T}} - \dkl{\encoder\fun{\sampledot \mid \trace}}{{\decoder\fun{ \sampledot \mid \trace}}} \right] \\
    =&
    \left\lbrace\begin{array}{@{}r@{\quad}l@{}l@{}}
    \displaystyle \expectedsymbol{}_{\substack{\state, \action, \reward, \state' \sim \stationary{\policy}\\ \latentstate, \latentstate' \sim \embed_{\encoderparameter}\fun{\sampledot \mid \state, \state'}}} \big[ \log \decoder^{\generative}\fun{\state' \mid \latentstate'} + \log \latentpolicy_{\decoderparameter}\fun{\action \mid \latentstate} + \log \decoder^{\rewards}\fun{\reward \mid \latentstate, \action}& \\[-12pt]
    \hfill - \dkl{\embed_{\encoderparameter}\fun{\sampledot \mid \state'}}{\latentprobtransitions_{\latentpolicy_{\decoderparameter}}\fun{\sampledot\mid \latentstate}} \big]
        & \text{if } \latentactions = \actions, \text{ and} \\[10pt]
    \displaystyle \expectedsymbol{}_{\substack{\state, \action, \reward, \state' \sim \stationary{\policy}\\ \latentstate, \latentstate' \sim \embed_{\encoderparameter}\fun{\sampledot \mid \state, \state'} \\ \latentaction \sim \latentembeda_{\encoderparameter}\fun{\sampledot \mid \latentstate, \action}}} \big[ \log \decoder^{\generative}\fun{\state' \mid \latentstate'} + \log\embeda_{\decoderparameter}\fun{\action \mid \latentstate, \latentaction} + \log \decoder^{\rewards}\fun{\reward \mid \latentstate, \latentaction} & \\[-25pt]
    \hfill - \dkl{\embed_{\encoderparameter}\fun{\sampledot \mid \state'}}{\latentprobtransitions_{\decoderparameter}\fun{\sampledot\mid \latentstate}} \quad \;\, & \\[5pt]
    \hfill - \dkl{\latentembeda_{\encoderparameter}\fun{\sampledot\mid\latentstate, \action}}{\latentpolicy_{\decoderparameter}\fun{\sampledot \mid \latentstate}} \big]
        & \text{otherwise},
    \end{array}
    \right.
\end{align*}
where $\trace = \defaulttrace$, $\seq{\latentvariable}{T} = \seq{\latentstate}{T}$ if $\latentactions = \actions$ and $\seq{\latentvariable}{T} = \tuple{\seq{\latentstate}{T}, \seq{\latentaction}{T - 1}}$ otherwise.
\end{corollary}
\begin{proof}
Taking 
\[f\fun{\state, \action, \reward, \state'} = \expectedsymbol{\latentstate, \latentstate' \sim \embed_{\encoderparameter}\fun{\sampledot \mid \state, \state'}}\big[ g\fun{\tuple{\state, \action, \reward, \state'}, \tuple{\latentstate, \latentstate'}} \big]\]
with
\begin{align*}
    & g\fun{\tuple{\state, \action, \reward, \state'}, \tuple{\latentstate, \latentstate'}}\\
    =& 
    \left\lbrace\begin{array}{@{}r@{\quad}l@{}l@{}}
    \log \fun{\decoder^{\generative}\fun{\state' \mid \latentstate'} \cdot \decoder^{\rewards}\fun{\reward \mid \latentstate, \action} \cdot \latentpolicy_{\decoderparameter}\fun{\action \mid \latentstate}} - \dkl{\embed_{\encoderparameter}\fun{\sampledot \mid \state'}}{\latentprobtransitions_{\latentpolicy_{\decoderparameter}}\fun{\sampledot\mid \latentstate}} & \text{if } \latentactions = \actions, \text{ and} \\[5pt]
    \expectedsymbol{\latentaction \sim \encoder^\actions\fun{\sampledot \mid \latentstate, \action}} \Big[ \log \fun{\decoder^{\generative}\fun{\state' \mid \latentstate'} \cdot \embeda_{\decoderparameter}\fun{\action \mid \latentstate, \latentaction} \cdot \decoder^{\rewards}\fun{\reward \mid \latentstate, \latentaction}} \hfill & \\[3pt]
    - \dkl{\embed_{\encoderparameter}\fun{\sampledot \mid \state'}}{\latentprobtransitions_{\decoderparameter}\fun{\sampledot\mid \latentstate, \latentaction}} \Big] - \dkl{\latentembeda_{\encoderparameter}\fun{\sampledot\mid\latentstate, \action}}{\latentpolicy_{\decoderparameter}\fun{\sampledot \mid \latentstate}} & \text{otherwise}
    \end{array}
    \right.
\end{align*}
yields the result.
\end{proof}

\subsection{Latent Distributions}
\smallparagraph{Discrete latent state space.} W.l.o.g., we assume latent states to have a binary representation, we further see each bit as a Bernoulli random variable.
Let $\logit \in \R$ and $\temperature \in \mathopen]0, 1 \mathclose]$, $\latentstate_\temperature \in \mathopen[0, 1\mathclose]$ has a \emph{relaxed Bernoulli distribution} $\latentstate_\temperature \sim \relaxedbernoulli{\logit}{\temperature}$ with logit $\logit$ and temperature parameter $\temperature$ iff
(a) if $l$ is a logistic sample with location parameter $\logit\temperature^{-1}$ and scale parameter $\temperature^{-1}$, i.e., $l \sim \logistic{\logit\temperature^{-1}}{\temperature^{-1}}$, then $\latentstate_\temperature = \sigmoid\fun{l}$, $\sigmoid$ being the \emph{sigmoid} function,
(b) $\lim_{\temperature \to 0} \relaxedbernoulli{\logit}{\temperature} = \bernoulli{\logit}$, meaning $\Prob(\textstyle\lim_{\temperature \to 0} \latentstate_\temperature = 1) = \sigmoid\fun{\logit}$, and
(c) let $p_{\logit, \temperature}$ be the density of $\relaxedbernoulli{\logit}{\temperature}$, then $p_{\logit, \temperature}\fun{\latentstate_\temperature}$ is log-convex in $\latentstate_\temperature$.
Since the logistic distribution belongs to the location-scale family, it is fully reparameterizable, i.e., $l \sim \logistic{\logit\temperature^{-1}}{\temperature^{-1}} \equiv x \sim \logistic{0}{1} \text{ and } l = \frac{\logit + x}{\temperature}$.
In practice, we train $\embed_{\encoderparameter}$ and $\latentprobtransitions_{\decoderparameter}$ to infer and generate $\log_2 |\latentstates|$ logits, and we anneal $\temperature$ to $0$ during training while using Bernoulli distributions for evaluation.
As suggested by \citet{DBLP:conf/iclr/MaddisonMT17}, we use two different temperature parameters and we fix their initial values to $\temperature_{\embed_{\encoderparameter}} = \nicefrac{2}{3}$ for the encoder distribution and $\temperature_{\probtransitions_{\decoderparameter}} = \nicefrac{1}{2}$ for the latent transition function.

\smallparagraph{Discrete latent action space.} We learn discrete latent distributions for $\encoder^{\actions}$ and $\latentpolicy_{\decoderparameter}$ by learning the logits of continuous relaxation of discrete distributions, through the \emph{Gumbel softmax trick}.
Let $n = |\latentactions|$, and $\latentactions = \set{\latentaction_1, \dots, \latentaction_n}$.
Drawing $\latentaction_i \sim \categorical{\logits}$  from a discrete (categorical) distribution in $\distributions{\latentactions}$ with logits parameter $\logits \in \R^n$ is equivalent to applying an $\argmax$ operator to the sum of \emph{Gumble samples} and these logits, i.e., $i = \argmax_{k} \left[ \logits_k - \log\fun{- \log \vect{\epsilon}_k} \right]$, where $\vect{\epsilon} \in \mathopen[0, 1\mathclose]^n$ is a uniform noise.
The sampling operator being fully reparameterizable here, $\argmax$ is however not derivable, preventing the operation to be optimized through gradient descent.
The proposed solution is to replace the $\argmax$ operator by a softmax function with temperature parameter $\temperature \in \mathopen]0, \left(n - 1\right)^{-1}\mathclose]$.

Concretely, $\vect{x} \in \mathopen[0, 1\mathclose]^n$ has a \emph{relaxed discrete distribution} $\vect{x} \sim \relaxedcategorical{\logits}{\temperature}$ with logits $\logits$, temperature parameter $\temperature$ iff
(a) let $\vect{\epsilon} \in [0, 1]^n$ be a uniform noise and $\vect{G}_k = {- \log \fun{- \log \vect{\epsilon}_k}}$, then $\vect{x}_k = {\frac{\exp{\fun{\fun{\logits_k + \vect{G}_k} \temperature^{-1}}}}{\sum_{i = 1}^{n} \exp{\fun{\fun{\logits_i + \vect{G}_i} \lambda^{-1}}}}}$ for each $k \in [n]$,
(b) $\lim_{\temperature \to 0} \relaxedcategorical{\logits}{\temperature} = \categorical{\logits}$, meaning $\Prob(\textstyle\lim_{\temperature \to 0} \vect{x}_k = 1) = \frac{\exp{\logits_k}}{\sum_{i = 1}^{n}\exp{\logits_i}}$,
and (c) let $p_{\logits, \temperature}$ be the density of $\relaxedcategorical{\logits}{\temperature}$, then $p_{\logits, \temperature}\fun{\vx}$ is log-convex in $\vx$.
In practice, we train $\encoder^{\actions}$ and $\latentpolicy_{\decoderparameter}$ to infer and generate logits parameter $\logits$, and we anneal $\temperature$ from $(n - 1)^{-1}$ to $0$ during training while using discrete distributions for evaluation.
Again, as prescribed by \citet{DBLP:conf/iclr/MaddisonMT17}, we use two different temperature parameters, one for $\encoder^{\actions}$ and another one for $\latentpolicy_{\decoderparameter}$.

\subsection{Prioritized Experience Replay}
Modern RL techniques allow to learn the parameters $\decoderparameter$ of a policy, model, or value function, by processing encountered experiences $\experience$ 
via a loss function $\loss\fun{\experience, \decoderparameter}$.
If these
are processed sequentially and discarded after all parameter updates, the latter are strongly correlated and rare events are quickly forgotten.
\emph{Prioritized experience replay buffers} \citep{DBLP:journals/corr/SchaulQAS15} are finite data structures $\replaybuffer$ that allow to overcome both issues by
(i) storing experiences $\experience$ with some assigned priority $p_{\experience}$, and
(ii) producing $\experience \sim \replaybuffer$ according to $P\fun{\experience} = p_\experience^\varsigma / \sum_{\experience' \in \replaybuffer} p_{\experience'}^\varsigma$, $\varsigma \in \mathopen[0, 1\mathclose]$, which allows minimizing $\expected{\experience \sim \replaybuffer}{w_{\experience} \cdot \loss\fun{\experience, \decoderparameter}}$ for \emph{importance sampling weight} $w_\experience > 0$.
The replay buffer $\replaybuffer$ is \emph{uniform} if $\varsigma = 0$.
Otherwise, when $\varsigma > 0$, a bias is introduced because priorities change the real distribution of experiences and consequently the estimates of the expectation of $\loss$.
This bias is alleviated via $w_{\experience} = \fun{|\replaybuffer| \cdot P\fun{\experience}}^{-\omega}$ for each $\experience \in \replaybuffer$, $\omega \in \mathopen[0, 1\mathclose]$, where $\omega = 1$ fully corrects it.

We use prioritized replay buffers to store transitions and sample them when optimizing our loss function, i.e., $\elbo$.
This allows to alleviate the posterior collapse problem, as shown in Figure~\ref{appendix:fig:prioritized-replay}.
We introduced two priority functions, yielding a bucket-based and a loss-based prioritized replay buffer.
In the following, we elaborate how precisely we assign priorities via the loss-based priority function.

\smallparagraph{Loss-based priorities.}~Draw $\tuple{\state, \action, \reward, \state'} \sim \stationary{\policy}$, we aim at assigning $p_{\tuple{\state, \action, \reward, \state'}}$ to its \emph{individual transition loss}, i.e.,
$\loss_{\elbo}$:
for $\latentstate, \latentstate' \sim \embed_{\encoderparameter}\fun{\sampledot \mid \state, \state'}$, $\latentaction \sim \encoder^\actions\fun{\sampledot \mid \latentstate, \action}$ if $\latentactions \neq \actions$ and $\latentaction = \action$ else,
    \begin{align*}
        & \loss_{\elbo}\fun{\tuple{\state, \action, \reward, \state'}, \tuple{\latentstate, \latentaction, \latentstate'}} \\
        =&  - \log \decoder^\generative\fun{\state' \mid \latentstate'} - \log \decoder^{\rewards}\fun{\reward \mid \latentstate, \latentaction} - \log \latentpolicy_{\decoderparameter}\fun{\latentaction \mid \latentstate} + \log \embed_{\encoderparameter}\fun{\latentstate' \mid \state'} \\
        & - \condition{=}\fun{\latentactions, \actions} \cdot \log \latentprobtransitions_{\latentpolicy_{\decoderparameter}}\fun{\latentstate' \mid \latentstate} - \condition{\neq}\fun{\latentactions, \actions} \cdot \fun{\log \embeda_{\decoderparameter}\fun{\action \mid \latentstate, \latentaction} + \log \latentprobtransitions_{\decoderparameter}\fun{\latentstate' \mid \latentstate, \latentaction} - \log \encoder^{\actions}\fun{\latentaction \mid \latentstate, \action}}.
    \end{align*}
Notice that we cannot assign directly $p_{\tuple{\state, \action, \reward, \state'}}$ to $\loss_{\elbo}$ since priorities require to be strictly positive to compute the replay buffer distribution.
A solution is to pass $\loss_{\elbo}\fun{\tuple{\state, \action, \reward, \state'}, \tuple{\latentstate, \latentaction, \latentstate'}}$ to the logistic function, i.e.,
\[
    p_{\tuple{\state, \action, \reward, \state'}} = x^{\star} \cdot \sigma\fun{k \cdot \fun{\loss_{\elbo}\fun{\tuple{\state, \action, \reward, \state'}, \tuple{\latentstate, \latentaction, \latentstate'}} - x_0}}
\]
where $x^{\star} > 0$ is the scale, $k > 0$ is the growth rate, and $x_0$ is the location of the logistic.
In practice, we set $x^{\star}$ to the desired maximum priority and we maintain upper ($L^{\max}$) and lower bounds ($L^{\min}$) on the loss during learning.
Along training steps, we tune these values according to the current loss and we set $x_0 = \frac{L^{\max} - L^{\min}}{2}$ and $k = \frac{x^{\star}}{L^{\max} - L^{\min}}$.

\begin{figure}
    \begin{subfigure}{0.495\textwidth}
        \centering
        \includegraphics[width=\textwidth]{ressources/cartpole_histogram.pdf} \\
        \caption{CartPole}
    \end{subfigure}
    \begin{subfigure}{0.495\textwidth}
        \centering
        \includegraphics[width=\textwidth]{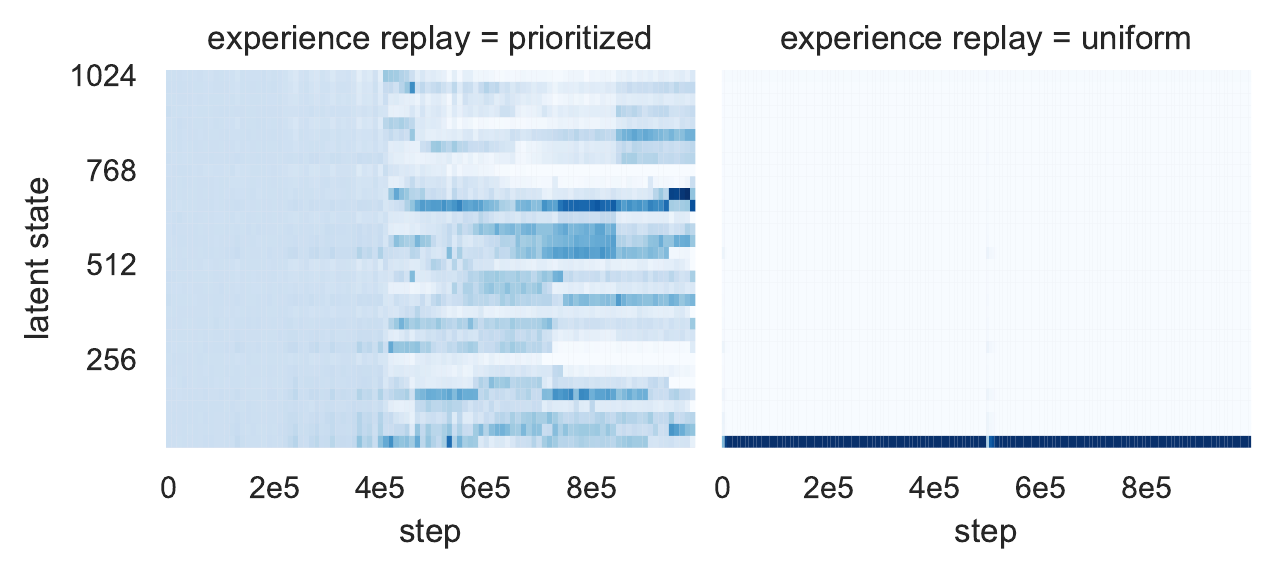} \\
        \caption{MountainCar}
    \end{subfigure}\\
    \begin{subfigure}{0.495\textwidth}
        \centering
        \includegraphics[width=\textwidth]{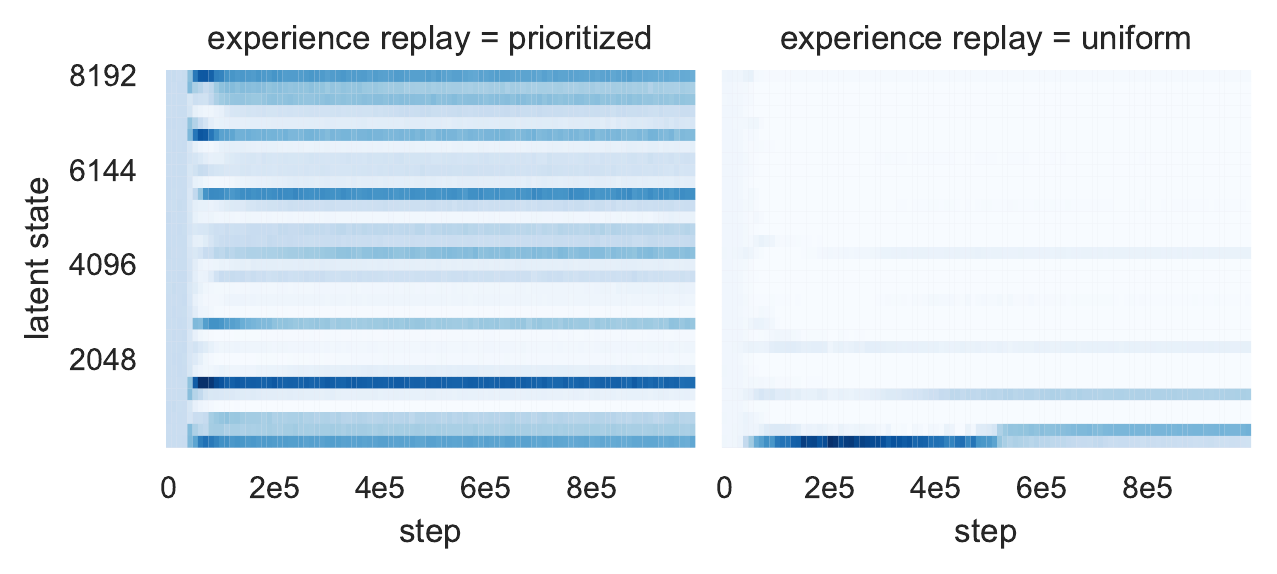} \\
        \caption{Acrobot}
    \end{subfigure}
    \begin{subfigure}{0.495\textwidth}
        \centering
        \includegraphics[width=\textwidth]{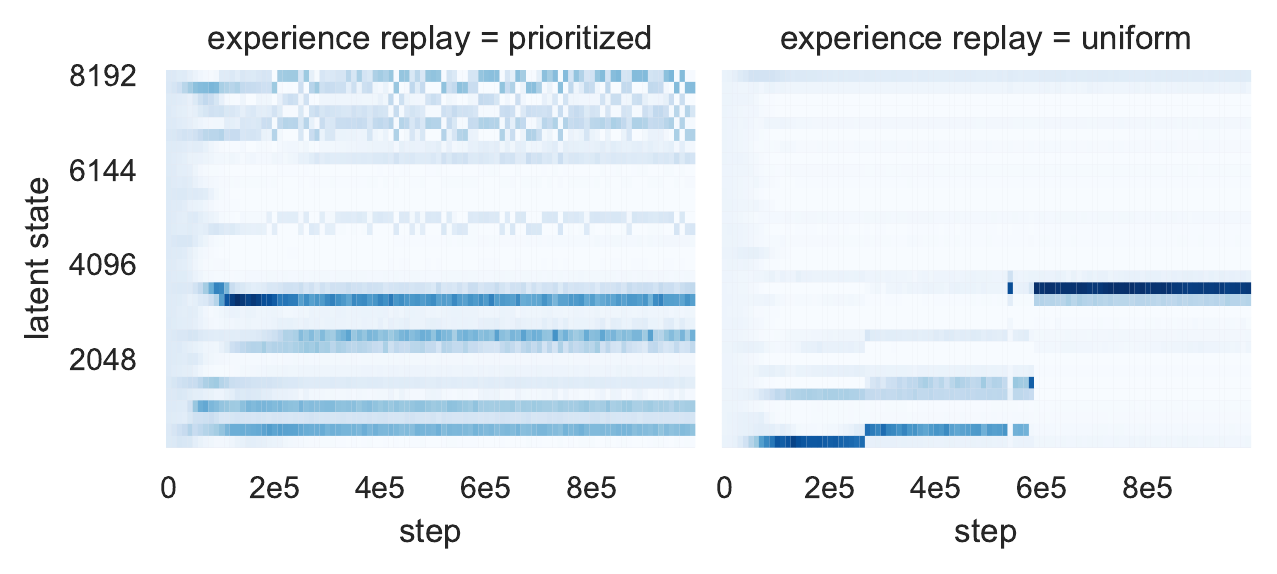} \\
        \caption{Pendulum}
    \end{subfigure}
    \caption{Latent space distribution along training steps. The intensity of the blue hue corresponds to the frequency of latent states produced by $\embed$ during training. We compare a bucket-based prioritized against a simple uniform experience replay. The latent space learned with transitions sampled from the uniform replay buffer collapses to (a) two, (b) a single, (c)~and~(d) few latent state(s).}
    \label{appendix:fig:prioritized-replay}
\end{figure}

\section{Details on RL Policy Distillation}
Concretely, the goal of $\latentpolicy_{\decoderparameter}$ is to mimic the behavior of $\policy$ through the latent spaces.
However, executing $\latentpolicy_{\decoderparameter}$ in $\mdp$ via $\embed_{\encoderparameter}$ and $\embeda_{\encoderparameter, \decoderparameter}$ as described in Sect.~\ref{section:deepmdp} to evaluate our latent space model (e.g., Lem.~\ref{lemma:pac-loss}) may
result in performance loss compared to that offered under $\policy$.
This is due to the abstraction learned via the ELBO: we do not expect a zero-distortion by encoding the state-action space through the VAE (and $\mdp_{\latentpolicy_{\decoderparameter}}$ to behave exactly like $\mdp_{\policy}$), but to minimize it.
Due to this remaining distortion, decisions taken according to $\latentpolicy_{\decoderparameter}$ may result in unreachable states under $\policy$, causing $\latentpolicy_{\decoderparameter}$ to behave poorly in that particular states since the learning process only allows to learn from transitions produced via $\policy$.
We propose two approaches that can be used alone or together to alleviate this problem.

\smallparagraph{Globally-robust RL policy.}~%
In general, policies $\policy$ learned through deep-RL provide \emph{local} performance, meaning they provide good performance in states $\state \in \states$ that are likely to be reached from $\sinit$ by executing $\policy$.
It is, however, often suitable to distil and verify \emph{globally-robust} RL policies, i.e., policies trained to provide good \emph{global} performance, in a wider range of states, in particular those that are likely to be reached through the distillation.
Training globally robust RL policies can be achieved by allowing the agent to learn (or pursuing its training) in a modified version of $\mdp$, where $\sinit$ is picked at random from $\states$.
Training $\policy$ this way will usually take longer since this may require more exploration.
The degree of randomness can be decided according to the exploration/exploitation trade-off and the additional training time allowed. 
Moreover, we argue that this additional time is acceptable in our context since enabling and applying model checking are time-consuming tasks by nature.

\smallparagraph{$\varepsilon$-mimic.}~Assuming $\policy$ is sufficiently robust for the distillation, we allow the encoder to process states reachable under both $\latentpolicy_{\decoderparameter}$ and $\policy$ by learning from a mixture $\hat{\policy}^{\varepsilon}_{\decoderparameter}$ of the two, defined as $\hat{\policy}^{\varepsilon}_{\decoderparameter}\fun{\action \mid \state} = \policy\fun{\action \mid \state} \cdot (1 - \varepsilon) + \latentpolicy_{\decoderparameter}\fun{\action \mid \state} \cdot \varepsilon$ for all $\state \in \states$, $\action \in \actions$, given $\varepsilon \in \mathopen[0, 1 \mathclose]$.
Given the robustness assumption, states entered according to a decision $\action \sim \latentpolicy_{\decoderparameter}\fun{\sampledot \mid \state}$ (produced with probability $\varepsilon$) should be efficiently processed
by $\policy$.
In practice, we start with a high value of $\varepsilon$ (e.g., $\varepsilon=\nicefrac{1}{2}$) to encourage exploration of states possibly reachable under $\latentpolicy_{\decoderparameter}$ and we anneal it to $0$ during training.

\section{Experiments}\label{appendix:experiments}
\subsection{Setup}
We used \textsc{TensorFlow} \texttt{2.4.1} \citepAR{tensorflow2015-whitepaper} to implement the neural networks modeling our parameterized distributions constituting our variational model, optimize the loss function ($\elbo$), and running our PAC approximation schemes.
Precisely, we used \textsc{TensorFlow Probability} \texttt{0.12.2} \citepAR{dillon2017tensorflow} to handle the probabilistic components of the VAE (e.g., distributions, reparameterization tricks, etc.), as well as \textsc{TF-Agents} \texttt{0.7.1} \citepAR{TFAgents} to handle the RL parts of the framework, coupled with \textsc{Reverb} replay buffers \citepAR{cassirer2021reverb}.

The experiments have been driven on a cluster running under \texttt{CentOS Linux 7 (Core)} composed of a mix of nodes containing Intel processors with the following CPU microarchitectures:
(i) \texttt{10-core INTEL E5-2680v2}, (ii) \texttt{14-core INTEL E5-2680v4}, and (iii) \texttt{20-core INTEL Xeon Gold 6148}.
We used $8$ cores and $128$ GB of memory for each of our experiments.

\subsection{Hyper-parameters}
Each distribution parameters (locations, scales, logits) are inferred by neural networks (multilayer perceptrons) composed of two dense hidden layers and $256$ cells by layer.
For all our experiments, we used the \textsc{TensorFlow}'s implementation of the Adam optimizer \citepAR{DBLP:journals/corr/KingmaB14} to learn the neural networks weights.
We fixed the size of the prioritized experience replay to $10^{6}$.
The training starts when the agent has collected $10^{4}$ transitions in $\mdp$.
Every time a transition is added to the replay buffer, its priority is set to the maximum value.
We used minibatches of size $128$ to optimize the loss function (i.e., $\elbo$) and we applied a minibatch update every time the agent executing $\policy$ performed $16$ steps in $\mdp$.
This allows the VAE to process at least $\nicefrac{1}{8}$ of the new transitions collected by sampling from the replay buffer (more details in \citealt{DBLP:journals/corr/SchaulQAS15}).
We fixed $\omega$ to $0.4$ as suggested by \citet{DBLP:journals/corr/SchaulQAS15}.

\smallparagraph{Annealing schemes.}~During training, we anneal different hyper-parameters to $0$ or $1$ according to the following annealing schemes:
(i) we start annealing the parameters at step $t_0 = 10^{4}$, 
(ii) let $\zeta \in \mathopen[ 0, \infty \mathclose]$, we anneal $\zeta$ to $0$ via $\zeta_t = \zeta \cdot (1 - \tau_{\zeta})^{t - t_0}$, or
(iii) let $\zeta \in \mathopen[ 0, 1]$, we anneal $\zeta$ to $1$ via $\zeta_{t} = \zeta + (1 - \zeta)\cdot (1 - (1 - \tau_{\zeta})^{t - t_0})$
for step $t > t_0$ and annealing term $\tau_{\zeta}$.

\smallparagraph{Unshared parameters.}~The list of single hyper-parameters used for the different environments is given in Table~\ref{table:hyperparameters}.
The activation function of all cells of the neural networks was chosen between relu and leaky relu, the learning rate of the optimizer in $\mathopen[10^{-4}, 10^{-3}\mathclose]$, the regularizer scale factor $\alpha$ from $\set{0.1, 1, 10, 100}$, $\varsigma$ in $\mathopen[0, 1\mathclose]$, the values of each annealing term in $\mathopen[ 10^{-6}, 10^{-4} \mathclose]$, and we tested both loss- and bucket-based experience replay buffers.

\begin{table}
    \centering
    \input{hyperparameters}
    \caption{Hyper-parameter choices for the different environments.
    We write
    $\temperature^X_Y$ for the temperature parameter used for the concrete relaxation of Bernoulli distributions when $X = \latentstates$ (resp. Categorical distributions when $X = \latentactions$).
    This is the encoder distribution if $Y=1$ and the one of the latent transition function (resp. distilled policy) else, if $Y=2$. 
    Moreover, we write $\alpha_{\scriptscriptstyle \actions}$ for the scale factor of the action entropy regularizer where the final entropy regularization term is $\alpha \cdot \fun{\entropy{\encoder} + \alpha_{\scriptscriptstyle \actions} \cdot \entropy{\encoder^{\actions}}}$.}
    \label{table:hyperparameters}
\end{table}

\subsection{Discrete Latent Dynamics}
As pointed out by \citet{DBLP:conf/icml/CorneilGB18}, the latent transition function learned via the continuous relaxation takes a simple form when discretized (i.e., when $\temperature_{\probtransitions_{\decoderparameter}} \to 0$):
the latent variables produced this way via $\latentprobtransitions_{\decoderparameter}$ are independent Bernoullis, meaning each bit of $\latentstate' \sim \latentprobtransitions_{\decoderparameter}\fun{\sampledot \mid \latentstate, \latentaction}$ are independent, for any $\latentstate \in \latentstates, \action \in \latentactions$.
To remedy this, we adopt the approach of \citet{DBLP:conf/icml/CorneilGB18} and we reconstruct the latent dynamics via \emph{frequency estimation} (e.g., \citealtAR{DBLP:conf/cav/BazilleGJS20}) to check the abstraction quality in a PAC manner (cf. Fig.~\ref{fig:plots}): for each environment, we evaluated the latent space model quality based on frequency-estimated latent dynamics from transitions collected by (repeatedly) simulating the environment for $2\cdot 10^5$ steps.

\subsection{Labeling functions}
The labeling functions $\labels$ we used for each environment are detailed in Table~\ref{appendix:table:labels}.
\input{labels}

\section{Reproducibility of the Results}
\smallparagraph{Seeding.}~The code we provide allows for setting random seeds to run the experiments, making the overall shape of the Figures we present in this paper reproducible.
However, the exact values obtained in each particular instance are not: our prioritized experience replays (and those of \textsc{TF-Agents}) rely on \textsc{Reverb} replay buffers, handling randomness via the \textsc{Abseil} \texttt{C++} library{\footnote{\url{https://abseil.io}}} which does not allow for manual seeding.
All other libraries using randomness are seeded properly.

%\clearpage
\bibliographystyleAR{aaai22}
\bibliographyAR{references}

%% file: hyperparameters.tex
\begin{tabular}{llllll}
\toprule
{} &           CartPole &        MountainCar &            Acrobot &           Pendulum &        LunarLander \\
\midrule
activation                                             &  leaky relu &  leaky relu &  relu &  relu &  relu \\
learning rate                                          & $10^{-3}$ & $10^{-3}$ & $10^{-4}$ & $10^{-4}$ & $10^{-4}$ \\
RL policy                                              &  DQN &  DQN &  DQN (robust) &  SAC (robust) &  SAC \\
$\log_2 \left | \latentstates \right|$                 &  9 &  10 &  13 &  13 &  16 \\
$\left| \latentactions \right|$                        &  / &  / &  / &  3 &  5 \\
$\temperature_{1}^{\scriptscriptstyle \latentstates}$  & $\nicefrac{2}{3}$ & $\nicefrac{2}{3}$ & $\nicefrac{2}{3}$ & $\nicefrac{2}{3}$ & $\nicefrac{2}{3}$ \\
$\temperature_{2}^{\scriptscriptstyle \latentstates}$  & $\nicefrac{1}{2}$ & $\nicefrac{1}{2}$ & $\nicefrac{1}{2}$ & $\nicefrac{1}{2}$ & $\nicefrac{1}{2}$ \\
$\temperature_{1}^{\scriptscriptstyle \latentactions}$ &  / &  / &  / & $\nicefrac{1}{2}$ & $\nicefrac{1}{4}$ \\
$\temperature_{2}^{\scriptscriptstyle \latentactions}$ &  / &  / &  / & $\nicefrac{1}{3}$ & $\nicefrac{1}{6}$ \\
$\tau_{\temperature_{1}}$                              & $10^{-6}$ & $10^{-6}$ & $10^{-6}$ & $10^{-6}$ & $10^{-6}$ \\
$\tau_{\temperature_{2}}$                              & $2\cdot 10^{-6}$ & $2\cdot 10^{-6}$ & $2\cdot 10^{-6}$ & $2\cdot 10^{-6}$ & $2\cdot 10^{-6}$ \\
$\alpha$                                               & $10^{1}$ & $10^{1}$ & $10^{1}$ & $10^{1}$ & $10^{1}$ \\
$\alpha_{\scriptscriptstyle \actions}$                                             &  / &  / &  / & $1$ & $10^{-1}$ \\
$\tau_{\alpha}$                                        & $10^{-5}$ & $10^{-5}$ & $7.5\cdot 10^{-5}$ & $7.5\cdot 10^{-5}$ & $10^{-5}$ \\
$\tau_{\beta}$                                         & $5\cdot 10^{-5}$ & $5\cdot 10^{-5}$ & $7.5\cdot 10^{-5}$ & $7.5\cdot 10^{-5}$ & $5\cdot 10^{-5}$ \\
prioritized experience replay                          &  bucket &  bucket &  bucket &  loss &  loss \\
$\varsigma$                                      & $\nicefrac{1}{3}$ & $\nicefrac{1}{3}$ & $3\cdot 10^{-1}$ & $3\cdot 10^{-1}$ & $3\cdot 10^{-1}$ \\
$\tau_{\omega}$                                        & $7\cdot 10^{-5}$ & $7.5\cdot 10^{-5}$ & $7\cdot 10^{-5}$ & $10^{-5}$ & $10^{-5}$ \\
$\varepsilon$                                          & $0$ & $0$ & $\nicefrac{1}{2}$ & $\nicefrac{1}{2}$ & $0$ \\
$\tau_{\varepsilon}$                                   &  / &  / & $10^{-5}$ & $10^{-5}$ &  / \\
\bottomrule
\end{tabular}

%% file: labels.tex
\begin{table}
\centering
\small
\begin{tabular}{l|l|l|l|l}
\hline
Environment &
  $\states \subseteq$ &
  Description, for $\vect{\state} \in \states$ &
  Labeling function $\labels\fun{\vect{\state}}$ &
  $\atomicprops = \set{p_1, \dots, p_n}$ \\ \hline
CartPole &
  $\R^4$ &
  \begin{tabular}[c]{@{}l@{}}\tabitem $\vect{\state}_1$: cart position\\ \tabitem $\vect{\state}_2$: cart velocity\\ \tabitem $\vect{\state}_3$: pole angle (rad)\\ \tabitem $\vect{\state}_4$: pole velocity at tip\end{tabular} &
  $\begin{aligned}\langle &\condition{<}\fun{\vect{\state}_1, 1.5},\\& \condition{<}\fun{\vect{\state}_3, 0.15} \rangle \end{aligned}$ &
  \begin{tabular}[c]{@{}l@{}}\tabitem $p_1$: safe cart position\\ \tabitem $p_2$: safe pole angle\end{tabular} \\ \hline
MountainCar &
  $\R^2$ &
  \begin{tabular}[c]{@{}l@{}}\tabitem $\vect{\state}_1$: position\\ \tabitem $\vect{\state}_2$: velocity\end{tabular} &
  $\begin{aligned}\langle &\condition{\geq}\fun{\vect{\state}_1, \nicefrac{1}{2}}, \\ & \condition{\geq}\fun{\vect{\state}_1, \nicefrac{-1}{2}}, \\ & \condition{\geq}\fun{\vect{\state}_2, 0}\rangle \end{aligned}$ &
  \begin{tabular}[c]{@{}l@{}}\tabitem $p_1$: target position\\ \tabitem $p_2$: right-hand side of the mountain\\ \tabitem $p_3$: car going forward\end{tabular} \\ \hline
Acrobot &
  $\R^6$ &
  \begin{tabular}[c]{@{}l@{}}Let $\theta_1, \theta_2 \in \mathopen[0, 2\pi \mathclose]$ be the angles\\ of the two rotational joints,\\ \tabitem $\vect{\state}_1 = \cos\fun{\theta_1}$\\ \tabitem $\vect{\state}_2 = \sin\fun{\theta_1}$\\ \tabitem $\vect{\state}_3 = \cos\fun{\theta_2}$\\ \tabitem $\vect{\state}_4 = \sin\fun{\theta_2}$\\ \tabitem $\vect{\state}_5$: angular velocity 1\\ \tabitem $\vect{\state}_6$: angular velocity 2\end{tabular} &
  $\begin{aligned}\langle &\condition{>}\fun{-\vect{\state}_1 -\vect{\state}_3 \cdot \vect{\state}_1 + \vect{\state}_4 \cdot \vect{\state}_2, 1}, \\ & \condition{\geq}\fun{\vect{\state}_1, 0}, \\ & \condition{\geq}\fun{\vect{\state}_2, 0} \\ & \condition{\geq}\fun{\vect{\state}_3, 0} \\ & \condition{\geq}\fun{\vect{\state}_4, 0}\\ & \condition{\geq}\fun{\vect{\state}_5, 0}\\ & \condition{\geq}\fun{\vect{\state}_6, 0}\rangle \end{aligned}$ &
  \begin{tabular}[c]{@{}l@{}}\tabitem $p_1$: RL agent target \\\tabitem $p_2$: $\theta_1 \in [0, \nicefrac{\pi}{2}] \cup [\nicefrac{3\pi}{2}, 2\pi]$ \\
  \tabitem $p_3$: $\theta_1 \in [0, \pi]$\\
  \tabitem $p_4$: $\theta_2 \in [0, \nicefrac{\pi}{2}] \cup [\nicefrac{3\pi}{2}, 2\pi]$\\
  \tabitem $p_5$: $\theta_2 \in [0, \pi]$\\
  \tabitem $p_6$: positive angular velocity (1)\\
  \tabitem $p_7$: positive angular velocity (2)\\
  \end{tabular} \\\hline
Pendulum &
  $\R^3$ &
  \begin{tabular}[c]{@{}l@{}}Let $\theta \in \mathopen[0, 2\pi \mathclose]$ be the joint angle\\ \tabitem $\vect{\state}_1 = \cos\fun{\theta}$\\ \tabitem $\vect{\state}_2 = \sin\fun{\theta}$\\ \tabitem $\vect{\state}_3$: angular velocity\end{tabular} &
  $\begin{aligned}\langle &\condition{\geq}\fun{\vect{\state}_1, \cos\fun{\nicefrac{\pi}{3}}}, \\ & \condition{\geq}\fun{\vect{\state}_1, 0}, \\ & \condition{\geq}\fun{\vect{\state}_2, 0} \\ & \condition{\geq}\fun{\vect{\state}_3, 0} \rangle \end{aligned}$ &
    \begin{tabular}[c]{@{}l@{}}
    \tabitem $p_1$: safe joint angle \\
    \tabitem $p_2$: $\theta \in [0, \nicefrac{\pi}{2}] \cup [\nicefrac{3\pi}{2}, 2\pi]$ \\
  \tabitem $p_3$: $\theta \in [0, \pi]$\\
  \tabitem $p_4$: positive angular velocity\\
  \end{tabular} \\\hline
LunarLander &
  $\R^8$ &
  \begin{tabular}[c]{@{}l@{}}\tabitem $\vect{\state}_1$: horizontal coordinates\\ \tabitem $\vect{\state}_2$: vertical coordinates\\ \tabitem $\vect{\state}_3$: horizontal speed\\ \tabitem $\vect{\state}_4$: vertical speed\\ \tabitem $\vect{\state}_5$: ship angle\\ \tabitem $\vect{\state}_6$: angular speed\\ \tabitem $\vect{\state}_7$: left leg contact\\ \tabitem $\vect{\state}_8$: right leg contact\end{tabular} &
  based on the heuristic function &
  \begin{tabular}[c]{@{}l@{}}
  \tabitem $p_1$: unsafe angle\\
  \tabitem $p_2$: leg ground contact\\
  \tabitem $p_3$: lands too rapidly\\
  \tabitem $p_4$: left inclination\\
  \tabitem $p_5$: right inclination\\
  \tabitem $p_6$: motors shut down
  \end{tabular}
\end{tabular}
\caption{Labeling functions used in our experiments.
We provide a short description of the state space and the meaning of the atomic propositions forming the labels.
Recall that labels and $\atomicprops$ are respectively binary and one-hot encoded.
Let $n = |\atomicprops|$ with $\atomicprops = \set{p_1, \dots, p_n}$, we write $p_{i}$ for the one-hot vector of size $n$ with the $i^{\text{th}}$ entry set to $1$, with $i \in [n] \setminus \set{0}$.
The LunarLander labeling function is based on the heuristic function of the OpenAI's implementation \citep{DBLP:journals/corr/BrockmanCPSSTZ16}.}
\label{appendix:table:labels}
\end{table}